\documentclass{article}

\usepackage{mathtools}
\usepackage{times}
\usepackage{fullpage}

\usepackage{xspace}
\usepackage{microtype}
\usepackage{graphicx}
 \usepackage{subfig}
\usepackage{booktabs} 
\usepackage[disable]{todonotes}
\usepackage{authblk}
\usepackage[numbers,square]{natbib}
 \usepackage[ruled]{algorithm2e}
 \usepackage{algorithmic}
 \usepackage{subdef}
\usepackage{cancel}
\newcommand{\hll}[1]{\colorbox{yellow}{#1}}

\newcommand{\trn}{\Dcal}
\newcommand{\dev}{\Vcal}

\newcommand{\deh}{\nabla_{\wb}h_{\wb}(\xb)}
\newcommand{\dehz}{\nabla h_{0}}
\newcommand{\wg}{\wb^*(\mub,\Scal)}

\newcommand{\sumv}{\sum_{q\in [Q]}\frac{\mu_q}{|V_q|} \sum_{j\in V_q}}

\usepackage{hyperref}

\newcommand{\xhdr}[1]{\vspace{1mm}\noindent{{\bfseries #1.}}}
\newcommand{\ehdr}[1]{\noindent{{\em #1: }}}
 \newcommand{\nn}{\nonumber}

\usepackage{cleveref}
\Crefname{ALC@unique}{Line}{Lines}

\newcommand{\set}[1]{\{#1\}}
\newcommand{\bnm}[1]{\left\| #1 \right\|}

\newcommand{\minz}[1]{\underset{#1}{\text{minimize}}}

\newcommand{\primal}{g}

\usepackage[numbers,square]{natbib}
 
\renewcommand{\obj}{f}
\usepackage{enumitem} 
\setlist{nosep}

\title{Training Data Subset Selection for Regression\\ with Controlled Generalization Error}
\newcommand{\dual}{F}
\newcommand{\af}{\widehat{\alpha}_f}
\newcommand{\kf}{\widehat{\kappa}_f}
\newcommand{\full}{Full-selection\xspace}
\newcommand{\Full}{Full-selection\xspace}

\newcommand{\random}{Random-selection\xspace}
\newcommand{\Random}{Random-selection\xspace}

\newcommand{\fullc}{Full-with-constraints\xspace}
\newcommand{\Fullc}{Full-with-constraints\xspace}

\newcommand{\randomc}{Random-with-constraints\xspace}
\newcommand{\Randomc}{Random-with-constraints\xspace}
\newcommand{\ourcon}{\our-without-constraint\xspace}

\renewcommand{\ie}{i.e.}
\author[1]{Durga Sivasubramanian}
\author[2]{Rishabh Iyer}
\author[1]{Ganesh Ramakrishnan}
\author[1]{Abir De}

\affil[1]{%
  {IIT Bombay, $\{$ durgas, ganesh, abir $\}$@cse.iitb.ac.in}
}
\affil[2]{%
  {UT Dallas, rishabh.iyer@utdallas.edu}
}
\date{} 

\graphicspath{{./FIG/}}

\newcommand{\our}{\textsc{SelCon}\xspace}
\newcommand{\oursub}[1]{\vspace{-1.5mm}\subsection{#1}\vspace{-1mm}}

\begin{document}

\maketitle
\begin{abstract}
Data subset selection from a large number of training instances has been a successful approach toward efficient and cost-effective machine learning. However, models trained on a smaller subset may show poor generalization ability. In this paper, our goal is to design an algorithm for selecting a subset of the training data, so that the model can be trained quickly, without significantly sacrificing on accuracy. More specifically, we focus on data subset selection for $L_2$ regularized regression problems and provide a novel problem formulation which seeks to minimize the training loss with respect to both the trainable parameters and  the subset of training data, subject to error bounds on the validation set. We tackle this problem using several  technical innovations. First, we represent this problem with simplified constraints using the dual of the original training problem and show that the objective of this new representation is a monotone and $\alpha$-submodular function, for a wide variety of modeling choices. Such properties lead us to develop \our, an efficient majorization-minimization algorithm for data subset selection, that admits an approximation guarantee even when the training provides an imperfect estimate of the trained model. Finally, our experiments on several datasets show that \our
trades off accuracy and efficiency more effectively than the current state-of-the-art.

\end{abstract}

\vspace{-2mm}
\section{Introduction}
\vspace{-1mm}

\label{sec:intro}
Data-driven estimation of the underlying statistical model is the central challenge in any supervised machine learning (ML) problem.
Thanks to the law of large numbers~\cite{casella2002statistical}, such a training procedure
often demands a huge number of training 
examples to ensure statistical reliability of the learned model.
Therefore, the success of several machine learning models 
can be attributed to the availability of a massive amount of data
and thus to the high performance computing infrastructures, 
\eg,  GPUs, multicore processors, high storage disks, {\em etc.}, which are
required to store and process such data. 
These computational resources involve large expenses,
additional energy utilization and maintenance costs. 
Mitigation of such overheads without sacrificing the accuracy of the predictive model is a challenging task, which often entails a careful selection of a smaller number of training instances, so that the training algorithm can be run in an environment with limited resources~\cite{lucic2017training,mirzasoleiman2019coresets,boutsidis2013near,kaushal2019learning,killamsetty2020glister,wei2014fast,liu2015svitchboard,bairi2015summarization,kirchhoff2014submodularity}. 
However, current data selection techniques do not explicitly account for the generalization error which may be exacerbated in the presence of a small sized training dataset. As a consequence, they can suffer from high generalization error, especially for large datasets.
\oursub{Present work}  
%
In response to the above limitations, our goal is to select a subset from training data in such a way that the model can be quickly trained in an environment with limited resources, while at the same time, provide good
predictive power. More specifically, we make the following contributions.

\xhdr{Novel formulation of data selection}
In this work,  we focus on the regression problem and 
introduce a novel problem formulation (Section~\ref{sec:formulation}) - which encodes the task of data selection for regression, while ensuring that the error on validation set remains below an acceptable level. Such an explicit use of the validation set during training improves the generalization ability of the inferred model, as indicated in~\cite{ren2018learning,killamsetty2020glister}. 

More specifically, given a model class and a fixed validation set, we seek to minimize an $L_2$ regularized constrained squared error loss with respect to both the parameter vector and the subset of training data, subject to a set of error bounds
on different portions of the validation set. The use of such error bounds as optimization constraints enhances the generalization ability of the inferred model in  the face of small training data. 
Moreover, the presence of multiple error constraints in our setup can be  useful in several data selection problems; \eg, learning with heterogeneous data where each constraint limits the error for each cluster of data~\cite{rothenhausler2018anchor};
fair regression with bounded group loss, where each constraint limits the error on the protected group(s), \etc~\cite{agarwal2019fair}.

In general, our data selection problem is NP-hard due to the presence of both the training set and the model parameters as optimization variables. However, it permits us to reformulate it into
a new optimization task with simplified constraints,
by making use of the Lagrangian dual of the original training problem.
This new optimization problem can be seen as an instance of cardinality-constrained
set function minimization problem, where the objective corresponds to 
the optimal training loss as a function of the candidate data subset.


\xhdr{Characterizing the loss function of data selection}
Having represented the optimal training loss as a set function,
%
we show that this function is monotone, $\alpha$-submodular~\cite{gatmiry2018non,lehmann2006combinatorial,hassani2017gradient} and enjoys a bounded generalized curvature~\cite{NIPS2013_c1e39d91,zhang2016submodular} for a wide variety of models including a class of nonlinear functions (Section~\ref{sec:method}).
These technical results can be useful in other related data selection problems and
therefore, are of independent interest.

\xhdr{Approximation algorithm for data selection} 
Finally, to solve our data selection problem, we design \our, a new majorization-minorization algorithm (Algorithm~\ref{alg:selcon}, Section~\ref{sec:algo}) building upon the semi-differentials  proposed by~\cite{iyer2013fast,iyer2015polyhedral}, which minimizes the set function characterized above. \our enjoys an approximation guarantee even when the training algorithm provides an imperfect estimate. While doing so, we obtain a new family of modular upper bounds of an $\alpha$-submodular function, which extends the bounds proposed in~\cite{iyer2013fast} and therefore, can be of independent technical interests. Moreover, \our can minimize any monotone,  $\alpha$-submodular function, going beyond the particular instance in this work, which makes it useful from a broader perspective. 
%

We evaluate\footnote{\scriptsize Our code and data is available at \url{https://github.com/abir-de/SELCON}} our framework on several real-world datasets, and demonstrate that \our trades off the accuracy
and efficiency more effectively than several baselines and state-of-the-art. We also demonstrate that the use of constrained validation set error
maintains the generalization ability of the inferred model in the presence of small training data. 
Finally, we test our framework on the application of fair regression with bounded group loss, which 
shows that \our offers fair prediction along with an effective trade-off between accuracy and efficiency. 
\oursub{Related work}
%
Algorithms for data selection predominantly follow two approaches.
The first approach~\cite{wei2014fast,wei2014unsupervised,liu2015svitchboard,bairi2015summarization} selects diverse training examples by maximizing submodular \emph{proxy} functions, \eg, facility location, \etc,
and \emph{then} use them to train the underlying model. 
%
%
%
The second approach selects \emph{coresets} --
weighted subsets of training examples -- alongside training the model over them~\cite{lucic2017training,mirzasoleiman2019coresets,killamsetty2021grad,campbell2018bayesian,boutsidis2013near,kaushal2019learning}. The choice of a coreset depends strongly on the model as well as on the training loss. Therefore, coreset selection algorithms vary widely across different ML settings, \eg, SVM~\cite{clarkson2010coresets}, Bayesian inference~\cite{campbell2018bayesian}, k-means clustering~\cite{har2004coresets}, regression~\cite{boutsidis2013near}, deep learning~\cite{mirzasoleiman2019coresets,killamsetty2021grad}, \etc
However, they do not explicitly control the validation set error, which often constrains their predictive power.

Our work is related to robust and efficient learning methods~\cite{ren2018learning,zhang2018generalized,killamsetty2020glister}, that utilize the validation set to improve the training performance via a bi-level optimization. However, these approaches do not explicitly control the validation set error the way we do.
Our work is also related to subset selection problems in the context of human-assisted machine learning~\cite{ruha,cuha}, that aim to select a training subset to outsource to humans, rather than facilitating efficient learning. Moreover, unlike us, these setups do not consider any validation constraint.  Our work is also connected with batch active learning methods~\cite{wei2015submodularity,hashemi2019submodular,kulkarni2018active,sener2018active},
that aim to select examples from training data in order to minimize the labeling cost. In contrast, our setup has access to all the labels and it aims to select data to improve efficiency.

In recent years, there is a flurry of works on maximizing non-submodular functions~\cite{Horel2016,das2011submodular,bian2017guarantees,Kuhnle2018,gatmiry2018non,Hassidim2018,Hassidim2017}. However, there is a paucity of work on minimizing $\alpha$-submodular functions. Very recently, Halabi \emph{et al.}~\cite{el2020optimal} aim to minimize the difference between two monotone $\alpha$-submodular and $\beta$-submodular functions. However, they do not consider a cardinality constraint, which makes their approach  less relevant to our setting.


\vspace{-2mm}
\section{Problem formulation}
\vspace{-1mm}
\label{sec:formulation}
 In the following,  we first setup the notation and contextualize our problem. Thereafter, we formally present  our data selection problem which involves simultaneous
selection of a subset $\Scal$ of the training dataset $\Dcal$ and training of a regression model $y\approx h_{\wb}(\xb)$, subject to validation error constraints. We obtain an alternative representation of this problem, using the Lagrangian dual of the parameter estimation task. Finally we formally show that our data selection problem is NP-Hard. 

\oursub{Notation}
Let $\{\xb_i,y_i\}_{i\in \trn}$ be the set of training samples and
$\{\xb_j,y_j\}_{j\in \dev}$ the set of validation samples. Here, $\xb_\bullet\in \RR^d$ are the features and $y_\bullet \in \RR$ are the corresponding response (output) variables.  We also have a partition of $Q$ subsets over the validation set,  {\em i.e.}, $\dev =  V_1 \cup V_2 \cup \ldots \cup V_Q$. Unless otherwise stated, $\bnm{\cdot}$ denotes the $L_2$ norm, {\em i.e.},  $\bnm{\xb}=\sqrt{\xb^\top \xb}$.
%

\oursub{Our broad objective}
We are provided a modeling framework $h_{\wb}:\RR^d\to \RR$ which can approximate the relationship 
between $\xb$ and $y$, {\em i.e.}, $y\approx h_{\wb}(\xb)$, where $\wb$ is a trainable parameter vector. 
Given the aforementioned setup, one can learn $\wb$ using standard least square estimation. 
%
In principle, one might be tempted to estimate $\wb$ using the entire set of training examples $\trn$, which would possibly give a statistically sound estimate of $\wb$. 
However, if the size of $\trn$ is large, such exhaustive training may be inefficient in a typical computing environment.  To tackle this problem, our goal is to determine a smaller subset  of training samples $\Scal\subset \Dcal$ such that it allows for efficient training of the model without significant drop in accuracy.

\oursub{Problem setting for data selection}
Given the full training set $\set{\xb_i,y_i}_{i\in\trn}$ and
the validation set $\set{\xb_j,y_j}_{j\in\dev}$ along with its partitions $\Vcal = \cup_{q\in[Q]} V_q$ and
the model class $h_\wb$, we consider minimization of the $L_2$ regularized training loss, jointly with respect to  parameters $\wb$ and the candidate subset $\Scal$, subject to a set of
constraints that bound the mean squared errors (MSE) on the $Q$ partitioning subsets of the validation set, {\em i.e.},
\begin{align}
 \mini_{ \Scal\subset \trn, \wb}\ \ & \sum_{i\in\Scal}[ \lambda\bnm{\wb}^2 +  (y_i-h_{\wb}(\xb_i))^2], \nn \\
 \text{subject to, }&  \frac{\sum_{j\in V_q}(y_j-h_{\wb}(\xb_j))^2}{|V_q|}  \le  \delta, \   \forall q\in [Q],\nn\\
 & |\Scal| = k. \label{eq:opt-hard}
\end{align}
Here, $\lambda$ is the coefficient of the regularizer; the cardinality constraint limits
the number of training samples to be chosen; and the validation error constraints ensure that the predictor's loss
remains below some acceptable level $\delta$ for the subsets $\set{V_q}$ of the validation set\footnote{\scriptsize For the sake of brevity, we assumed the same value of $\delta$ across different validation subsets $\{V_q\}$.}. \\

\xhdr{Discussion on {multiple} validation error bounds}
Note that the absence of validation error constraints in the basic problem setting 
 may  result in efficient training, but might not generalize well owing to small size of the training  data. The validation error constraints in \eqref{eq:opt-hard} ameliorate this problem, by attenuating the generalization error which might have exacerbated in the face of a small sized training data.

We note that, in order to improve the generalization ability, one may consider bounding the MSE on the entire validation set as one single constraint, {\em viz.},  $\frac{1}{|\Vcal|}\sum_{j\in\Vcal}(y_j - h_{\wb}(\xb_j))^2 \le  \delta $, 
rather than constraining the MSE for multiple subsets of the validation set as in Eq.~\eqref{eq:opt-hard}. 
However, we envision the use of formulation~\eqref{eq:opt-hard} in several applications.
For example, in the case of fair regression with bounded group loss,
the validation set can be partitioned in a way that each subset $V_q$ corresponds to the sub-population for a protected group, so that
the individual MSE for each protected group remains small.
Our setup can also be useful in learning from heterogeneous data, wherein the heterogeneity could have arisen owing to multiple sources
of data, time-shifts in the distribution, \etc. To address such requirements, the validation set can be partitioned into different subsets, where each subset represents a partition with similar properties.
 
\oursub{A soft-constraint approach}
 It is evident that arbitrarily reducing $\delta$ would eventually make the error constraints  infeasible in the above optimization problem~\eqref{eq:opt-hard}. Therefore, we relax the constraints by
provisioning for some margin of violation of these  constraints. To this aim,
we introduce new slack variables $\xi_1,\xi_2,...,\xi_Q$ and replace each hard validation error (inequality) constraint in Eq.~\eqref{eq:opt-hard} by a soft constraint, {\em i.e.},  $\frac{1}{|V_q|}(y_j- h_{\wb} (\xb_j))^2 \le \delta + \xi_q$ similar to the soft-SVM formulation. Here $\xi_q$
measures the extent of error violation in the constraint $\frac{1}{|V_q|}(y_j- h_{\wb} (\xb_j))^2 \le\delta$. Finally, we minimize
the sum of regularized loss computed over the candidate set $\Scal$, along with a weighted sum of the slack variables that penalizes the constraint violation to yield the optimization problem in Eq.~\eqref{eq:opt-soft}, {\em i.e.},  
\begin{align}
 &\hspace{-3mm} \mini_{\Scal\subset \trn,\wb, \{\xi_q\}_{q\in[Q]}}\,   \sum_{i\in\Scal} [ \lambda\bnm{\wb}^2 \hspace{-1mm} +  (y_i-h_{\wb}(\xb_i))^2]  \hspace{-0.4mm}  +  \hspace{-0.4mm} C\hspace{-1mm} \sum_{q\in V_q} \xi_q,\hspace{-2mm}  \nn\\ 
 & \hspace{1mm}\text{subject to, }  \frac{\sum_{j\in V_q}(y_j-h_{\wb}(\xb_j))^2}{|V_q|}  \le  \delta + \xi_q \quad \forall q\in [Q], \nn\\ 
 & \qquad \qquad \quad \xi_q \ge 0\ \quad \forall \, q \in [Q] \text{ and, }\ |\Scal| = k \label{eq:opt-soft} 
\end{align}
where $\set{\xi_q}$ are the optimization variables in addition to the parameter vectors $\wb$ and
the candidate set $\Scal$ that were already specified in~\eqref{eq:opt-hard}. Through $C$, we can control the extent of penalization on the of violation of the validation set error. We note that
as $C\to\infty$, the above formulation becomes equivalent to its hard constrained counterpart~\eqref{eq:opt-hard}.

We may consider two possible approaches to solve the optimization problem in Eq.~\eqref{eq:opt-soft}.
In the first approach, we initially minimize the optimization problem~\eqref{eq:opt-soft}  with respect to $\Scal$ for fixed $\wb$ and $\set{\xi_q}$; and thereafter minimize the inner optimization objective with respect to $\wb$ and $\set{\xi_q}$. 
This can be viewed as an instance of minimizing the sum of $k$ smallest elements, which we expect to be intractable, since it is a concave minimization problem. In the second approach, given a fixed set $\Scal$ we first minimize~\eqref{eq:opt-soft} with respect to $\wb$ and $\set{\xi_q}$;
and thereafter, minimize this quantity with respect to $\Scal$. In this work, we focus on the second approach, which, as we will show in Section~\ref{sec:algo}, provides a tractable solution with an approximation guarantee.
For any given set $\Scal$, let the optimal
value of the parameters be $\wb^*(\Scal)$ and $\xi^* _q(\Scal)$.
%
We note that, if we define,
%
\begin{align}
\primal(\Scal) =  \sum_{i\in\Scal}[ & \lambda\bnm{\wb^*(\Scal)}^2 +  (y_i-h_{\wb^*(\Scal)}(\xb_i))^2]  + C\sum_{q\in[Q]} \xi^* _q(\Scal),\label{eq:gdef}
\end{align}
then, our  data selection problem becomes equivalent to
\begin{align}
 & \mini_{\Scal} \primal(\Scal), \ \ \text{subject to, } |\Scal| =  k.\label{eq:gopt} 
\end{align}
\vspace{-5mm}
\oursub{Representation of Eq.~\eqref{eq:opt-soft} with simplified constraints}
Next, we obtain an alternative representation of the data selection problem, by making use of the Lagrangian dual\footnote{\scriptsize The dual is formed with respect to the model parameters $\wb$ and $\set{\xi_q}$, which allows us to augment the validation error constraints and $\set{\xi_q}\ge 0$ into the new objective. However, it still remains as a constrained optimization problem with respect to $\Scal$. } of the optimization problem~\eqref{eq:opt-soft} for a fixed $\Scal$, as formalized in the following proposition (Proven in Appendix~\ref{app:proof-dual} in the supplementary material).
As we shall discuss, such a new representation becomes equivalent to Eq.~\eqref{eq:gopt} for convex loss functions.\looseness-1
\begin{proposition} \label{prop:dual}
Given a fixed training set $\Scal$, let $\mub=[\mu_q]_{q\in[Q]}$ be the Lagrangian multipliers for the constraints
$ \set{\frac{1}{{|V_q|}}{\sum_{j\in V_q}(y_j-h_{\wb}(\xb_j))^2}\le   \delta + \xi_q}_{q\in[Q]}$ in the optimization problem~\eqref{eq:opt-soft} and $\dual(\wb,\mub,\Scal)$ be defined as follows:
\begin{align}
\dual(\wb,\mub,\Scal)& =    \sum_{i\in\Scal}  [ \lambda\bnm{\wb}^2 +  (y_i-h_{\wb}(\xb_i))^2]  
  +  \sum_{q\in[Q]}  \hspace{-1mm} \mu_q  \left[ \frac{\sum_{j\in V_q}(y_j-h_{\wb}(\xb_j))^2}{|V_q|}  -  \delta\right] \label{eq:adef}  \hspace{-2mm}
\end{align}
Then, for the fixed set $\Scal$, the dual of the optimization problem~\eqref{eq:opt-soft} for estimating $\wb$ and $\set{\xi_q}$ is given by,
\begin{align}
 \maxi_{\bm{0}\le \mub \le C \bm{1}} \ \mini_{\wb} \ \ &  \dual(\wb,\mub,\Scal)\label{eq:aopt}
\end{align}
\end{proposition}
Let the inner minimization sub-problem of the above optimization problem have the solution $\wb^*(\mub,\Scal)$ for a given
$\mub$ and $\Scal$. If the corresponding outer maximization problem has the solution $\mub^* = \mub^*(\Scal)$ for a given $\Scal$,
then the above dual problem has an optimal solution at $(\wb^*(\mub^*(\Scal), \Scal), \mub^*(\Scal))$. 
To this end, given any set $\Scal$, we write the solution of this dual problem as the following set function.
\begin{align}
 \displaystyle\obj(\Scal) = \dual(\wb^*(\mub^*(\Scal),\Scal ), \mub^*(\Scal), \Scal)  \label{eq:def}
\end{align}
Subsequently, we aim to select $|\Scal|$ by solving the following optimization problem. 
\begin{align}
 \mini_{\Scal\subset \trn} f(\Scal)\qquad \text{subject to, \ } |\Scal| = k. \label{eq:fopt}
\end{align}
 
\xhdr{Relation between $f(\Scal)$ and $g(\Scal)$}
Given a fixed $\Scal$, the optimization problems~\eqref{eq:opt-soft} and~\eqref{eq:aopt} are equivalent for convex losses. However, they may not be equivalent for non-convex losses and, by weak-duality, $f(\Scal)$ would serve as a lower bound for $g(\Scal)$. This leads us to the following proposition.
\begin{proposition}
\label{prop:fg}
Given that $f(\cdot)$ and $g(\cdot)$ are defined in Eqs.~\eqref{eq:gdef} and~\eqref{eq:def} respectively,
$f(\Scal) \le g(\Scal)$ and the equality holds if the loss $(y-h_{\wb}(\xb))^2$ is convex with respect to $\wb$.  Hence, $\min_{\Scal, |\Scal|=k} f(\Scal) \le \min_{\Scal, |\Scal|=k} g(\Scal)$.
\end{proposition}

\oursub{Differences with weighted sum of training and validation loss}

\xhdr{Weighted sum of training and validation loss} Instead of our model, one can consider minimizing a weighted combination of training and validation losses, as follows:
\begin{align}
 &\minz{\wb,\eta} \ \eta\sum_{i\in\Scal} \big[\lambda ||\wb||^2 +   (y_i-h_{\wb}(\xb_i))^2\big] +    (1-\eta)\ (k/|V|)\sum_{j\in V}(y_j-h_{\wb}(\xb_j))^2 \label{eq:etaw}
\end{align}
The multiplier $k/ |V|$ in the second term above ensures correct scaling w.r.t. the first term.  Now, along with  $\eta$ can be estimated in two ways.

\ehdr{$\eta$ is a hyperparameter} We can treat $\eta$ as hyperparameter and cross validate them on another validation set $V{'}$. However, due to the requirement for tuning this additional hyperparameter, this approach is extremely time consuming and therefore, is not suitable for  {efficient data selection}. 

\ehdr{$\eta$ is a trainable parameter}
In this alternative approach, we train the $\eta$ along with $\wb$. Such a setup uses no additional validation set $V'$. However, since ${\min}_{\eta\in [0,1]}  (a \eta + (1-\eta) b) = \min\set{a,b}$, the  problem~\eqref{eq:etaw} reduces to
\begin{align}
 \minz{\Scal, \wb} \ \min\bigg\{ & \sum_{i\in\Scal} \big[\lambda ||\wb||^2  
  +  (y_i-h_{\wb}(\xb_i))^2\big],  \frac{k}{|V|} \sum_{j\in V_q}(y_j-h_{\wb}(\xb_j))^2\bigg\}
\end{align}
Hence, it can latch on either minimizing only training set error \emph{or} only validation set error, which results in (i) training only on validation set  \emph{or,} (ii) selecting subset without controlling  generalization error. 
%

\xhdr{Our approach} In our work, the Lagrangian multipliers $\mub$ of the dual objective $F(\wb, \mub, \Scal)$ defined in Eq.~\eqref{eq:adef} can also be viewed as weights for validation error. However, we  neither treat them as hyperparameters, nor learn them by simply minimizing the objective as in Eq.~\eqref{eq:etaw} above. Rather, our formulation naturally casts a max-min optimization task described in Eq.~\eqref{eq:aopt}, that trains $\wb$ and $\mub$ in an adversarial manner. This also ensures that the validation error is not much higher than $\delta$. In contrast, the formulation in Eq.~\eqref{eq:etaw}    neither trains $\wb$ and $\eta$ using max-min optimization nor incorporates~$\delta$.

\oursub{Hardness analysis for our approach}

Given any fixed training subset $\Scal$, we can learn the optimal 
solution of the problem~\eqref{eq:opt-soft} using a standard optimization technique. In fact,
it can be computed in polynomial time if the loss $(y-h_{\wb}(\xb))^2$  is convex\todo{Citation?}. 
However, simultaneously determining the optimal set $\Scal^*$ and the optimal parameters $\wb^*$ for that optimal set is not possible in polynomial time, as suggested by the following proposition (proof in Appendix~\ref{app:proof-hardness} in the supplementary).
\vspace{-1ex}
\begin{proposition}\label{prop:hardness}
 Both the variants of the data selection problems~\eqref{eq:gopt} and~\eqref{eq:fopt} are NP-Hard.
 \end{proposition}
\vspace{-1ex}
We will focus on minimizing $f(\Scal)$ rather than $g(\Scal)$, since that allows us to design a tractable algorithm with approximation guarantee for a wide class of modeling choices including nonlinear functions, and which works well in practice.  Moreover, since $f(\Scal)=g(\Scal)$ for convex losses, such an approximation guarantee also holds for $g(\Scal)$ in the specific case of linear regression.


\vspace{-2mm}
\section{Characterization of $f(\Scal)$}  
\vspace{-1mm}

\label{sec:method}
We next show that $f(\Scal)$
is monotone and $\alpha$-submodular and then, bound its generalized curvature, which 
would be subsequently used to design an efficient approximation algorithm for the optimization problem in Eq.~\eqref{eq:fopt}.
To help formally state the results, we begin with defining the following properties.
\begin{definition} Given a ground set $\Dcal$ and a set function $f:2^{\Dcal}\to \RR$, let us define $f(a \,|\,\Scal) = f(\Scal \cup\set{a})-f(\Scal)$. Then we have the following definitions.
\begin{enumerate}[leftmargin=*]
\item \emph{Monotonicity: }$f(\cdot)$ is monotone if $f(a \,|\,\Scal) \ge 0$ for all $\Scal\subset \Dcal$ and $a\in \Dcal\cp\Scal$.
\item \emph{$\alpha$-submodularity: }$f(\cdot)$ is $\alpha$-submodular with the submodularity parameter $\alpha >0$, if for $\Scal\subseteq \Tcal$  and $a\in \Dcal\cp\Tcal$, we have $f(a \,|\,\Scal)  \geq \alpha\, f(a \,|\,\Tcal)$~\cite{hashemi2019submodular,zhang2016submodular,el2020optimal}.
\item \emph{Generalized curvature: }Given a set $\Scal$,
    the generalized curvature of $f(\Scal)$ is defined as~\cite{NIPS2013_c1e39d91,zhang2016submodular}
    \begin{align}
        \kappa_f (\Scal) = 1-\min_{a\in \Dcal} \frac{f(a|\Scal\cp \set{a})}{ f(a|\emptyset)}.
    \end{align}
    \end{enumerate}
    \label{eq:deff}
\end{definition}
Note that, $\alpha$-submodularity is a natural extension of submodularity. An $\alpha$-submodular function $f(\Scal)$ is submodular if   $\alpha = 1$. Moreover, note that
an $\kappa_{f}(\Scal)\ge 1-1/\alpha$. For a general monotone function $f$,  $\alpha \le 1$. 

\oursub{Monotonicity of $\obj(\Scal)$}
We formalize the monotonicity of $\obj(\Scal)$, as defined in Eq.~\eqref{eq:def}, in the following proposition (proof in Appendix~\ref{app:proof-mon}).
\begin{proposition}\label{prop:mon}
For any model $h_{\wb}$, $f(\Scal)$
is monotone, {\em i.e.}, $f(\Scal\cup\set{a})-f(\Scal) \ge 0$ for all $\Scal\subset \Dcal$ and $a\in \Dcal\cp\Scal$.
\end{proposition}

\oursub{$\alpha$-submodularity of $\obj(\Scal)$}
Next, we set about to present our key results on $\alpha$-submodularity of
$f(\Scal)$ for different modeling choices of $h_\wb$. 
To this aim, we first characterize the submodularity parameter of $f(\Scal)$ for any bounded Hessian nonlinear model, 
in terms of $\lambda, C, \delta$ and some specific properties of the dataset (proof in Appendix~\ref{app:alpha-sub-key}). 
\begin{theorem}\label{thm:alpha-sub-nonlin}
Assume that $|y|\le y_{\max}$; 
$h_{\wb}(\xb)=0$ if $\wb=\bm{0}$, \ie, $h_{\wb}(\xb)$ has no bias term; 
$h_{\wb}$ is $H$-Lipschitz, \ie, $|h_{\wb}(\xb)|\le H\bnm{\wb} $;
the eigenvalues of the Hessian matrix of $(y- h_{\wb}(\xb))^2)$ have a finite upper bound, \ie,
$\displaystyle\mathrm{Eigenvalue}  (\nabla^2 _{\wb} (y- h_{\wb}(\xb))^2)$ $ \le  2\chi_{\max} ^2 $; and, 
define $ \ell^* = \min_{a\in\Dcal} {\min_{\wb}}\ \ \chi_{\max} ^2 \cdot \bnm{\wb}^2 +  (y_a-h_{\wb}(\xb_a))^2~>~0$. 
Then, for  $\lambda \ge \max\left\{ \chi_{\max} ^2,   {32(1+CQ)^2\ymx ^2 H^2}/{ \ell^*}  \right\}$,  
 $f(\Scal)$ is a $\alpha$-submodular set function, where 
 \begin{align}
    \alpha \ge \widehat{\alpha}_f = 1- \dfrac{32(1+CQ)^2\ymx ^2 H^2}{{\lambda \ell^*}},\label{eq:alphaNon}
\end{align}
\vspace{-2mm}
\end{theorem} 
Note that as $\lambda\to \infty,$ we have $ \alpha \to 1$, which implies that for large $\lambda$, $f(\Scal)$ becomes close to submodular.

\noindent \emph{Proof sketch:} 
The proof of the above theorem consists of two steps.
In the first step, we show that $f(\Scal\cup\set{a})-f(\Scal) \ge \min_{\wb}\lambda ||\wb||^2 + (y_i- h_{\wb}(\xb_a))^2 $. 
Next, we derive that $f(\Tcal\cup\set{a})-f(\Tcal) \le \lambda\bnm{\wb^*(\mub^*(\Ta),\Tcal)}^2 +  (y_a-h_{\wb^*(\mub^*(\Ta),\Tcal)}(\xb_a))^2$. Finally, we use different
properties of $f(\cdot)$ and the data to get a lower bound on the ratio of the above two quantities.

For a linear model $h_{\wb}(\xb) = \wb^\top \xb$, we exploit additional properties of the underlying model to obtain a slightly tighter bound (Proven in Appendix~\ref{app:cor-alpha-sub-lin}).
\begin{proposition}\label{cor:alpha-sub-lin}
Given $0< y_{\min} \le |y|~\le~y_{\max}$,~$h_{\wb}(\xb)=\wb^\top \xb$,  
$\bnm{\xb}\le x_{\max}$,  we set the regularizing coefficient as $\lambda \ge\max\big\{\xm^2,   { {16}(1+CQ)^2\ymx ^2 \xm^2}/{    { y_{\min} ^2 } }.\big\}$. Then $f(\Scal)$ is a $\alpha$-submodular set function, where 
\begin{align}
   \alpha  \ge \widehat{\alpha}_f =   1 -  \dfrac{ {16}(1+CQ)^2\ymx ^2 \xm^2}{\lambda   { y_{\min} ^2 } }. \label{eq:alphaLin}
\end{align}
\end{proposition}

Subset selection for linear regression problems has been widely studied in 
literature~\cite{hashemi2019submodular,ruha}. Except for~\cite{ruha}, these approaches optimize measures associated with the covariance matrix, rather than explicitly
minimizing the training loss subject to the validation set error bound. While~\cite{ruha} also optimizes the training loss, it applies to a completely different application of learning under human assistance, which is why it does not take into account the validation set error bound. 

\oursub{Generalized curvature}
Next, we provide a unified bound on the generalized curvature ({\em c.f.}, Definition~\ref{eq:deff}) for both linear and nonlinear modeling choices of $h_{\wb}(\xb)$, as formalized in the following proposition (proven in Appendix~\ref{app:curv}).
\begin{proposition}\label{prop:curv}
Given the assumptions stated in Theorem~\ref{thm:alpha-sub-nonlin},  the generalized curvature 
$k_{f}(\Scal)$ for any set $\Scal$ satisfies $\kappa_f(\Scal) \le \widehat{\kappa}_f = 1-\dfrac{\ell^*}{(CQ+1)\ymx^2}$.
\end{proposition}

\vspace{-2mm}
\section{The \our algorithm }  
\vspace{-1mm}
\label{sec:algo}
In this section, we design  \our, an iterative approximation algorithm
to minimize $f(\Scal)$, by leveraging the semi-differential based approach proposed by Iyer \emph{et al.}~\cite{iyer2013fast}.
However, they only consider submodular optimization problems having access to an exact measurement of the objective. 
In contrast, \our works for $\alpha$-submodular functions and 
enjoys an approximation guarantee even when it can only access an imperfect estimate of the learned parameters.  

\vspace{-2mm}
\subsection{Outline of \our}
\vspace{-1mm}
At the very outset, \our is an iterative Majorization-Minimization
algorithm for minimizing a monotone $\alpha$-submodular function. We first develop
a modular upper bound of $f(\Scal)$. Then, at each iteration, we minimize this upper bound and
refine the estimate of the candidate set $\Scal$.

\xhdr{Modular upper bound of $f(\Scal)$}
Given an $\alpha$-submodular function $f$  and a fixed set $\widehat{\Scal}$, we can obtain 
the modular upper bound of $f(\Scal)$, as follows (see details in Appendix~\ref{app:modupper}).
\begin{lemma}\label{mod-upper-weaksubmod}
Given a fixed set $\Schat$ and an $\alpha$-submodular function $f(\Scal)$, let the modular function $\mfs$ be defined as follows:
\begin{align}
\hspace{-4mm}\mfs= &    f(\Schat) -   \sum_{i \in \Schat}  \alpha f(i | \Schat \backslash \set{i}) + \sum_{i \in \Schat \cap \Scal} \alpha f(i | \Schat \backslash \set{i}) +   \sum_{i \in \Scal \backslash \Schat} \frac{f(i | \emptyset)}{\alpha} .\hspace{-1mm}\label{eq:mdef}
\end{align}
Then, $f(\Scal) \le \mfs$  for all $\Scal\subseteq\Dcal$.
\end{lemma}
Note that when $\alpha=1$, {\em i.e.}, $f$ is submodular, the expression $\mfs$
coincides with the existing modular upper bounds for submodular functions~\cite{nemhauser1978analysis,iyer2013fast,iyer2012algorithms}. 
Given a $\Schat$, $\mfs$ is modular in $\Scal$. Therefore,
as suggested by Eq.~\eqref{eq:mdef}, in order to minimize this modular upper bound $m$ with respect to a $k$-member set $\Scal$, we need to compute the last two terms, {\em i.e.}, $\alphaf f(i | \Schat \backslash \set{i})$ for all $i\in\Schat$ and,  
$ f(i | \emptyset)/\alphaf$ for all $i\not\in\Schat$; and finally, choose the $k$ smallest elements 
based on these quantities. 

\xhdr{The iterative procedure}
We summarize \our in Algorithm~\ref{alg:selcon}. 
%
Given the current estimate of the candidate set $\Schat$, \our 
computes $\alphaf\fest(i | \Schat \backslash \set{i})$ for $i\in \Schat$ in line \fmodular and  $\fest(i | \emptyset)/\alphaf$ for $i\not\in \Schat$ in line \smodular. The algorithm next picks the $k$ smallest  values in line \srtl to minimize $m$ and update  $\Schat$. 
Note that computation of $f$ here requires
an estimate of the model parameters $\wb$ and the Lagrangian multipliers $\mub$. However, a training algorithm might only provide a noisy or imperfect estimate of these parameters. Hence, we can only compute $\fest(\bullet)$, an imperfect estimate of $f(\bullet)$. 
Appendix~\ref{app:conv-prop} presents the convergence properties of \our.

\begin{algorithm}[t]
\small
\begin{algorithmic} [1]
\REQUIRE Training data $\Dcal$, $\lambda$, $\alphaf$, initial subset $\Scal_0$ of size $k$\; initial model parameters.
\STATE $\Schat \leftarrow \Scal_0$
\FORALL {$i \in \Dcal$}
   \STATE    $(\west, \muest), \fest(\set{i})  \leftarrow \displaystyle  \mathrm{Train}(F(\wb, \mub, \{i\}))$   
   \vspace{0.3mm}
\ENDFOR
\vspace{2mm}
\FOR{$l \in [L]$}
    \STATE  $ (\west, \muest), \fest(\Schat)\leftarrow \displaystyle  \Train( \dual(\wb, \mub, \Schat))$  
    \FORALL {$i \in \Schat$}
    \vspace{0.5mm}
        \STATE  $\fest(\Schat \setminus \set{i}) \leftarrow \Train(\dual(\wb, \mub, \Schat\cp\set{i}))$
        \STATE  $m[i] \leftarrow \alphaf[\fest(\Schat ) - \fest(\Schat \setminus \set{i})] $ 
    \ENDFOR
    \STATE For all $i \notin \Schat$, set $m[i] \ =   \fest(i\, | \, \emptyset)/\alphaf$  
    \STATE Pick the $k$ smallest elements from $\set{m[i] | }_{i\in \Dcal}$ to update $\Schat$
    \STATE $\Scal^{(l)}\leftarrow \Schat$ 
\ENDFOR
\STATE Return $\Schat$, $\west, \muest$
 \caption{\our{} Algorithm}  \label{alg:selcon} 
 \end{algorithmic}
\end{algorithm}
\vspace{-2mm}
\subsection{Approximation guarantee}
\label{eq:approxguarantee}
\vspace{-1mm}
We now show that \our{}   admits a bounded approximation guarantee in the case of both perfect and imperfect estimates of the parameters $(\west, \muest)$.

\xhdr{Results with perfect parameter estimates} In the following, we present our first result on the approximation guarantee (proof in Appendix~\ref{app:a0}). 
\vspace{-1mm}
\begin{theorem}\label{thm:a0}
If the training algorithm in Algorithm~\ref{alg:selcon} (lines \ftrain, \strain, \ttrain)
provides perfect estimates of the model parameters, it obtains a set $\widehat{\Scal}$ which satisfies:
\begin{align}
    f(\widehat{\Scal}) \leq \frac{k}{\af (1 + (k-1)(1 - \kf)\af)} f(\Scal^*)
\end{align}
where $\alphaf$ and $\kappaf$ are as stated in Theorem~\ref{thm:alpha-sub-nonlin} and   Proposition~\ref{prop:curv} respectively.
\end{theorem}
\todo{Add exact expression for the bound (approximately even) with the data and add a discussion on it.}
\xhdr{Results with imperfect parameter estimates}
 Data-driven training algorithms may not provide the optimal value of  model parameters, even if the underlying loss function is convex. Therefore, in practice, \our can only access an imperfect estimate of $\widehat{\wb}, \widehat{\mub}$ in lines \ftrain, \strain and \ttrain. Submodular and weakly submodular optimization in the presence of imperfect estimates has been widely studied in literature~\cite{Qian2017a,el2020optimal,Hassidim2018,Hassidim2017,Horel2016,Singla2016}. However, to the best of our knowledge, they do not tackle the problem of cardinality-constrained minimization of an $\alpha$-submodular function.
In this context, a notable contribution of our work is that, \our also enjoys a relaxed approximation guarantee in these cases, which renders it practically useful. We formally state the result as follows (proven in Appendix~\ref{app:a1}).
\vspace{-1mm}
\begin{theorem}\label{thm:a1}
If the training algorithm (lines \ftrain, \strain, \ttrain) in Algorithm~\ref{alg:selcon} 
provides imperfect estimates, so that
$$\bnm{F(\west,\muest,\Scal)- F(\wb^*(\mub^*(\Scal),\Scal),\mub^*(\Scal),\Scal)} \le \epsilon\quad \text{for any}\  \Scal,$$
then Algorithm~\ref{alg:selcon} obtains a set $\widehat{\Scal}$ that satisfies:
\begin{align}
    f(\widehat{\Scal}) \leq \left(\dfrac{k}{\af(1 + (k-1)(1 - \kf)\af)} +  \frac{2k\epsilon}{\ell}\right) f(S^*),\nn
\end{align}
where $\ell= \min_{i\in\Dcal}\min_{\wb} \lambda||\wb||^2 + (y_i-h_{\wb}(\xb_i))^2 $, $\alphaf$ and  $\kappaf$ are obtained in Theorem~\ref{thm:alpha-sub-nonlin} and Proposition~\ref{prop:curv}, respectively.
\end{theorem}
\vspace{-1mm}

\xhdr{Discussion on the approximation ratio} A trite calculation shows that, for the regime of $\lambda$ defined in Theorem~\ref{thm:alpha-sub-nonlin} and a small value of $\epsilon$, the approximation ratio of \our is $ O(y_{\max} ^4 / y_{\min}^4)$. While such a ratio may appear to be  conservative, there are several applications such as house price prediction or stock prediction, where $y_{\max} / y_{\min}$ may not be too high. Apart from that, one can always pre-process the dataset by adding an offset to $y$ and augmenting a constant in the feature $\xb$, to control this ratio, as illustrated  in Appendix~\ref{app:real}.
\todo{A reviewer might be extremely tempted to ask if we have performed even a single experiment in which we pre-process the dataset by adding an offset to $y$ and augmenting a constant in the feature $\xb$, to control the approximation ratio.}
Moreover, since our approximation ratio holds for any monotone $\alpha$-submodular function with bounded curvature,
it can be of independent technical interest.


\vspace{-3mm}
\section{Experiments}
\vspace{-1mm}
\label{sec:real}
 \newcommand{\craig}{CRAIG\xspace}
\newcommand{\glister}{GLISTER\xspace}
\newcommand{\ccr}{Community-and-crime\xspace}
\vspace{-1mm}

In this section, we present experimental results and analysis on several real-world datasets to evaluate the performance of \our
against several competitive baselines. Thereafter, we show that our framework is also practically useful in a fair regression setup, where the validation loss bounds are used to ensure that the error for each protected group is below an acceptable level of threshold. Appendix~\ref{app:real} contains additional experiments.
\vspace{-1mm}
\subsection{Experimental setup}
\vspace{-0.5mm}
\label{sec:exp-setup}
\xhdr{Datasets} We experiment with five real world datasets, {\em viz.}, Cadata (16718 instances), Law~(20800 instances),
NYSE-High (701348 instances),
NYSE-Close (701348 instances),
 and \ccr (1994 instances), all briefly described in Appendix~\ref{app:expsetup}.

\xhdr{Baselines} We compare \our \xspace against seven baselines. 
(1)~\emph{\Full:} It uses full data for training without any validation error constraint.
(2)~\emph{\Fullc:} It uses full data for training, subject to the same validation error constraints 
used in \our.
(3)~\emph{\Random:} It samples a training subset uniformly at random, but it does not employ any constraint on validation set.
(4)~\emph{\Randomc:} It is the same as \Random, except that it uses the constraints on validation set errors. 
(5)~\emph{CRAIG~\cite{mirzasoleiman2019coresets}:} This is a coreset based data selection method, that however, does not use any constraint on the validation set.
(6)~\emph{GLISTER~\cite{killamsetty2020glister}:} This is a data selection method that uses validation set to fine tune the trained model, which however, does not pose any explicit constraint on the validation set error.
(7)~\emph{\our-without-constraint:} Here, we solve the optimization problem~\eqref{eq:opt-soft}, without the validation error constraints.

\xhdr{Implementation details} In Algorithm~\ref{alg:selcon}, if we set the number of epochs for  $\Train(\;)$ in line~\strain to $T$,  this training routine runs for $N = LT$ epochs, where $L$ is the number of iterations of the for-loop (lines \sloop--\eloop).
To make a fair comparison, we used the same number of epochs $N$ and the same batch size $b$ across all baselines and \our for training the underlying model. 
Specifically, we set $N=2000$ for Cadata and  Law , $N=5000$ for the NYSE datasets; and, $b=\min\set{|\Scal|,1000}$ across all datasets.
Additionally, \our involves two more sets of  small scale optimization problems (lines \ftrain and  \ttrain respectively), where we set the number of epochs as $3$. Moreover for the optimization of $f(\Scal\backslash \set{i})$ in line \ttrain, we use the same batch size $b=\min\set{|\Scal|,1000}$ as stated earlier.  
In each experiment, we used (random) 89\% 
training, 1\% validation and 10\% test folds. We employed pytorch with the adam optimizer for all experiments.  Further details about the implementation are provided in Appendix~\ref{app:expsetup}.
\begin{figure*}[t]
\centering
{ \includegraphics[width=0.9\textwidth]{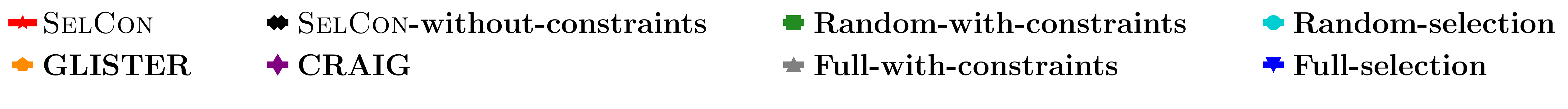}}\\ \vspace*{-5mm}
\hspace*{-0.7cm}
\subfloat{ \includegraphics[trim={0 40 0 0}, clip,width=0.20\textwidth]{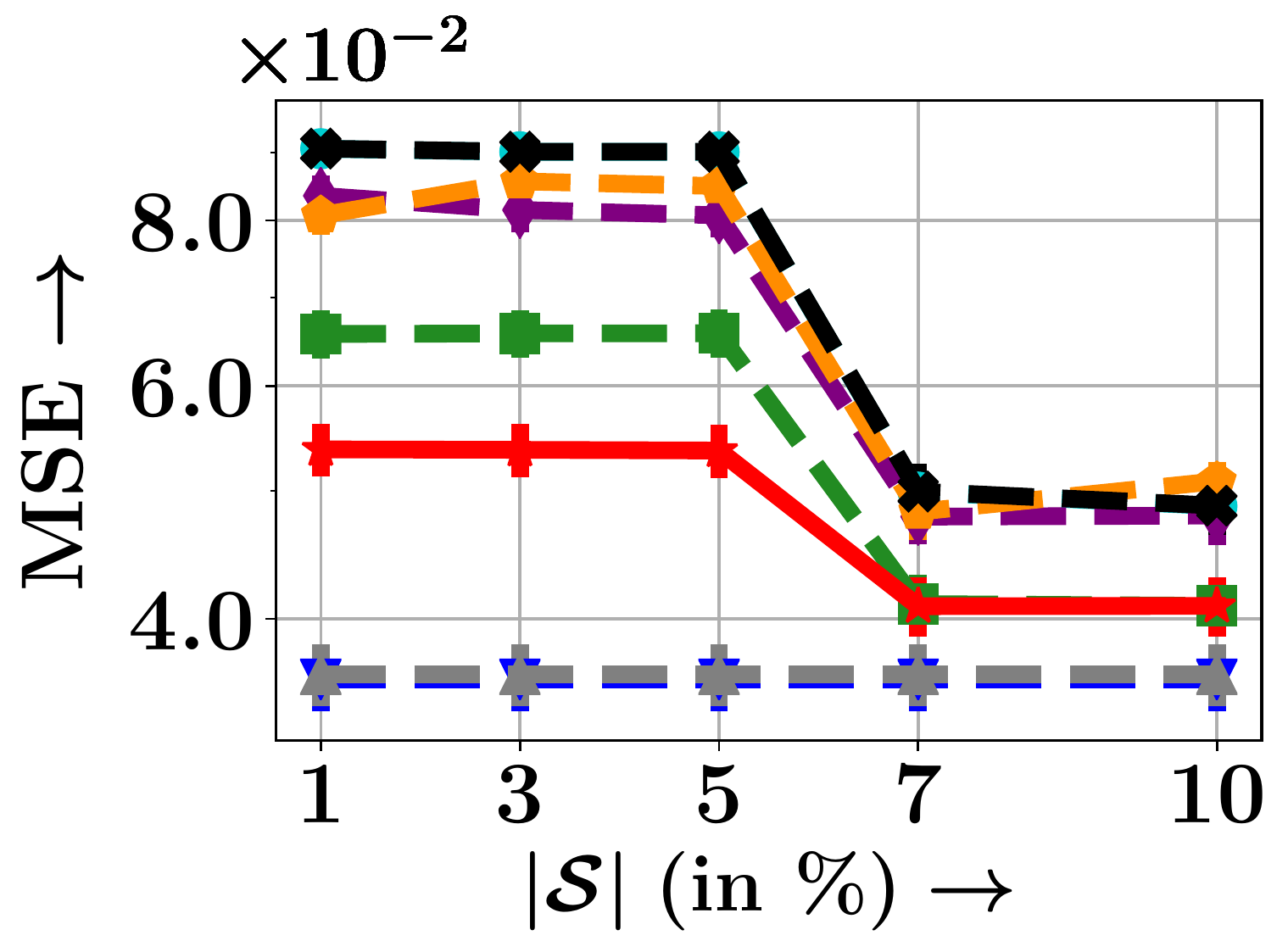}}\hspace{2mm}
\hspace*{0.2cm}\subfloat{\includegraphics[trim={0 40 0 0}, clip,width=0.20\textwidth]{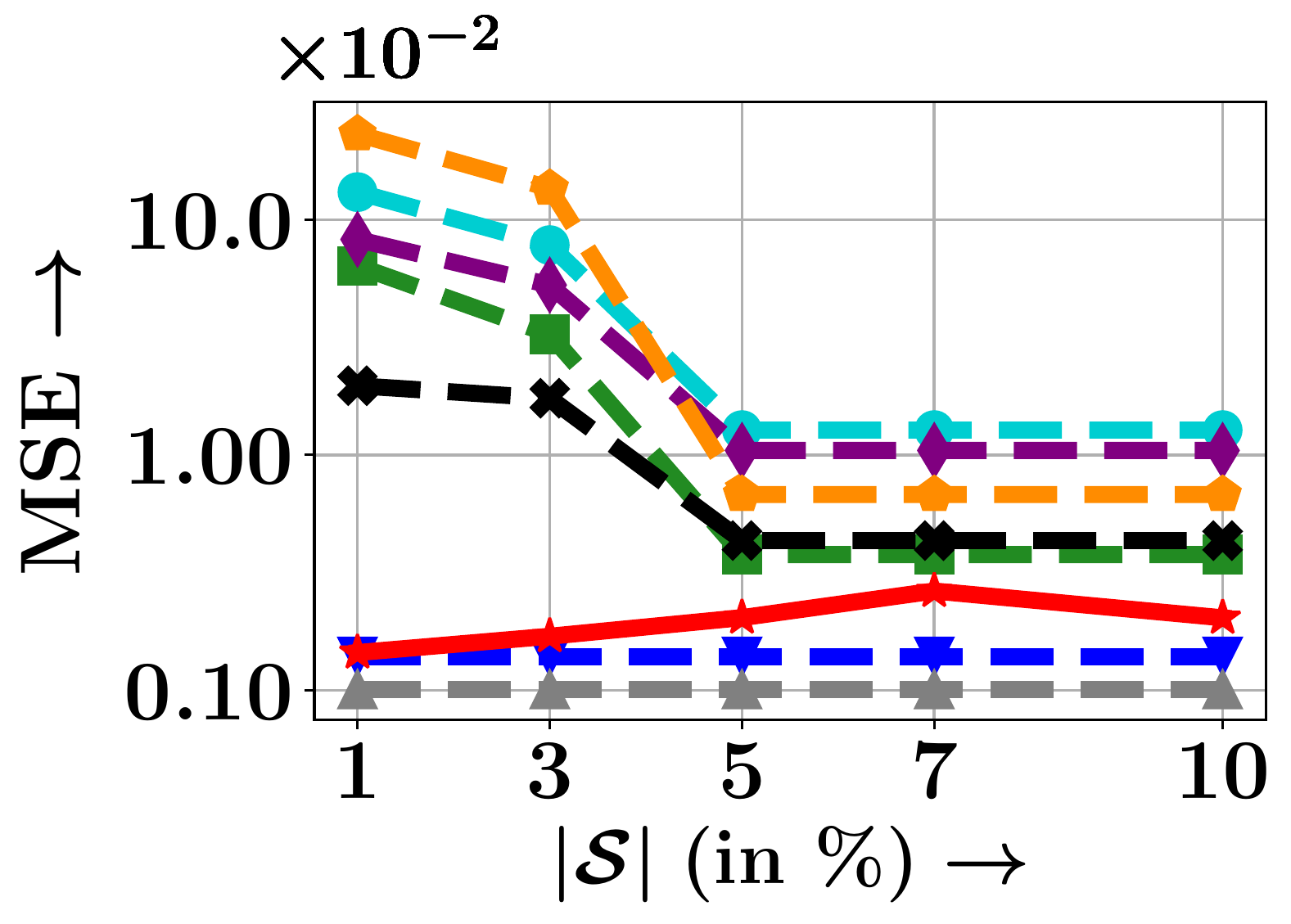}}\hspace{2mm}
\hspace*{0.2cm}\subfloat{ \includegraphics[trim={0 40 0 0}, clip,width=0.20\textwidth]{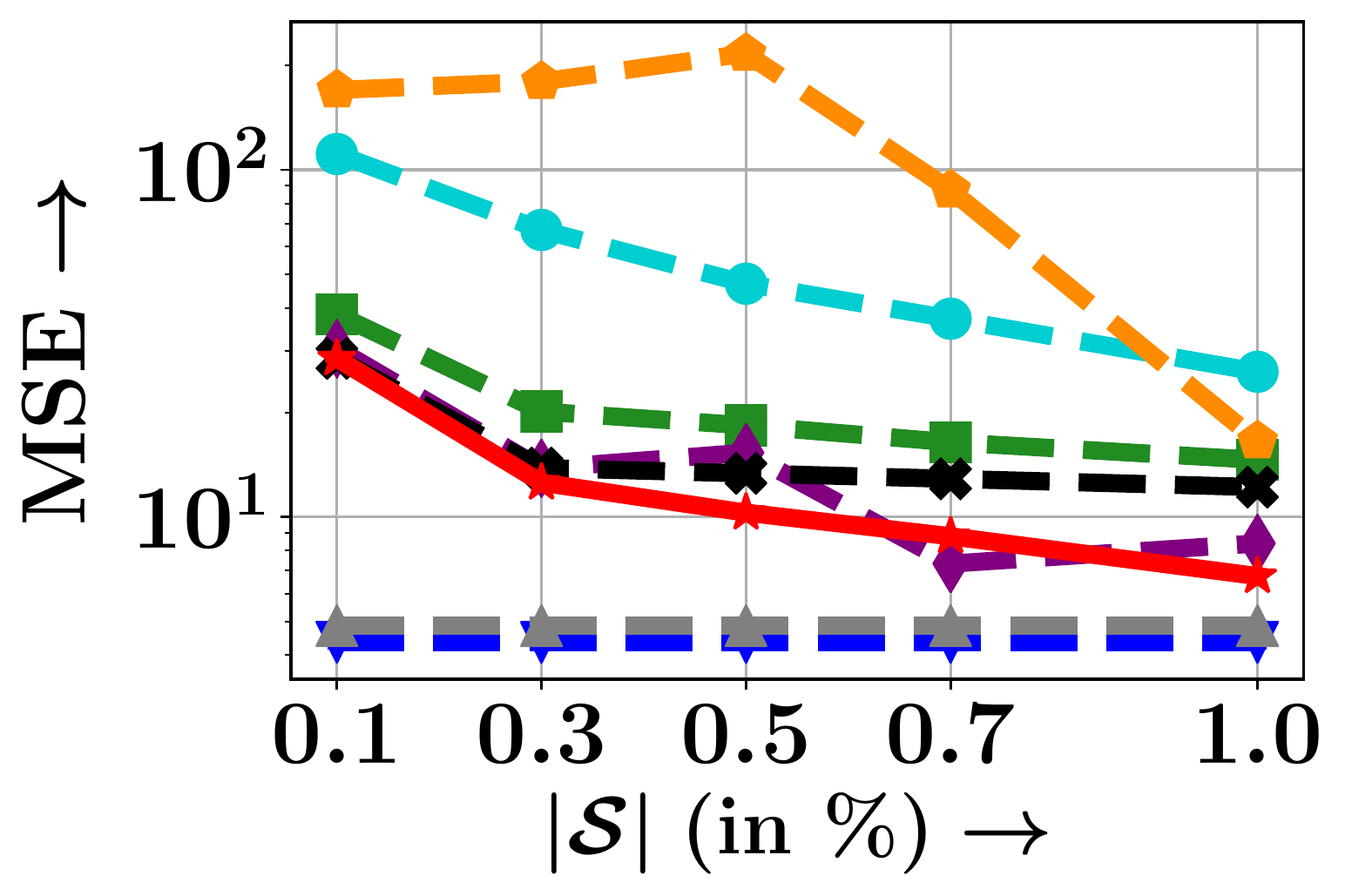}}\hspace{2mm}
\hspace*{0.2cm}\subfloat{ \includegraphics[trim={0 40 0 0}, clip,width=0.20\textwidth]{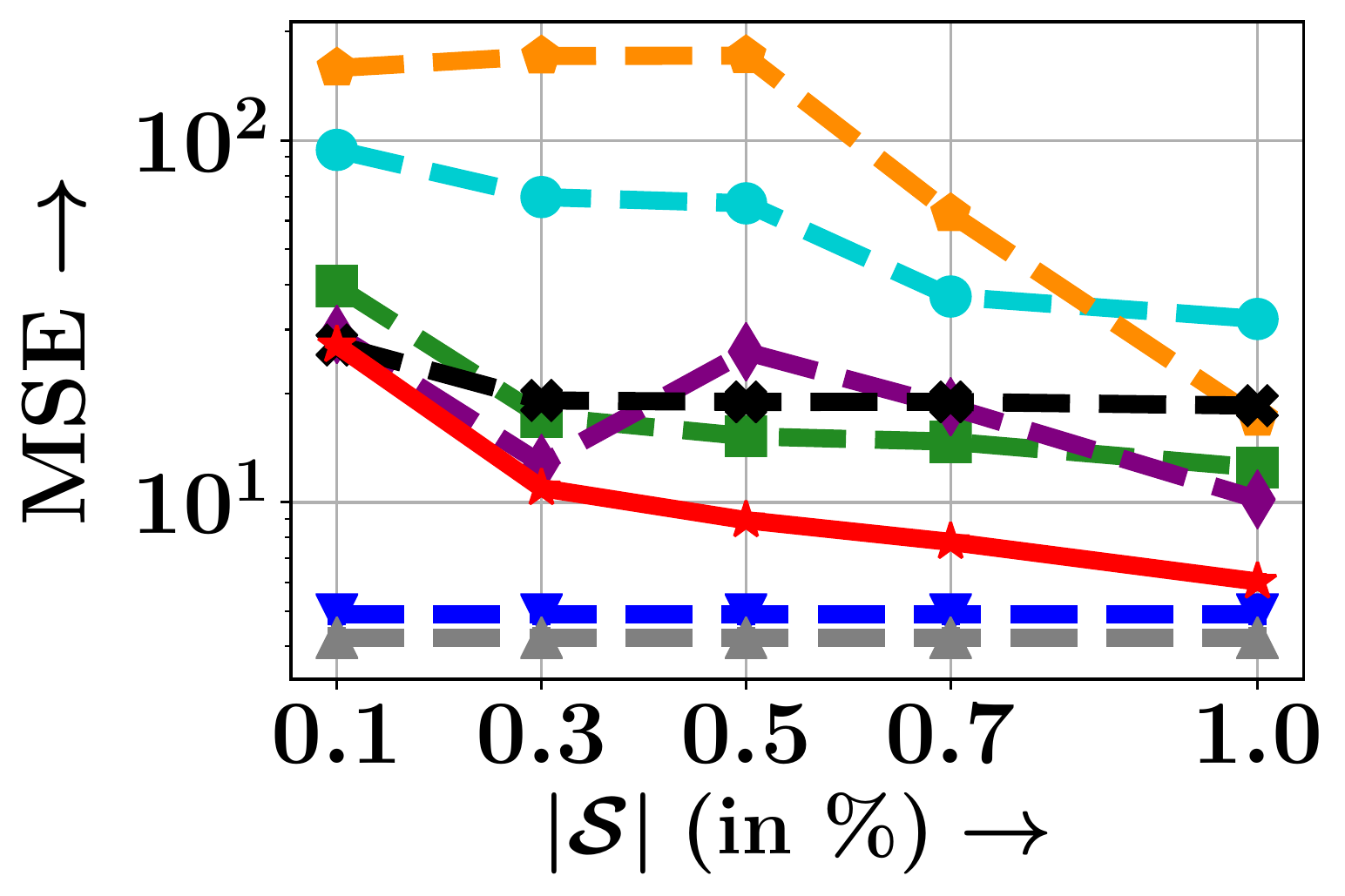}}\\
\hspace*{-0.1cm}\hspace*{-0.6cm}\subfloat[Cadata]{\setcounter{subfigure}{1} \includegraphics[ width=0.20\textwidth]{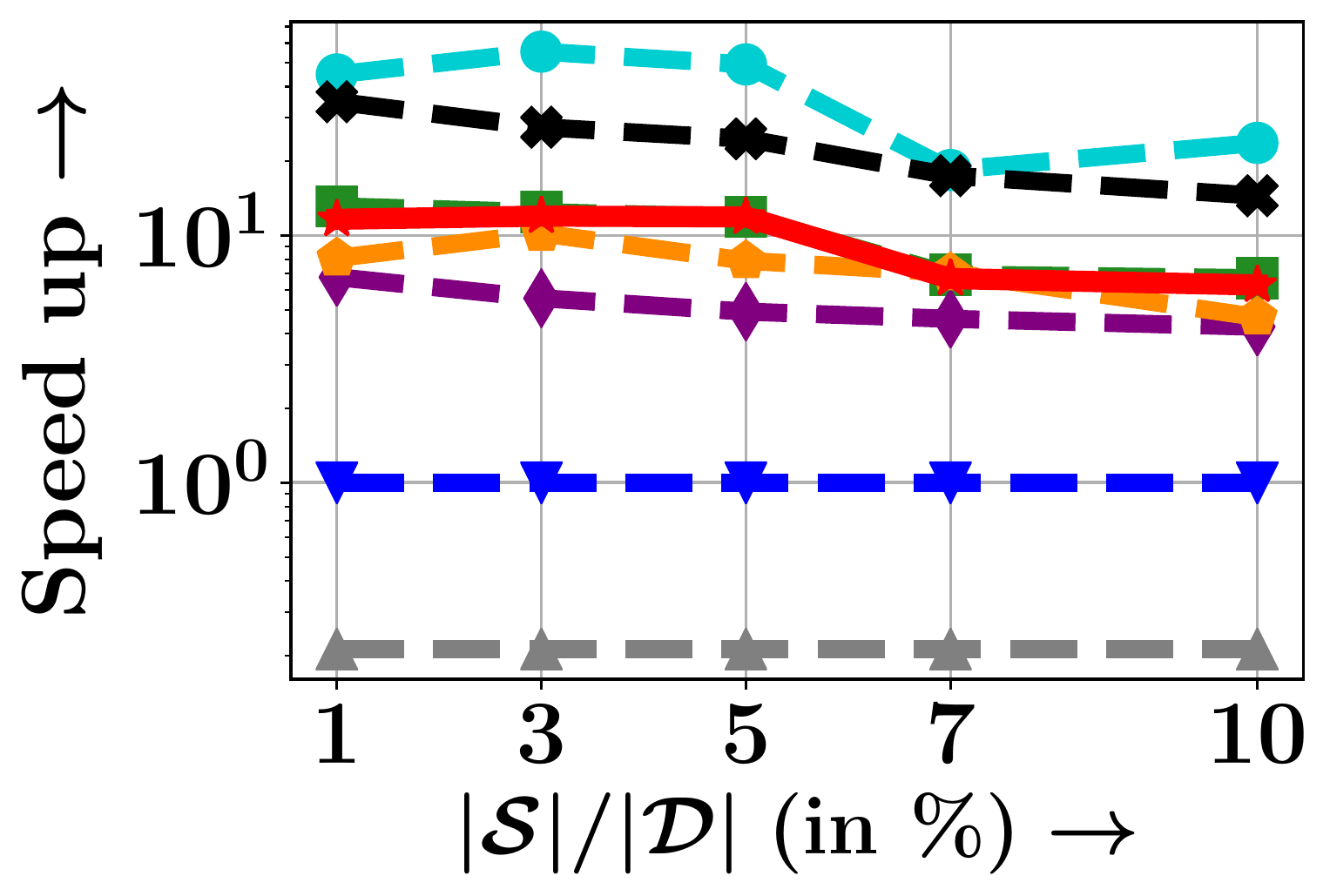}}\hspace{2mm}
\hspace*{0.2cm}\subfloat[Law]{\includegraphics[width=0.20\textwidth]{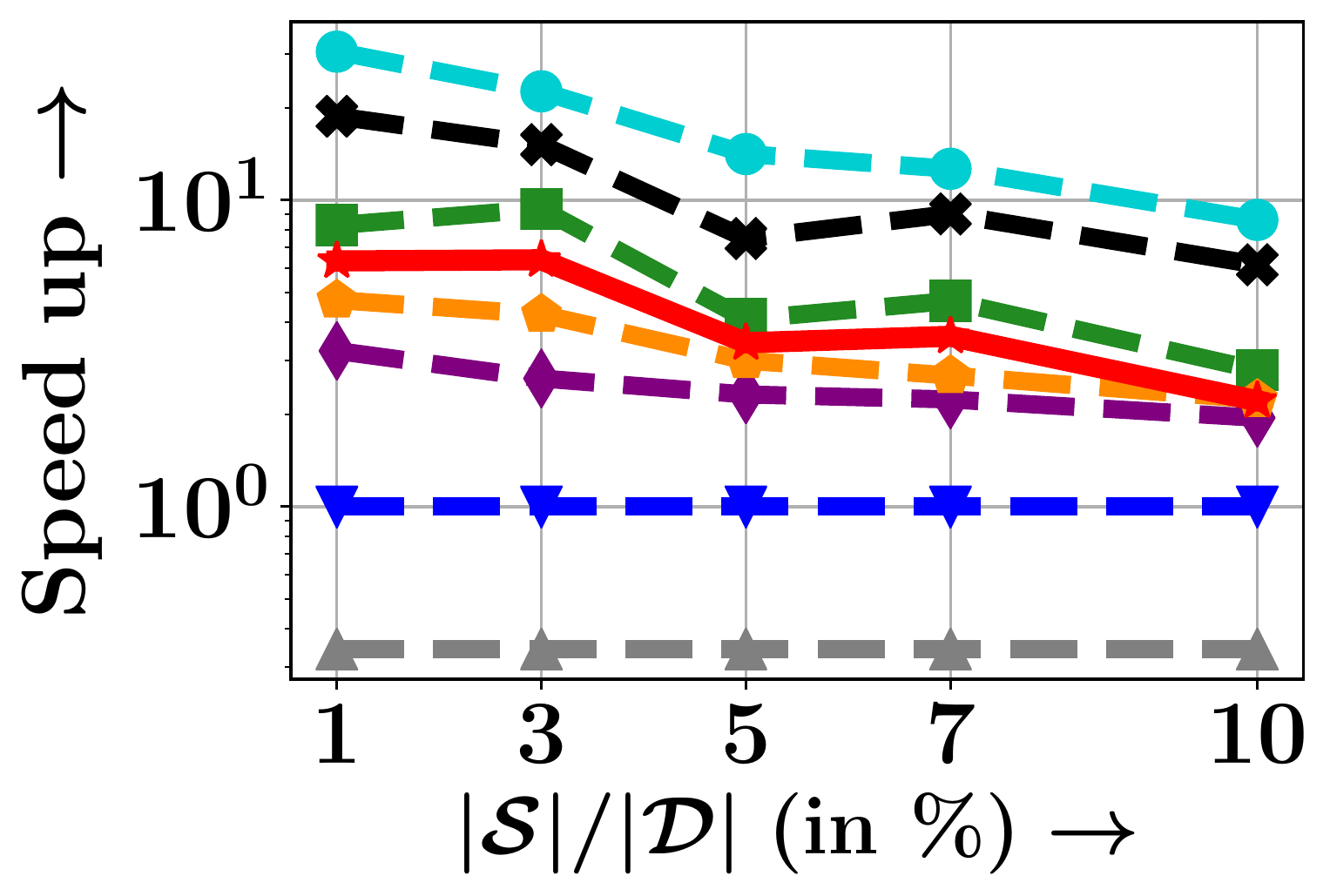}}\hspace{2mm}
\hspace*{0.2cm}\subfloat[NYSE-High]{ \includegraphics[width=0.20\textwidth]{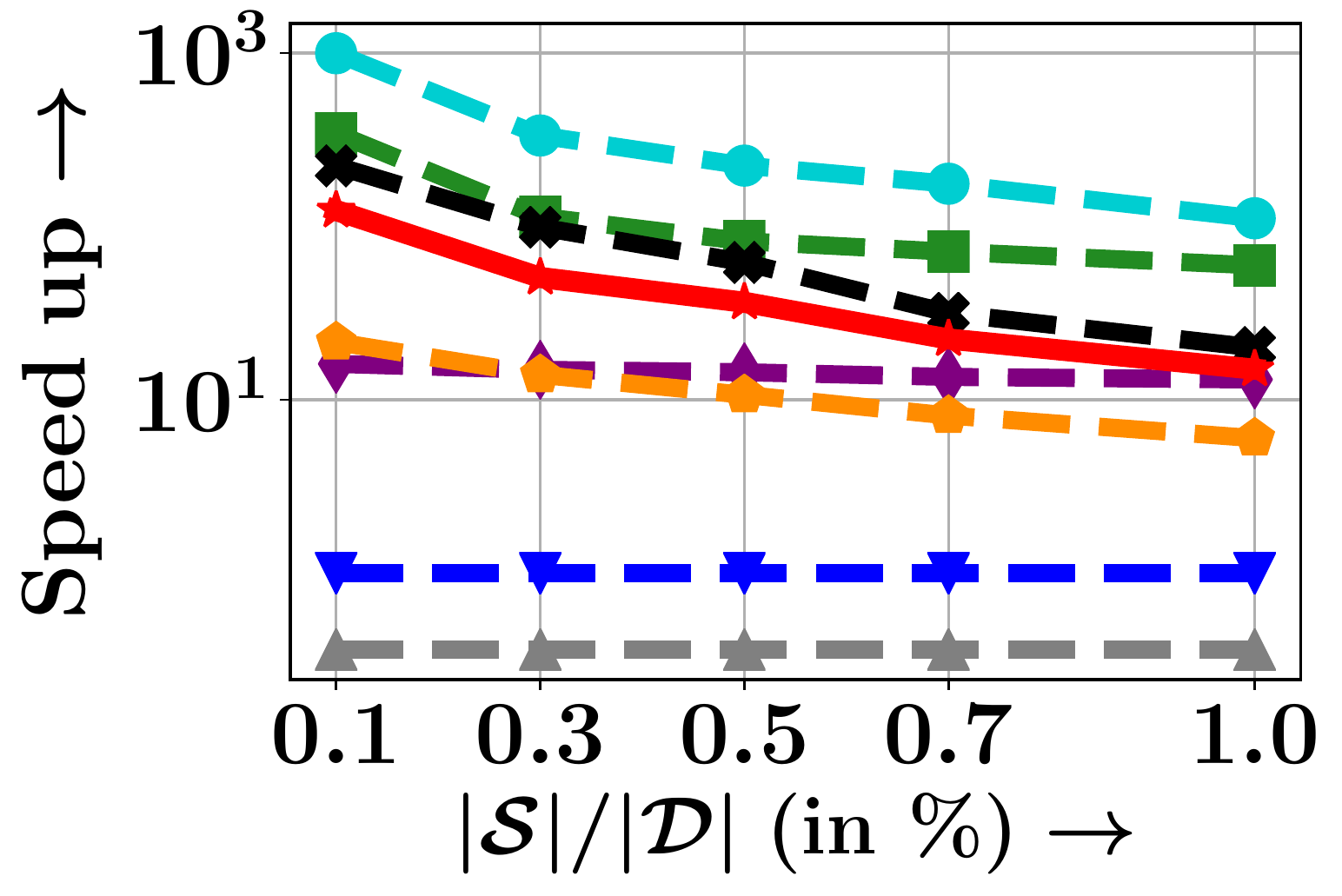}}\hspace{2mm}
\hspace*{0.2cm}\subfloat[NYSE-Close]{ \includegraphics[width=0.20\textwidth]{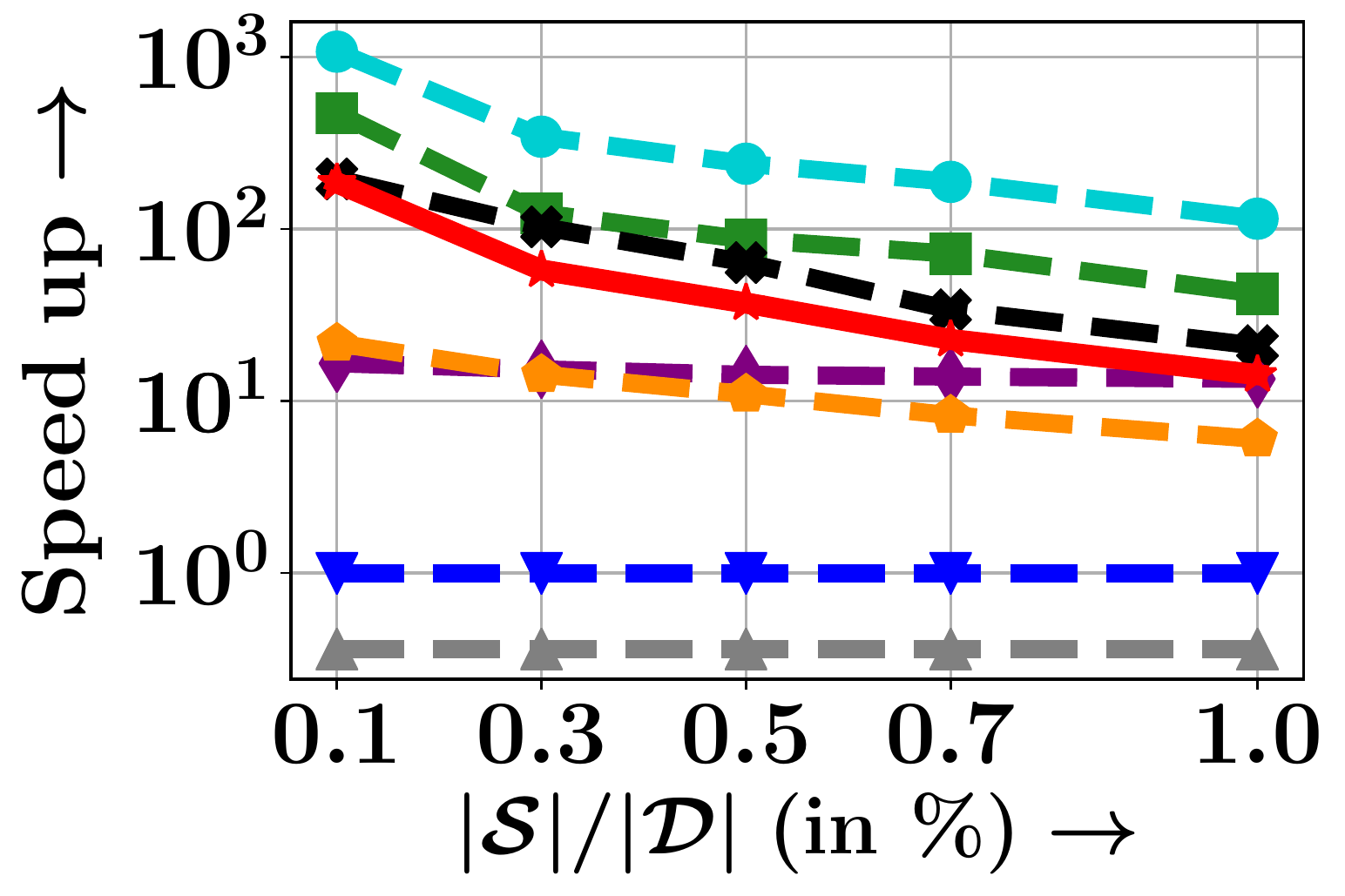}}
\caption{Variation of performance in terms of the mean squared error (MSE, top row) and the computational 
efficiency in terms of speed up  with respect to \full (bottom row) for all methods, {\em i.e.}, \our (Algorithm~\ref{alg:selcon}), \our-without-constraints, \random, \randomc, \full, \fullc, CRAIG~\cite{mirzasoleiman2019coresets} and GLISTER~\cite{killamsetty2020glister} across different datasets with $10\%$ held-out set and 1\% validation set. 
We set  the number of partitions $Q=1$. }
\label{fig:main}
\end{figure*}
\begin{figure*}[!!!t]
\centering
{ \includegraphics[width=0.96\textwidth]{FIG/Acc_vs_S/leg.pdf}}\\ \vspace*{-5mm}
\subfloat{ \includegraphics[width=0.22\textwidth]{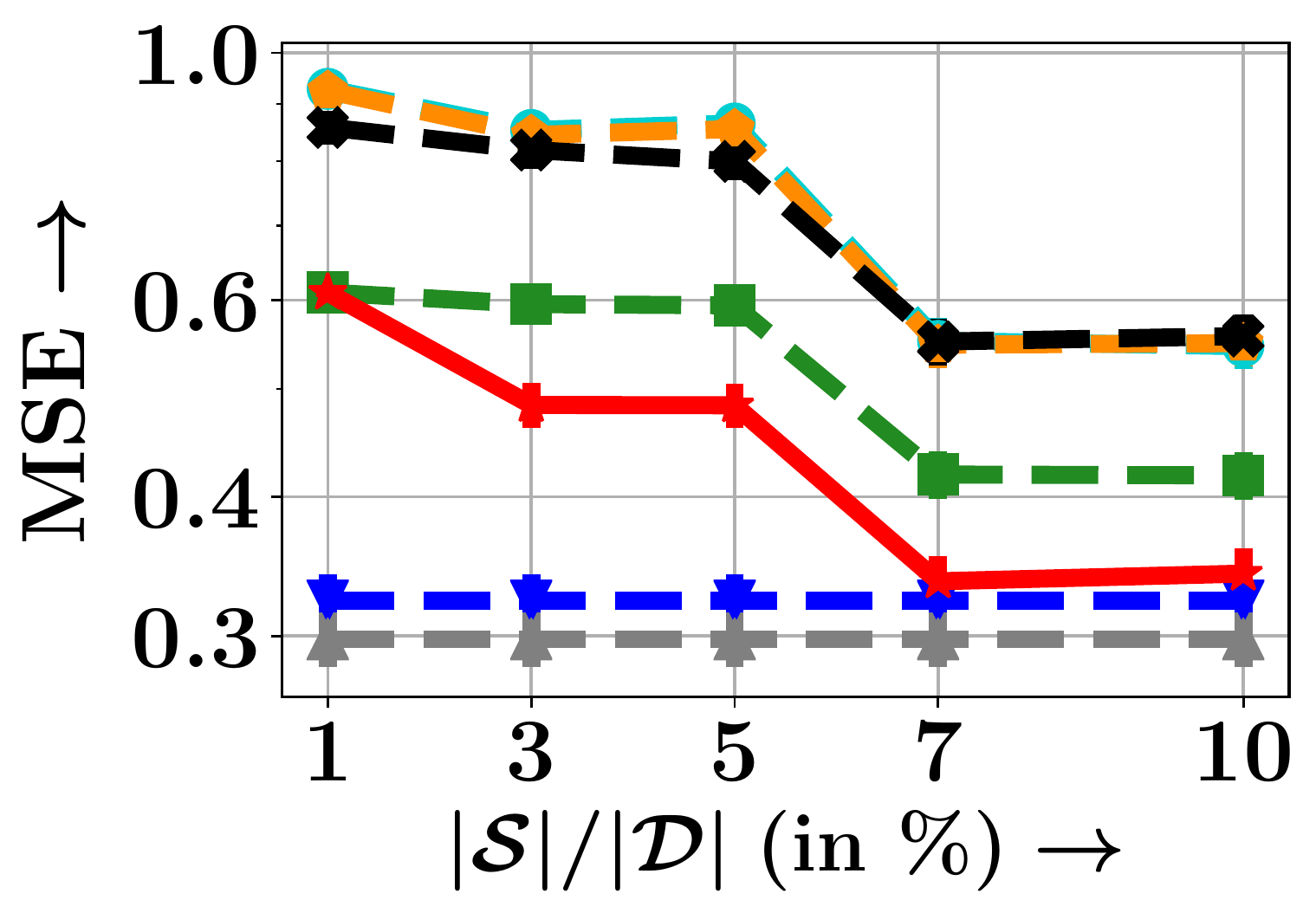}}\hspace{2mm}
\subfloat{\includegraphics[width=0.22\textwidth]{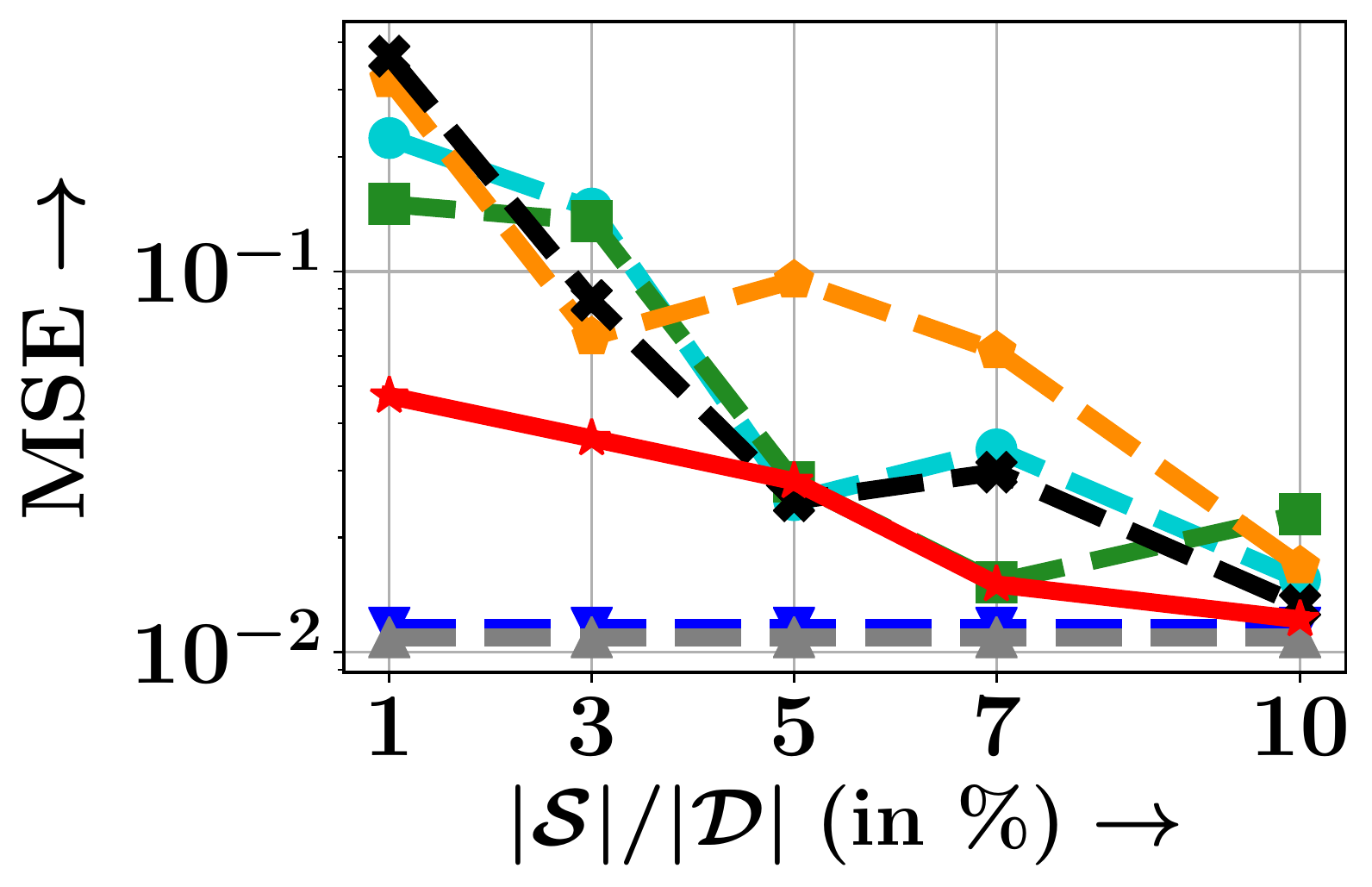}}\hspace{2mm}
\subfloat{\includegraphics[width=0.22\textwidth]{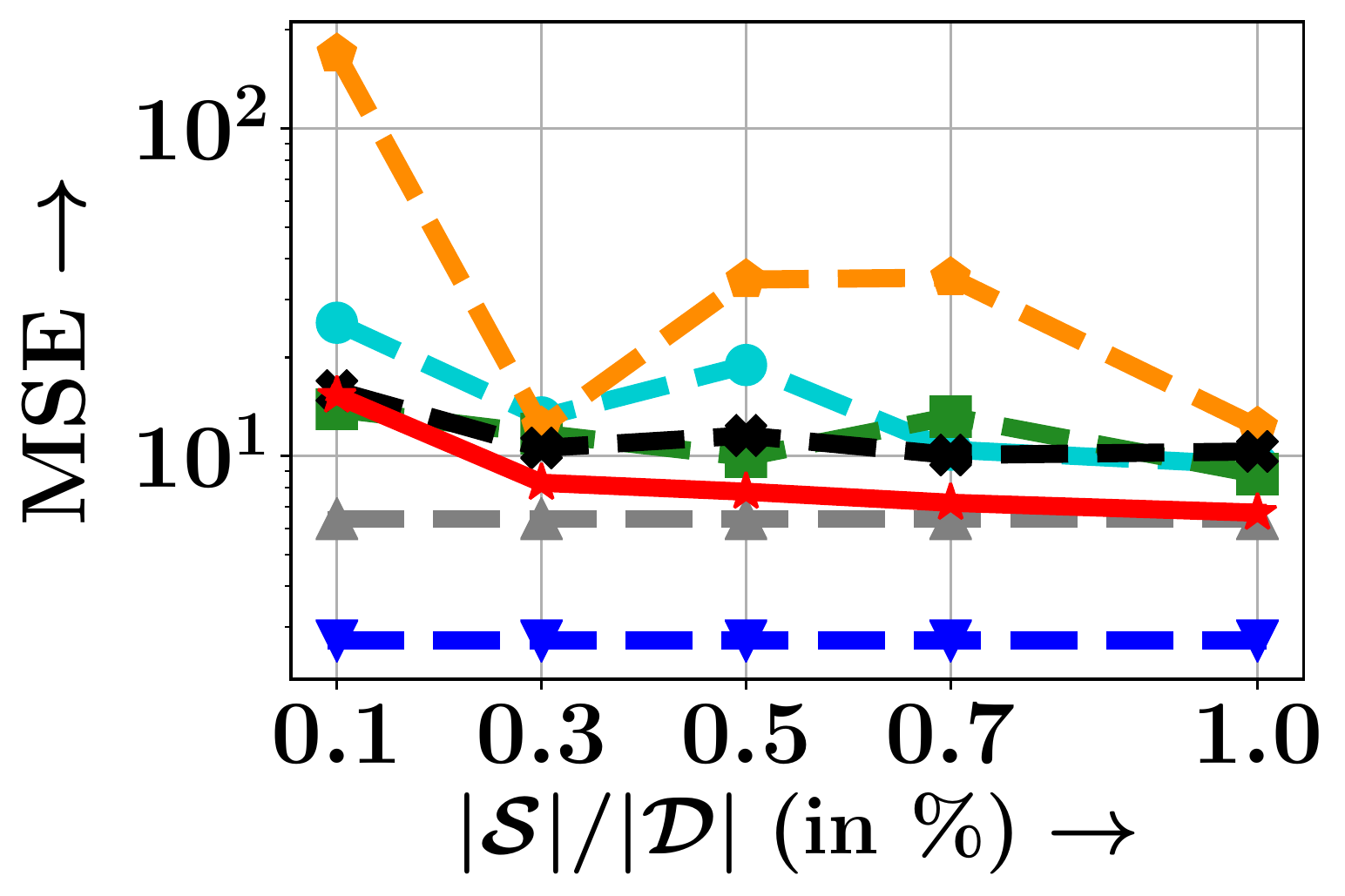}}\hspace{2mm}
\subfloat{ \includegraphics[,width=0.22\textwidth]{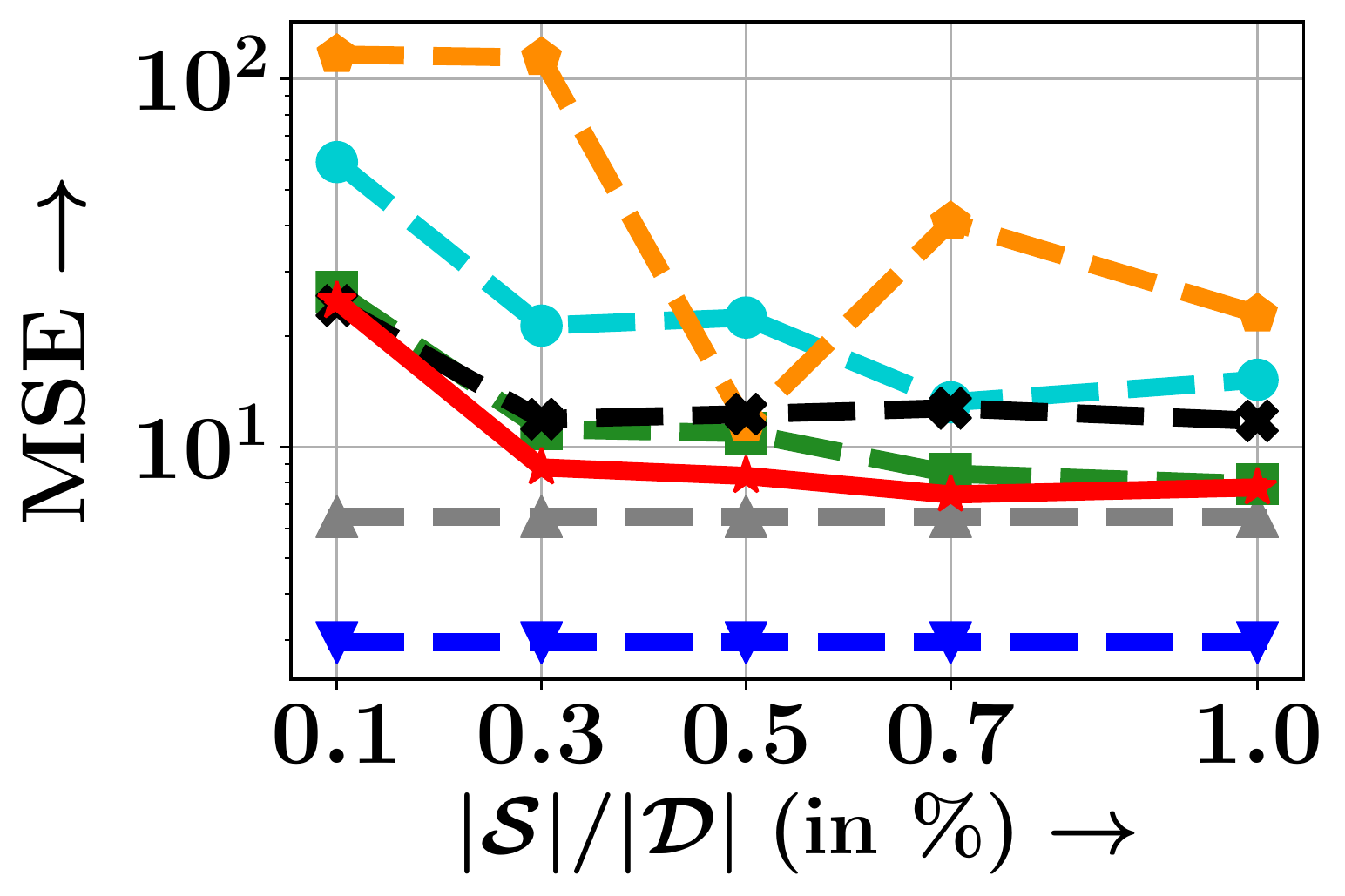}}\\
\vspace{3mm}
\subfloat[Cadata]{\setcounter{subfigure}{1}\includegraphics[ width=0.22\textwidth]{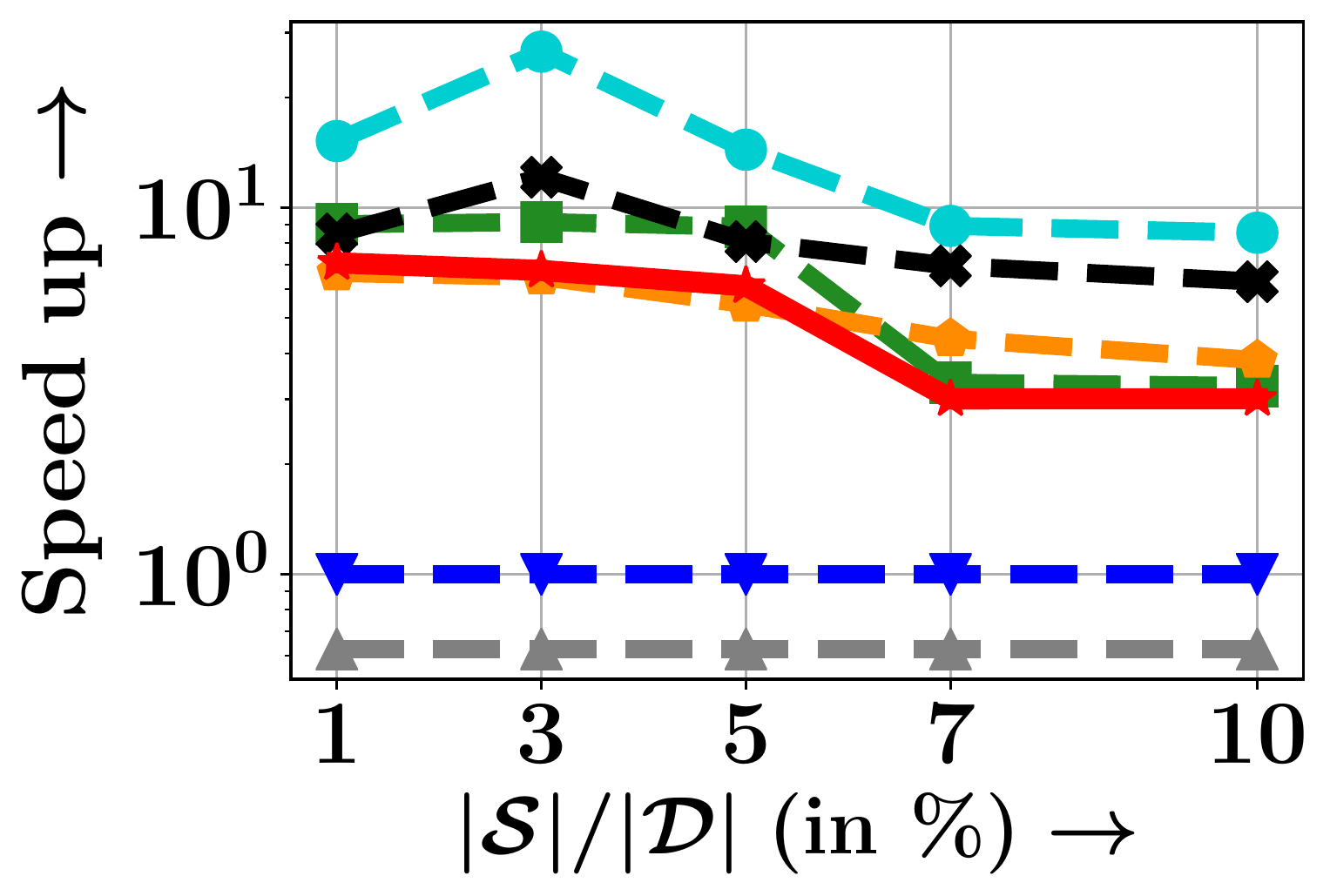}}\hspace{2mm}
\subfloat[Law]{\includegraphics[width=0.22\textwidth]{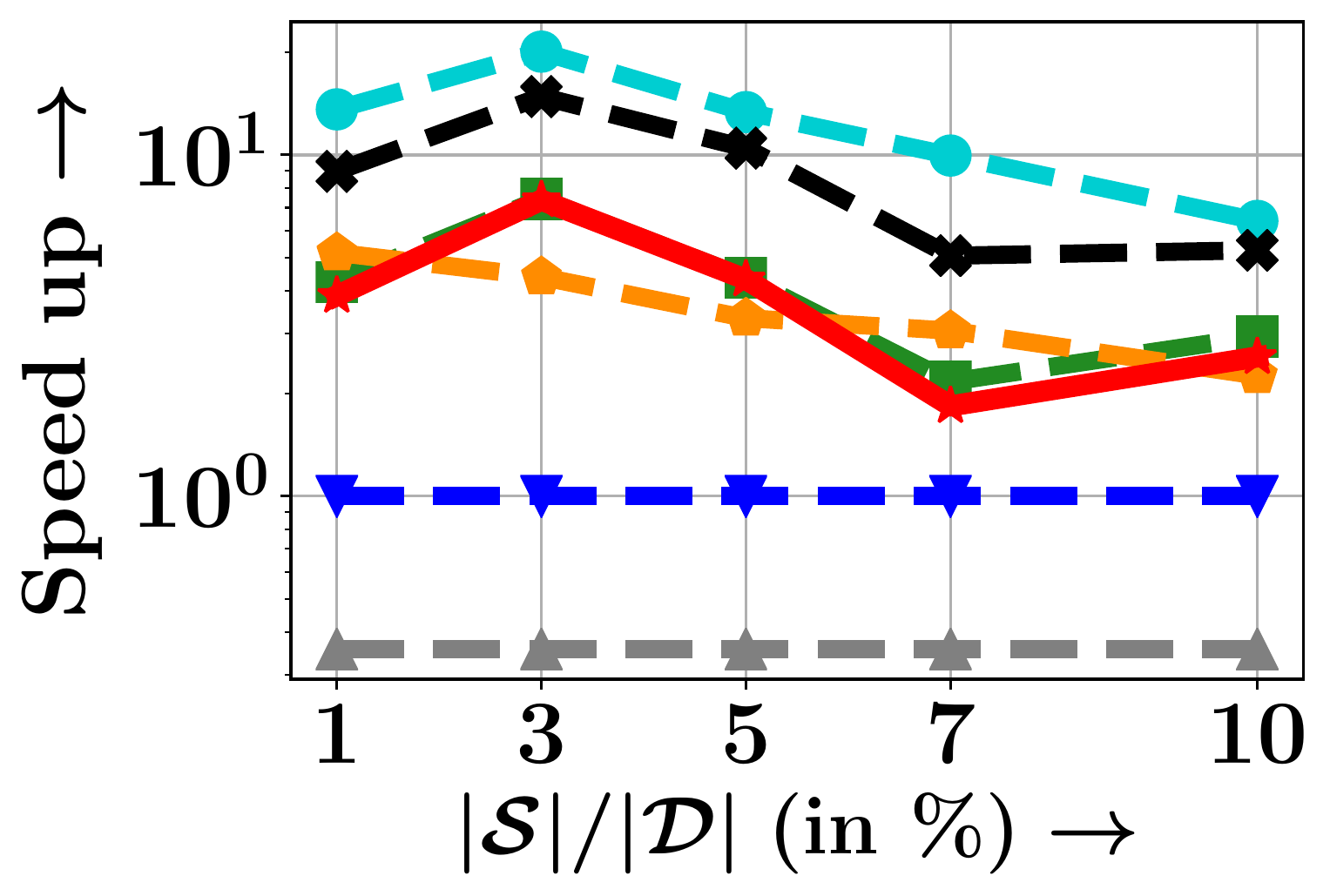}}\hspace{2mm}
\subfloat[NYSE-High]{\includegraphics[width=0.22\textwidth]{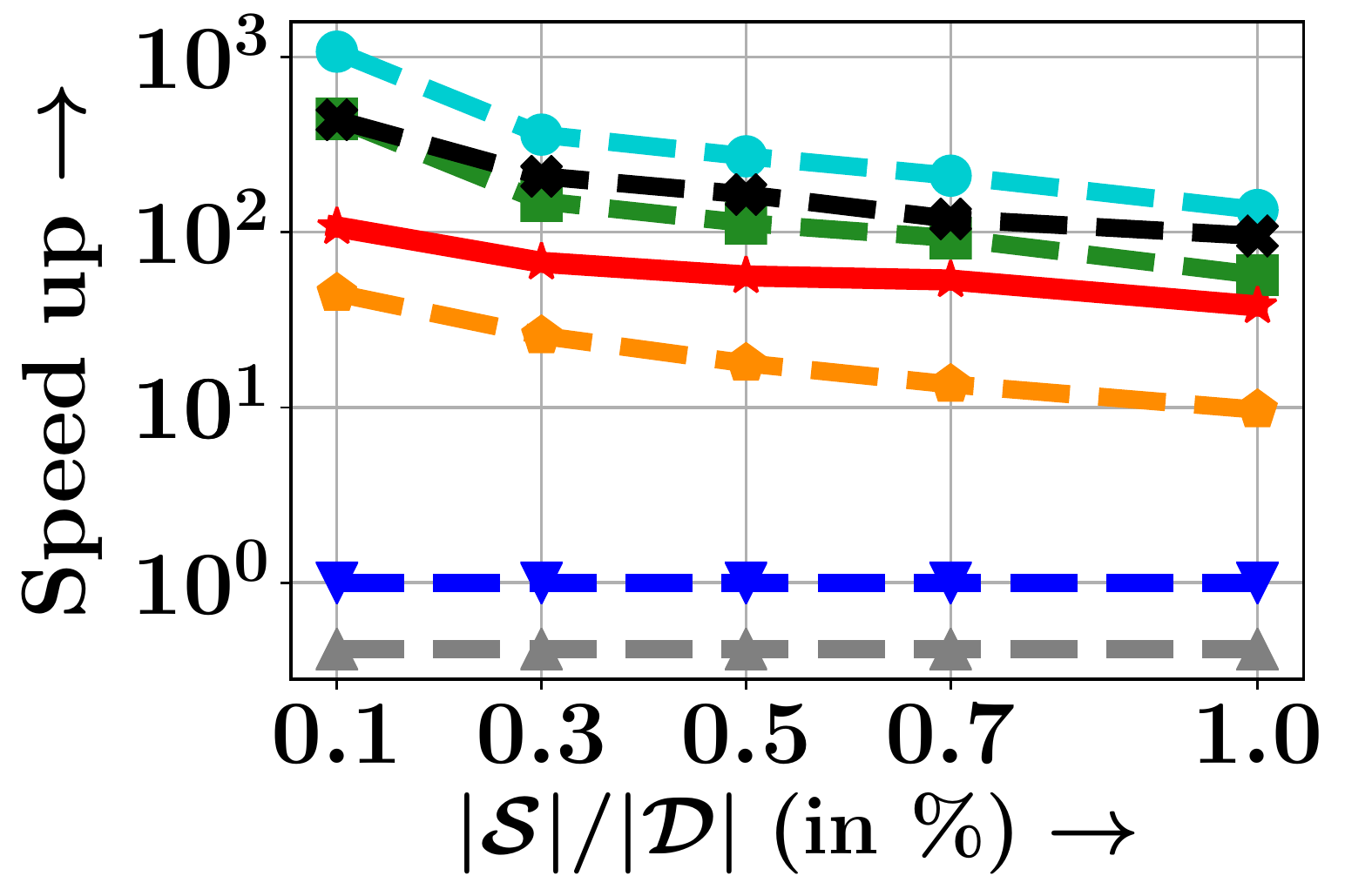}}\hspace{2mm}
\subfloat[NYSE-Close]{\includegraphics[width=0.22\textwidth]{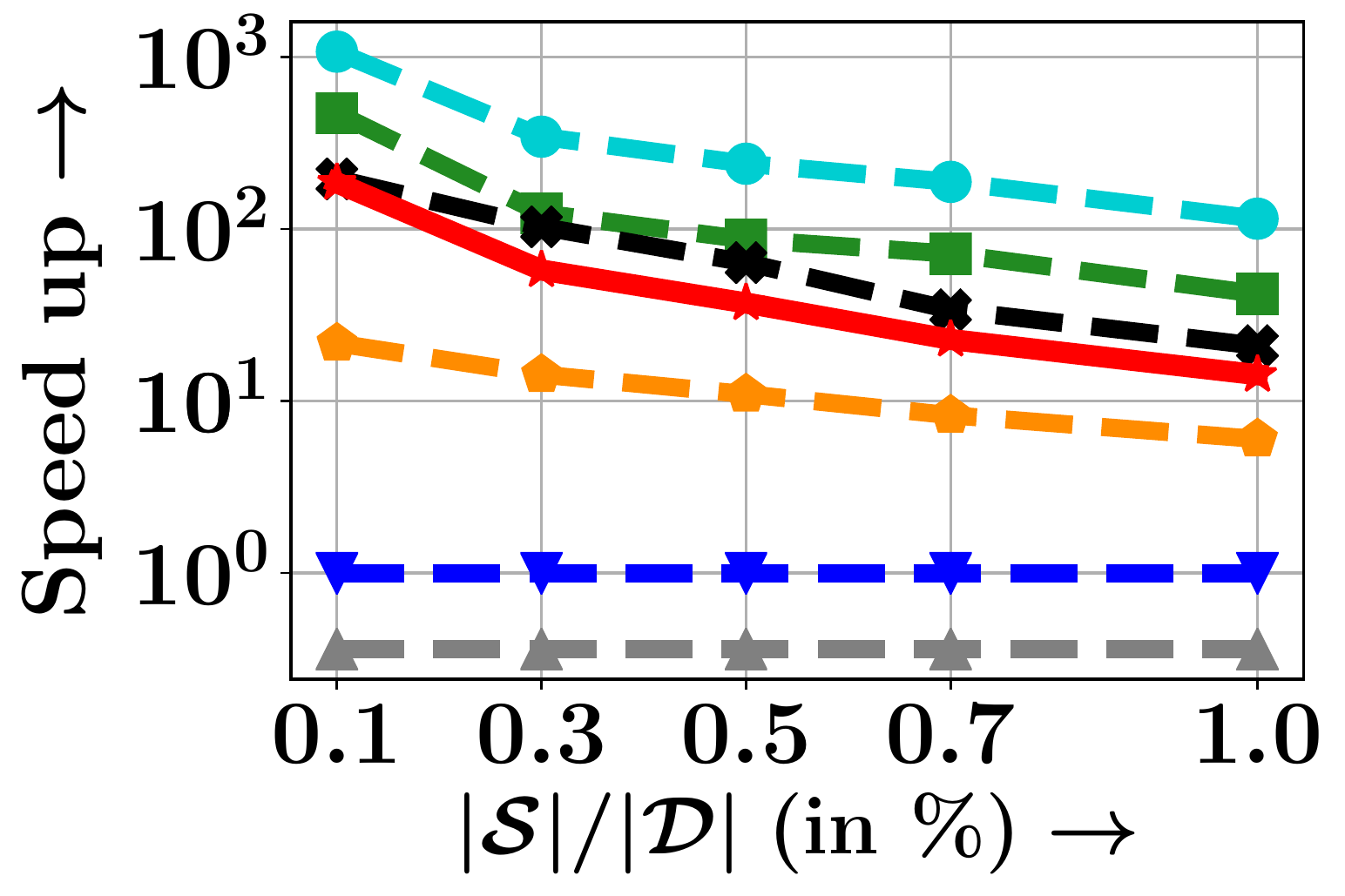}}
\caption{Variation of performance for nonlinear regression using $\hb_w(\xb) = \wb ^\top \text{ReLU}(\bm{W}\xb)$, in terms of the mean squared error (MSE, top row) and the computational efficiency in terms of speed up  with respect to \full (bottom row) for all methods, {\em i.e.}, \our (Algorithm~\ref{alg:selcon}), \our-without-constraints, \random, \randomc, \full, \fullc, CRAIG~\cite{mirzasoleiman2019coresets} and GLISTER~\cite{killamsetty2020glister} across different datasets with $10\%$ held-out set and 1\% validation set. 
We set  the number of partitions $Q=1$.}
\label{fig:deep}
\end{figure*}

\oursub{Predictive performance and  efficiency}
\label{sec:exp-main}
\vspace{-0.5mm}
We evaluate the performance of each data selection method in terms of the mean squared error (MSE) $\EE[(y-\hat{y})^2]$ on
 the test set.
 We also compute the computational efficiency of a method in terms of the speed-up it achieves with respect to \full, {\em i.e.}, $\text{RunTime}_{\text{\Full}}/\text{RunTime}_{\text{method}}$, where $\text{RunTime}_{\bullet}$ is time taken by the corresponding method to complete both the subset selection and model training.  
Here, we constrain the total loss on the validation set, {\em i.e.}, we set $Q=1$.

\xhdr{Linear regression}
Here, we compare the performance of \our for linear regression ($h_{\wb}(\xb)=\wb^\top \xb$) against all the baselines across the first four datasets\footnote{\scriptsize{Due to its small size, we ignore \ccr in this experiment.}}, described in Section~\ref{sec:exp-setup}. 
Moreover, for smaller datasets, {\em i.e.}, Cadata and Law, we consider $|\Scal|/|\Dcal| \in [0.01, 0.1]$, whereas, for larger datasets, {\em i.e.}, NYSE-High and NYSE-Close, we consider $|\Scal|/|\Dcal|\in [0.001,0.01]$.
In Figure~\ref{fig:main}, we summarize the results.  We make the following observations.
(i) \our shows better predictive accuracy than all the baselines except \full and \fullc in most of the cases.
The performance of \randomc is often comparable with \our especially when  $|\Scal|/|\Dcal|$ is too high ( $> 5\%$ in  Cadata and Law) or too low ($<0.3\%$ in NYSE datasets). On the Law dataset, \our's performance is noteworthy -  with 1\% training data, it performs at par with \Full.
In most cases, the performance gain provided by \our over \randomc is statistically significant (Wilcoxon signed-rank test, p-value $= 0.05$) while \our consistently outperforms the other baselines. 
 %
(ii) \our shows a significant speed up with respect to \fullc, \full, \glister and \craig.
In fact, with 1\% subset size, \our shows a 10$\times$ speed up with respect to \full,  often with negligible loss in accuracy (see Law and NYSE-Close).
%
%
However, \our is slower than \random, \randomc and \ourcon. This is because \ourcon does not have any validation loss constraints; and, none of the random heuristics involves any additional overhead time due to subset selection.
(iii) \craig and \glister do not involve any explicit validation set constraints, which often curbs their predictive power. 
On the other hand,  even \randomc is able to outperform them in terms of
the predictive performance, which is because of their improved generalization ability due to the presence of the explicit validation error constraints.
%
%

\xhdr{Nonlinear regression} Next, we analyze the performance and efficiency of \our, when $\hb_w(\xb) = \wb ^\top \text{ReLU}(\bm{W}\xb)$.
In Figure~\ref{fig:deep}, we summarize the results\footnote{\scriptsize{We omitted the results for CRAIG in nonlinear regression because the data selection component of  CRAIG needs to be run for several epochs for non-convex losses~\cite{mirzasoleiman2019coresets}, and hence, it did not scale for the large datasets.}} which shows that \our can trade off between 
efficiency and performance more effectively than the baselines  (results similar to  linear regression). 

\xhdr{Effect of $\delta$}
We next investigate the effect of $\delta$ on the predictive performance for different sizes of $|\Scal|$.
In Figure~\ref{fig:delta-vs-err}, we summarize the results for linear models, which shows that for different values of $|\Scal|$, the performance generally improves as $\delta$ decreases. 

\vspace*{-2mm}
\subsection{Application to fair regression}
\vspace*{-1mm}
\label{sec:exp-fairness}
Fairness in regression requires that the prediction error limited to any protected group is below a pre-specified label~\cite{agarwal2019fair} and therefore, such an application naturally fits in our setting. To that end, 
we apply our approach to the Law and \ccr datasets~\cite{agarwal2019fair} and enforce fairness violation measured in terms of $\EE[ |(y_i-\hat{y}_i)^2-(y_j-\hat{y}_j)^2 | \,\big|\, i\in V_q, j\in \Vcal\backslash V_q]$ ($Q=8$ in Law, $Q=4$ in \ccr) in the dataset, we  partition the validation set $\Vcal$ into the subsets $V_1,..,V_{Q}$, so that each subset $V_q$ consists of individuals with the race $r_q$, {\em i.e.}, $V_q = \set{(\xb_j,y_j)\,|\, \text{Race of individual }j  = r_q}$.
 \begin{figure}[!!!t]
\centering
\vspace{-3mm}

\hspace*{0.1cm}\hspace*{-0.6cm}\subfloat{\setcounter{subfigure}{1}  \includegraphics[width=0.12\columnwidth]{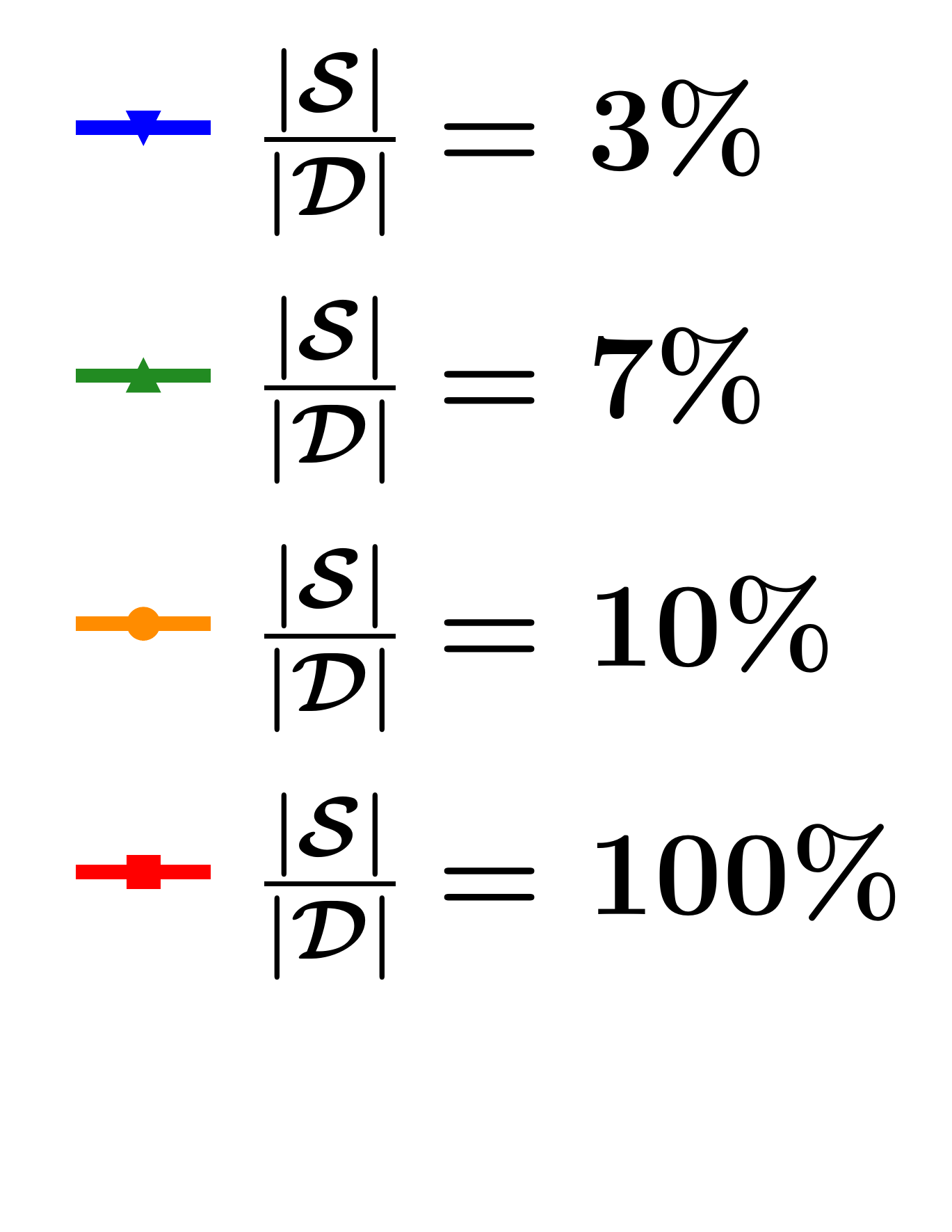}}
\hspace*{0.1cm}\subfloat[Cadata]{\setcounter{subfigure}{1}  \includegraphics[ width=0.25\columnwidth]{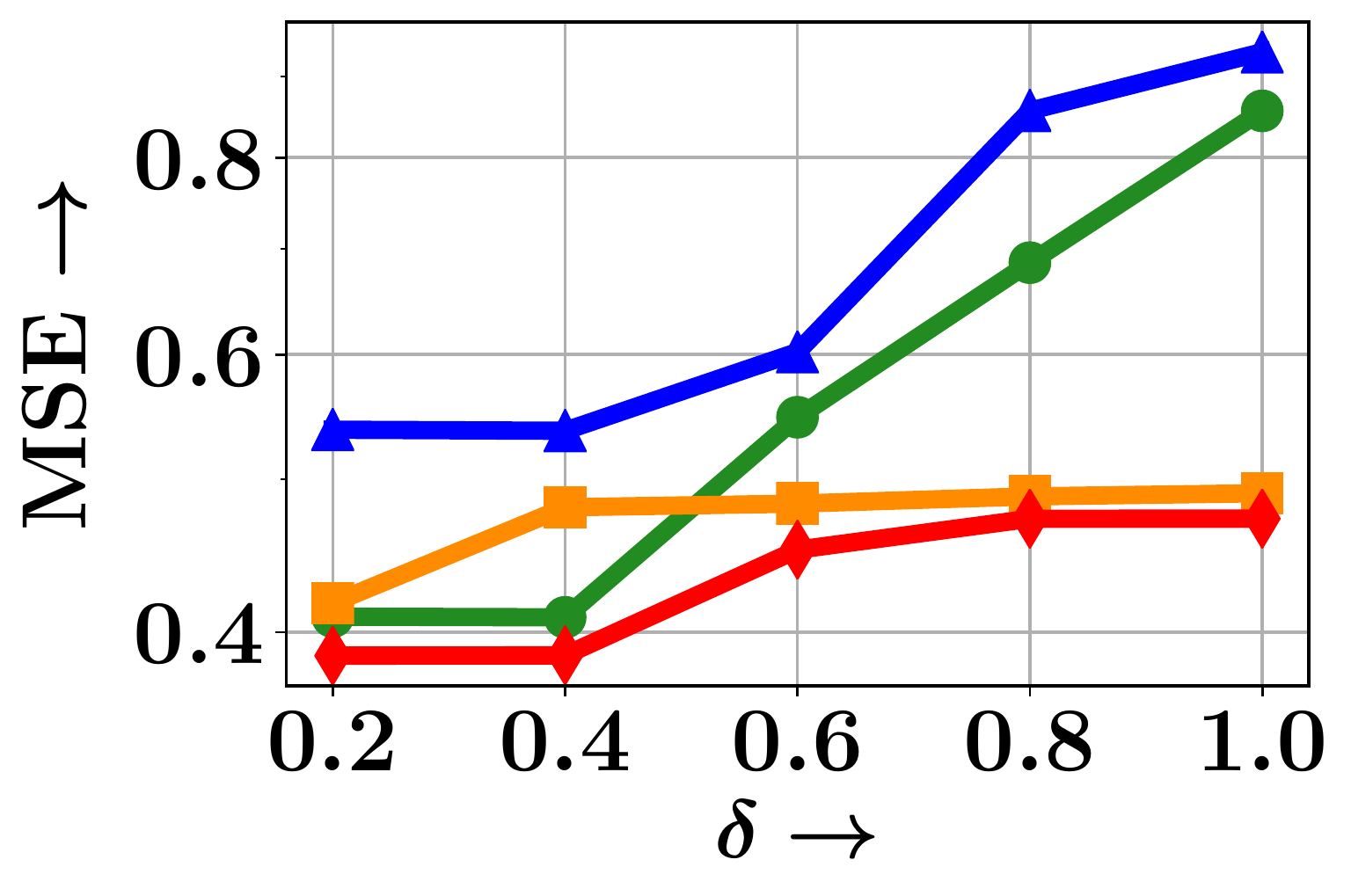}}  
\hspace*{0.1cm}\subfloat[Law]{ \includegraphics[width=0.25\columnwidth]{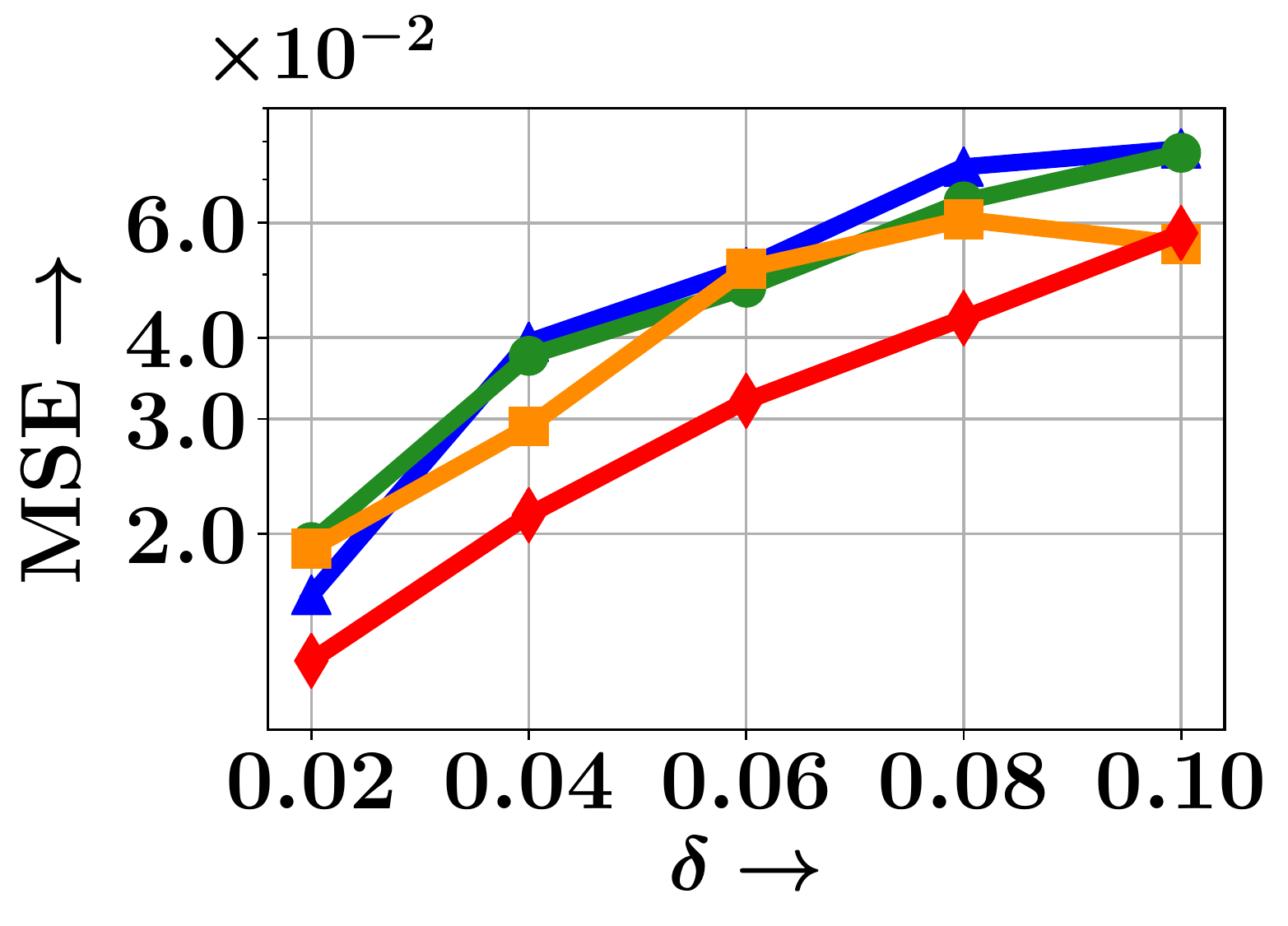}}
\vspace{-1.8mm}
\caption{Variation of mean squared error (MSE) across different values of validation error bound $\delta$ for different sizes of $|\Scal|$. We observe that for different values of $|\Scal|$, the performance generally improves as $\delta$ decreases. }
\label{fig:delta-vs-err}
\vspace{-2mm}
\end{figure}
\begin{figure}[!!!t]
\centering
\hspace{2mm}{ \includegraphics[trim={0 20 0 0}, clip,width=0.5\columnwidth]{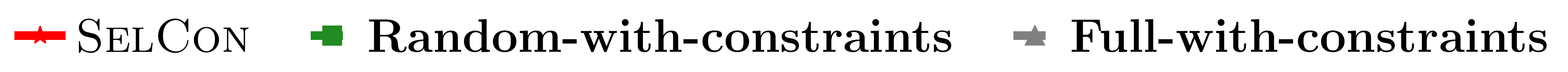}}\\
\vspace*{-3.5mm}
\hspace*{0.1cm}\hspace*{-0.6cm}\subfloat[Law]{  \setcounter{subfigure}{1} \includegraphics[ width=0.30\columnwidth]{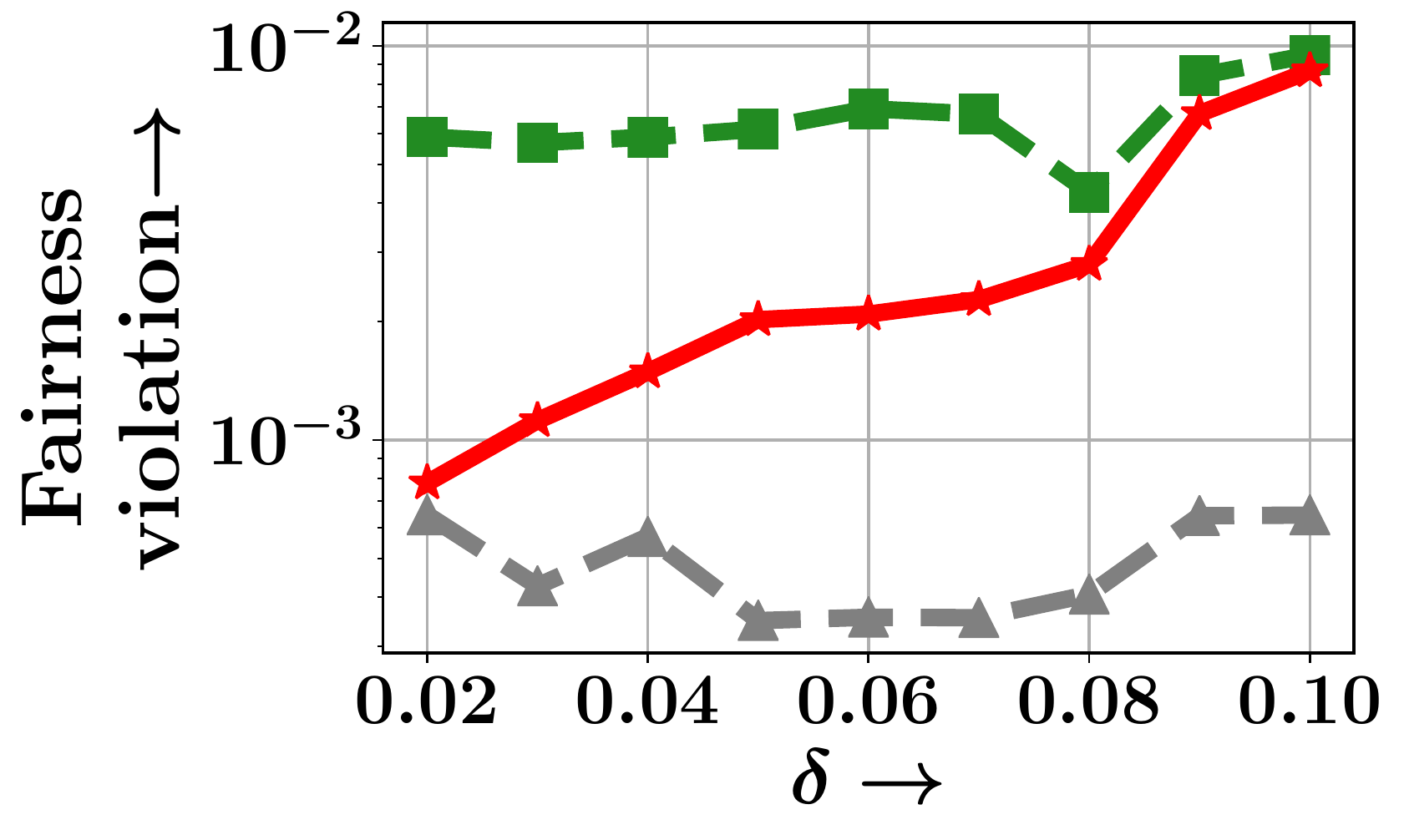}} \hspace{2mm}
\hspace*{0.1cm}\subfloat[\ccr]{ \includegraphics[width=0.30\columnwidth]{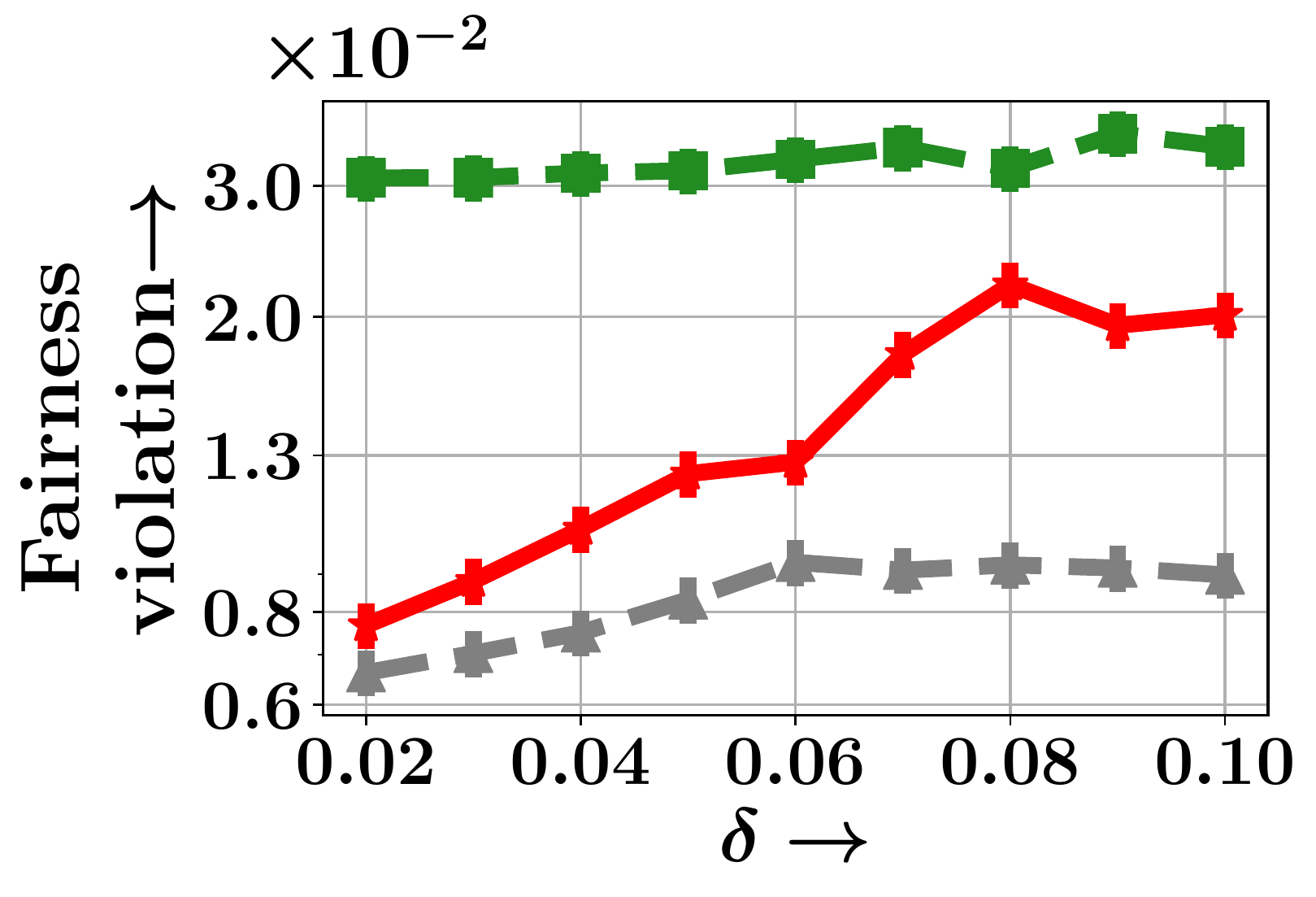}}
\vspace{-2mm}
\caption{Data selection in fair regression - These plots show the
variation of fairness violation measured in terms of $\EE[ |(y_i-\hat{y}_i)^2-(y_j-\hat{y}_j)^2 | \,\big|\, i\in V_q, j\in \Vcal\backslash V_q]$ with $\delta$ for $|\Scal|=0.1|\Vcal|$. 
Here,  $V_q$ consists
of individuals with a particular race $r_q$. We observe that \our guarantees fairness more effectively than \randomc. }
\label{fig:fair}
\vspace{-1mm}
\end{figure}

\xhdr{Results}  In Figure~\ref{fig:fair}, we plot the performance of \our in terms of the mean fairness violation, as measured by $\EE[|(y_i-\hat{y}_i)^2-(y_j-\hat{y}_j)^2| \,\big|\, i\in V_q, j\in \Vcal\backslash V_q]$ for various values of  $\delta$. We compare \our's performance against \fullc and \randomc, the only other methods that can enforce fairness by means of error constraints on the validation set.
Evidently, \our guarantees fairness more effectively than \randomc. Moreover, for low values of $\delta$, the performance of \our is close to \fullc.


\section{Conclusion}
\vspace{-1mm}
We presented  a novel data subset selection formulation that aims to select a subset $\Scal$ by controlling the generalization errors.
%
Specifically, we considered $L_2$ regularized regression over candidate training set $\Scal$, subject to the error bounds on different partitions of the validation set. 
Such error bounds reduce the generalization error that could otherwise increase owing to training on a small sized data.
Thereafter, we reformulated our data selection task as a new optimization problem and showed its equivalence to minimization of a monotone and $\alpha$-submodular function.
Finally, we designed a majorization-minimization based
approximation algorithm \our to solve this problem in the face of imperfect training.
Our experiments show that \our can more effectively trade off between accuracy and efficiency than several baselines. Our work opens several areas for future work; \eg,  it can be extended to data selection for classification as well as  data removal for outlier detection.

\section*{Acknowledgements}
We thank anonymous reviewers for providing constructive feedback. Durga Sivasubramanian is supported by a the Prime Minisiter Research Fellowship. Ganesh and Abir are also grateful to IBM Research, India (specifically the IBM AI Horizon Networks - IIT Bombay initiative) for their support and sponsorship. Abir also acknowledges the DST Inspire Award and IITB Seed Grant. Rishabh acknowledges support from the UT Dallas startup grant.

\bibliographystyle{abbrv}
\bibliography{refs.bib}

\begin{thebibliography}{10}

\bibitem{agarwal2019fair}
A.~Agarwal, M.~Dudik, and Z.~S. Wu.
\newblock Fair regression: Quantitative definitions and reduction-based
  algorithms.
\newblock In {\em International Conference on Machine Learning}, pages
  120--129. PMLR, 2019.

\bibitem{bairi2015summarization}
R.~Bairi, R.~Iyer, G.~Ramakrishnan, and J.~Bilmes.
\newblock Summarization of multi-document topic hierarchies using submodular
  mixtures.
\newblock In {\em ACL}, pages 553--563, 2015.

\bibitem{bhatia2017consistent}
K.~Bhatia, P.~Jain, P.~Kamalaruban, and P.~Kar.
\newblock Consistent robust regression.
\newblock In {\em Advances in Neural Information Processing Systems}, pages
  2110--2119, 2017.

\bibitem{bian2017guarantees}
A.~A. Bian, J.~M. Buhmann, A.~Krause, and S.~Tschiatschek.
\newblock Guarantees for greedy maximization of non-submodular functions with
  applications.
\newblock {\em arXiv preprint arXiv:1703.02100}, 2017.

\bibitem{boutsidis2013near}
C.~Boutsidis, P.~Drineas, and M.~Magdon-Ismail.
\newblock Near-optimal coresets for least-squares regression.
\newblock {\em IEEE transactions on information theory}, 59(10):6880--6892,
  2013.

\bibitem{campbell2018bayesian}
T.~Campbell and T.~Broderick.
\newblock Bayesian coreset construction via greedy iterative geodesic ascent.
\newblock In {\em International Conference on Machine Learning}, pages
  698--706, 2018.

\bibitem{casella2002statistical}
G.~Casella and R.~L. Berger.
\newblock {\em Statistical inference}, volume~2.
\newblock Duxbury Pacific Grove, CA, 2002.

\bibitem{clarkson2010coresets}
K.~L. Clarkson.
\newblock Coresets, sparse greedy approximation, and the frank-wolfe algorithm.
\newblock 2010.

\bibitem{das2011submodular}
A.~Das and D.~Kempe.
\newblock Submodular meets spectral: Greedy algorithms for subset selection,
  sparse approximation and dictionary selection.
\newblock {\em arXiv preprint arXiv:1102.3975}, 2011.

\bibitem{ruha}
A.~De, P.~Koley, N.~Ganguly, and M.~Gomez-Rodriguez.
\newblock Regression under human assistance.
\newblock {\em AAAI}, 2020.

\bibitem{cuha}
A.~De, N.~Okati, A.~Zarezade, and M.~Gomez-Rodriguez.
\newblock Classification under human assistance.
\newblock {\em AAAI}, 2021.

\bibitem{el2020optimal}
M.~El~Halabi and S.~Jegelka.
\newblock Optimal approximation for unconstrained non-submodular minimization.
\newblock In {\em International Conference on Machine Learning}, pages
  3961--3972. PMLR, 2020.

\bibitem{gatmiry2018non}
K.~Gatmiry and M.~Gomez-Rodriguez.
\newblock Non-submodular function maximization subject to a matroid constraint,
  with applications.
\newblock {\em arXiv preprint arXiv:1811.07863}, 2018.

\bibitem{har2004coresets}
S.~Har-Peled and S.~Mazumdar.
\newblock On coresets for k-means and k-median clustering.
\newblock In {\em Proceedings of the thirty-sixth annual ACM symposium on
  Theory of computing}, pages 291--300, 2004.

\bibitem{hashemi2019submodular}
A.~Hashemi, M.~Ghasemi, H.~Vikalo, and U.~Topcu.
\newblock Submodular observation selection and information gathering for
  quadratic models.
\newblock {\em arXiv preprint arXiv:1905.09919}, 2019.

\bibitem{hassani2017gradient}
H.~Hassani, M.~Soltanolkotabi, and A.~Karbasi.
\newblock Gradient methods for submodular maximization.
\newblock In {\em Advances in Neural Information Processing Systems}, pages
  5841--5851, 2017.

\bibitem{Hassidim2017}
A.~Hassidim and Y.~Singer.
\newblock Submodular optimization under noise.
\newblock 2017.

\bibitem{Hassidim2018}
A.~Hassidim and Y.~Singer.
\newblock Optimization for approximate submodularity.
\newblock pages 394--405. Curran Associates Inc., 2018.

\bibitem{Horel2016}
T.~Horel and Y.~Singer.
\newblock Maximization of approximately submodular functions.
\newblock 2016.

\bibitem{iyer2012algorithms}
R.~Iyer and J.~Bilmes.
\newblock Algorithms for approximate minimization of the difference between
  submodular functions, with applications.
\newblock {\em arXiv preprint arXiv:1207.0560}, 2012.

\bibitem{iyer2015polyhedral}
R.~Iyer and J.~Bilmes.
\newblock Polyhedral aspects of submodularity, convexity and concavity.
\newblock {\em arXiv preprint arXiv:1506.07329}, 2015.

\bibitem{iyer2013fast}
R.~Iyer, S.~Jegelka, and J.~Bilmes.
\newblock Fast semidifferential-based submodular function optimization:
  Extended version.
\newblock In {\em ICML}, 2013.

\bibitem{NIPS2013_c1e39d91}
R.~K. Iyer, S.~Jegelka, and J.~A. Bilmes.
\newblock Curvature and optimal algorithms for learning and minimizing
  submodular functions.
\newblock In {\em NeurIPS}, 2013.

\bibitem{kaushal2019learning}
V.~Kaushal, R.~Iyer, S.~Kothawade, R.~Mahadev, K.~Doctor, and G.~Ramakrishnan.
\newblock Learning from less data: A unified data subset selection and active
  learning framework for computer vision.
\newblock In {\em 2019 IEEE Winter Conference on Applications of Computer
  Vision (WACV)}, pages 1289--1299. IEEE, 2019.

\bibitem{killamsetty2021grad}
K.~Killamsetty, D.~Sivasubramanian, B.~Mirzasoleiman, G.~Ramakrishnan, A.~De,
  and R.~Iyer.
\newblock Grad-match: A gradient matching based data subset selection for
  efficient learning.
\newblock {\em arXiv preprint arXiv:2103.00123}, 2021.

\bibitem{killamsetty2020glister}
K.~Killamsetty, D.~Subramanian, G.~Ramakrishnan, and R.~Iyer.
\newblock Glister: A generalization based data selection framework for
  efficient and robust learning.
\newblock {\em In AAAI}, 2021.

\bibitem{kirchhoff2014submodularity}
K.~Kirchhoff and J.~Bilmes.
\newblock Submodularity for data selection in machine translation.
\newblock In {\em Proceedings of the 2014 Conference on Empirical Methods in
  Natural Language Processing (EMNLP)}, pages 131--141, 2014.

\bibitem{koley2021demarcating}
P.~Koley, A.~Saha, S.~Bhattacharya, N.~Ganguly, and A.~De.
\newblock Demarcating endogenous and exogenous opinion dynamics: An
  experimental design approach.
\newblock {\em arXiv preprint arXiv:2102.05954}, 2021.

\bibitem{Kuhnle2018}
A.~Kuhnle, J.~D. Smith, V.~G. Crawford, and M.~T. Thai.
\newblock Fast maximization of non-submodular, monotonic functions on the
  integer lattice.
\newblock {\em arXiv preprint arXiv:1805.06990}, 2018.

\bibitem{kulkarni2018active}
A.~Kulkarni, N.~R. Uppalapati, P.~Singh, and G.~Ramakrishnan.
\newblock An interactive multi-label consensus labeling model for multiple
  labeler judgments.
\newblock In {\em Proceedings of the Thirty-Second {AAAI} Conference on
  Artificial Intelligence, (AAAI), 2018}, pages 1479--1486. {AAAI} Press, 2018.

\bibitem{lehmann2006combinatorial}
B.~Lehmann, D.~Lehmann, and N.~Nisan.
\newblock Combinatorial auctions with decreasing marginal utilities.
\newblock {\em Games and Economic Behavior}, 55(2):270--296, 2006.

\bibitem{liu2015svitchboard}
Y.~Liu, R.~Iyer, K.~Kirchhoff, and J.~Bilmes.
\newblock Svitchboard ii and fisver i: High-quality limited-complexity corpora
  of conversational english speech.
\newblock In {\em Sixteenth Annual Conference of the International Speech
  Communication Association}, 2015.

\bibitem{lucic2017training}
M.~Lucic, M.~Faulkner, A.~Krause, and D.~Feldman.
\newblock Training gaussian mixture models at scale via coresets.
\newblock {\em The Journal of Machine Learning Research}, 18(1):5885--5909,
  2017.

\bibitem{mirzasoleiman2019coresets}
B.~Mirzasoleiman, J.~Bilmes, and J.~Leskovec.
\newblock Coresets for data-efficient training of machine learning models.
\newblock {\em In Proc. ICML}, 2020.

\bibitem{nemhauser1978analysis}
G.~L. Nemhauser, L.~A. Wolsey, and M.~L. Fisher.
\newblock An analysis of approximations for maximizing submodular set
  functions—i.
\newblock {\em Mathematical programming}, 14(1):265--294, 1978.

\bibitem{pace1997sparse}
R.~K. Pace and R.~Barry.
\newblock Sparse spatial autoregressions.
\newblock {\em Statistics \& Probability Letters}, 33(3):291--297, 1997.

\bibitem{Qian2017a}
C.~Qian, J.-C. Shi, Y.~Yu, K.~Tang, and Z.-H. Zhou.
\newblock Subset selection under noise.
\newblock pages 3560--3570, 2017.

\bibitem{ren2018learning}
M.~Ren, W.~Zeng, B.~Yang, and R.~Urtasun.
\newblock Learning to reweight examples for robust deep learning.
\newblock In {\em International Conference on Machine Learning}, pages
  4334--4343, 2018.

\bibitem{rothenhausler2018anchor}
D.~Rothenh{\"a}usler, N.~Meinshausen, P.~B{\"u}hlmann, and J.~Peters.
\newblock Anchor regression: heterogeneous data meets causality.
\newblock {\em arXiv preprint arXiv:1801.06229}, 2018.

\bibitem{sener2018active}
O.~Sener and S.~Savarese.
\newblock Active learning for convolutional neural networks: A core-set
  approach.
\newblock In {\em International Conference on Learning Representations}, 2018.

\bibitem{Singla2016}
A.~Singla, S.~Tschiatschek, and A.~Krause.
\newblock Noisy submodular maximization via adaptive sampling with applications
  to crowdsourced image collection summarization.
\newblock 2016.

\bibitem{wei2014fast}
K.~Wei, R.~Iyer, and J.~Bilmes.
\newblock Fast multi-stage submodular maximization.
\newblock In {\em International conference on machine learning}, pages
  1494--1502. PMLR, 2014.

\bibitem{wei2015submodularity}
K.~Wei, R.~Iyer, and J.~Bilmes.
\newblock Submodularity in data subset selection and active learning.
\newblock In {\em International Conference on Machine Learning}, pages
  1954--1963, 2015.

\bibitem{wei2014unsupervised}
K.~Wei, Y.~Liu, K.~Kirchhoff, and J.~Bilmes.
\newblock Unsupervised submodular subset selection for speech data.
\newblock In {\em 2014 IEEE International Conference on Acoustics, Speech and
  Signal Processing (ICASSP)}, pages 4107--4111. IEEE, 2014.

\bibitem{law}
Wightman.
\newblock Lsac national longitudinal bar passage study.
\newblock 1998.

\bibitem{wilcoxon1992individual}
F.~Wilcoxon.
\newblock Individual comparisons by ranking methods.
\newblock In {\em Breakthroughs in statistics}, pages 196--202. Springer, 1992.

\bibitem{zhang2016submodular}
H.~Zhang and Y.~Vorobeychik.
\newblock Submodular optimization with routing constraints.
\newblock In {\em Proceedings of the AAAI conference on artificial
  intelligence}, volume~30, 2016.

\bibitem{zhang2018generalized}
Z.~Zhang and M.~Sabuncu.
\newblock Generalized cross entropy loss for training deep neural networks with
  noisy labels.
\newblock In {\em Advances in neural information processing systems}, pages
  8778--8788, 2018.

\end{thebibliography}
\newpage
 
\appendix
\onecolumn
\onecolumn

\begin{center}
 \Large{Appendix}
\end{center}

\section{Proofs of the technical results in Section~\ref{sec:formulation}}
\label{app:formulation}

\subsection{Proof of Proposition~\ref{prop:dual}}
\label{app:proof-dual}
\begin{numproposition}{\ref{prop:dual}}
Given a fixed training set $\Scal$, let $\mub=[\mu_q]_{q\in[Q]}$ be the Lagrangian multipliers for the constraints
$ \set{\frac{1}{{|V_q|}}{\sum_{j\in V_q}(y_j-h_{\wb}(\xb_j))^2}\le   \delta + \xi_q}_{q\in[Q]}$ in the optimization problem~\eqref{eq:opt-soft} and $\dual(\wb,\mub,\Scal)$ be defined as follows:
\begin{align}
\hspace{-2.5mm}\dual(\wb,\mub,\Scal)& \hspace{-0.3mm}=    \sum_{i\in\Scal}  [ \lambda\bnm{\wb}^2 +  (y_i-h_{\wb}(\xb_i))^2] \nn \\[-2ex]
 & \ \ \  +  \hspace{-1mm} \sum_{q\in[Q]}  \hspace{-1mm} \mu_q  \hspace{-1mm} \left[ \frac{\sum_{j\in V_q}(y_j-h_{\wb}(\xb_j))^2}{|V_q|}  -  \delta\right] \label{eq:adefxx}  \hspace{-2mm}
\end{align}
Then, for the fixed set $\Scal$, the dual of the optimization problem~\eqref{eq:opt-soft} for estimating $\wb$ and $\set{\xi_q}$ is given by,
\begin{align}
 \maxi_{\bm{0}\le \mub \le C \bm{1}} \ \mini_{\wb} \ \ &  \dual(\wb,\mub,\Scal)\label{eq:aopt-app}
\end{align}
\end{numproposition}
\begin{proof}
The dual problem of our data selection problem~\eqref{eq:opt-soft} is given as:
\begin{align}
 &\maxi_{\mub \ge 0,\bm{\nu}} \mini_{\wb, \{\xi_q\}_{q\in[Q]}}\,   \sum_{i\in\Scal} [ \lambda\bnm{\wb}^2 \hspace{-1mm} +  (y_i-h_{\wb}(\xb_i))^2]  \hspace{-0.4mm}  +  \hspace{-0.4mm} C\hspace{-1mm} \sum_{q\in V_q} \xi_q +  \sum_{q\in[Q]}\mu_q \left[\frac{\sum_{j\in V_q}(y_j-h_{\wb}(\xb_j))^2}{|V_q|}  - \delta - \xi_q \right] - \nu_q \xi_q \nn 
\end{align}
Differentiating with respect to $\bm{\xi}$, we get $\mub+\bm{\nu} = C \bm{1}$, which proves the Proposition (giving us the constraint $\bm{0} \leq \bm{\mu} \leq C\bm{1}$).
\end{proof}
\subsection{Proof of Proposition~\ref{prop:hardness}}
\label{app:proof-hardness}

\begin{numproposition}{\ref{prop:hardness}}
 Both the variants of the data selection problems~\eqref{eq:gopt} and~\eqref{eq:fopt} are NP-Hard.
\end{numproposition}

\begin{proof}
Consider our data selection problem as follows:
\begin{align}
 &\hspace{-3mm} \mini_{\Scal\subset \trn,\wb, \{\xi_q\}_{q\in[Q]}}\,   \sum_{i\in\Scal} [ \lambda\bnm{\wb}^2 \hspace{-1mm} +  (y_i-h_{\wb}(\xb_i))^2]  \hspace{-0.4mm}  +  \hspace{-0.4mm} C\hspace{-1mm} \sum_{q\in V_q} \xi_q,\hspace{-2mm}  \nn\\ 
 & \hspace{1mm}\text{such that, }  \frac{\sum_{j\in V_q}(y_j-h_{\wb}(\xb_j))^2}{|V_q|}  \le  \delta + \xi_q \quad \forall q\in [Q], \nn\\ 
 & \qquad \qquad \quad \xi_q \ge 0\ \quad \forall \, q \in [Q] \text{ and, }\ |\Scal| = k \label{eq:opt-soft-app} 
\end{align}
We make $C=0$ and $h_{\wb}(\xb)=\wb^\top \xb$. Then the problem becomes equivalent to the robust regression problem~\cite{bhatia2017consistent}, \ie,
\begin{align}
 \mini_{\Scal\subset \trn,\wb}\,   \sum_{i\in\Scal} [ \lambda\bnm{\wb}^2 \hspace{-1mm} +  (y_i- \wb^\top \xb)^2], \qquad\text{such that, } |\Scal| = k, 
\end{align}
which is known to be NP-hard.
\end{proof}
\newpage
\section{Techninal results on Section~\ref{sec:method} and their proofs}
\label{app:method-proofs}

\subsection{Proof of Proposition~\ref{prop:mon}}
\label{app:proof-mon}

\begin{numproposition}{~\ref{prop:mon}}
For any model $h_{\wb}$, $f(\Scal)$
is monotone, {\em i.e.}, $f(\Scal\cup\setel{a})-f(\Scal) \ge 0$ for all $\Scal\subset \Dcal$ and $a\in \Dcal\cp\Scal$.
\end{numproposition}

\begin{proof}
We note that
\begin{align}
  f(\Scal \cup   \setel{a}) -f(\Scal)
  &  =  F\left( \wb^*(\mub^*(\Scal\cup\setel{a}),\Scal\cup \setel{a}), \mub^*(\Scal\cup \setel{a}),\Scal\cup\setel{a}\right)
-F\left( \wb^*(\mub^*(\Scal),\Scal), \mub^*(\Scal),\Scal\right) \\
& = \underbrace{F\left( \wb^*(\mub^*(\Scal\cup\setel{a}),\Scal\cup \setel{a}), \mub^*(\Scal\cup \setel{a}),\Scal\cup\setel{a}\right)-
 F\left( \wb^*(\mub^*(\Scal),\Sa ), \mub^*(\Scal),\Sa\right)}_{\ge 0}\nn \\
&\quad + F\left( \wb^*(\mub^*(\Scal),\Sa), \mub^*(\Scal),\Sa\right)
-F\left( \wb^*(\mub^*(\Scal),\Scal), \mub^*(\Scal),\Scal\right)\\
& \overset{(i)}{\ge} F\left( \wb^*(\mub^*(\Scal),\Sa), \mub^*(\Scal),\Sa\right)
-F\left( \wb^*(\mub^*(\Scal),\Scal), \mub^*(\Scal),\Scal\right)\\
&=   F\left( \wb^*(\mub^*(\Scal),\Sa), \mub^*(\Scal),\Sa\right)- F\left( \wb^*(\mub^*(\Scal),\Sa), \mub^*(\Scal),\Scal\right)\nn \\
&\qquad + \ubr{F\left( \wb^*(\mub^*(\Scal),\Sa), \mub^*(\Scal),\Scal\right) -F\left( \wb^*(\mub^*(\Scal),\Scal), \mub^*(\Scal),\Scal\right)}_{\ge 0}\\
& \ov{(ii)}{\ge} 
   F\left( \wb^*(\mub^*(\Scal),\Sa), \mub^*(\Scal),\Sa\right)- F\left( \wb^*(\mub^*(\Scal),\Sa), \mub^*(\Scal),\Scal\right)\nn\\
   &  = \sum_{i\in\Sa}  [ \lambda\bnm{\wb^*(\mub^*(\Scal),\Sa)}^2 +  (y_i-h_{\wb^*(\mub^*(\Scal),\Sa)}(\xb_i))^2]  \nn\\
&  \qquad  +   \sum_{q\in[Q]}   \mu^* _ q (\Scal)   \left[ \frac{\sum_{j\in V_q}(y_j-h_{\wb^*(\mub^*(\Scal),\Sa)}(\xb_j))^2}{|V_q|}  -  \delta\right]  \\
& \qquad - \sum_{i\in\Scal}  [ \lambda\bnm{\wb^*(\mub^*(\Scal),\Sa)}^2 +  (y_i-h_{\wb^*(\mub^*(\Scal),\Sa)}(\xb_i))^2]  \nn\\
&  \qquad \qquad   -   \sum_{q\in[Q]}   \mu^* _ q (\Scal)   \left[ \frac{\sum_{j\in V_q}(y_j-h_{\wb^*(\mub^*(\Scal),\Sa)}(\xb_j))^2}{|V_q|}  -  \delta\right]  \\
& = \lambda\bnm{\wb^*(\mub^*(\Scal),\Sa)}^2 +  (y_a-h_{\wb^*(\mub^*(\Scal),\Sa)}(\xb_a))^2
\end{align}
Here, inequality (i) is due to the fact that:
 $\mub^*(\Scal\cup\setel{a}) = \argmax_{0\le \mub\le C} F\left( \wb^*( \mub ,\Scal\cup \setel{a}), \mub,\Scal\cup\setel{a}\right)$; and, inequality (ii) is due to the fact that:
$ \wb^*(\mub^*(\Scal),\Scal) = \argmin_{\wb} \ F\left( \wb, \mub^*(\Scal),\Scal\right)$.
\end{proof} 

\subsection{Proof of Theorem~\ref{thm:alpha-sub-nonlin}}
\label{app:alpha-sub-key}
\begin{numtheorem}{\ref{thm:alpha-sub-nonlin}}
Assume that $|y|\le y_{\max}$; 
$h_{\wb}(\xb)=0$ if $\wb=\bm{0}$, {\em i.e.}, $h_{\wb}(\xb)$ has no bias term; 
$h_{\wb}$ is $H$-Lipschitz, \ie, $|h_{\wb}(\xb)|\le H\bnm{\wb} $;
the eigenvalues of the Hessian matrix of $(y- h_{\wb}(\xb))^2)$ have a finite upper bound, {\em i.e.},
$\displaystyle\mathrm{Eigenvalue}  (\nabla^2 _{\wb} (y- h_{\wb}(\xb))^2)$ $ \le  2\chi_{\max} ^2 $; and, 
define $ \ell^* = \min_{a\in\Dcal} {\min_{\wb}}\ \ \chi_{\max} ^2 \cdot \bnm{\wb}^2 +  (y_a-h_{\wb}(\xb_a))^2~>~0$. 
Then, for  $\lambda \ge \max\left\{ \chi_{\max} ^2,   {32(1+CQ)^2\ymx ^2 H^2}/{ \ell^*}  \right\}$,  
 $f(\Scal)$ is a $\alpha$-submodular set function, where 
\begin{align}
    \alpha \ge \widehat{\alpha}_f = 1- \dfrac{32(1+CQ)^2\ymx ^2 H^2}{{\lambda \ell^*}},\label{eq:alphaNonxx}
\end{align}
\end{numtheorem} 

\begin{proof}
We assume that: $\Scal \subset \Tcal$. Hence, $|\Tcal|>0$.
Let us define: $\lossa(\wb) = \lambda \bnm{\wb}^2 +  (y_a-h_{\wb}(\xb_a))^2$,  $\wmin= \text{argmin}_{\wb} \lossa(\wb)$.
Finally, we denote $\ell^* = \min_{a\in\Dcal} {\min_{\wb}} \chi_{\max} ^2 \bnm{\wb}^2 +  (y_a-h_{\wb}(\xb_a))^2 $.
Next, we have that:
\begin{align}
  \dfrac{ f(\Scal \cup \setel{a})  -f(\Scal)}{f(\Tcal \cup \setel{a}) -f(\Tcal)}
 &  \ge \dfrac{\lossa\left(\wb^*(\mub^*(\Scal),\Sa)\right)}{ \lossa\left(\wb^*(\mub^*(\Ta),\Tcal)\right) }
 \quad (\text{Due to Lemma~\ref{lem:fbounds}})
\nn \\
&  \ge \dfrac{\lossa\left(\wmin\right)}{ \lossa\left(\wb^*(\mub^*(\Ta),\Tcal)\right) }  \quad (\text{Since $\lossa\left(\wb^*(\mub^*(\Scal),\Sa)\right) \ge \lossa\left(\wmin\right)$})\nn 
\end{align}
\begin{align}
& \ov{(i)}{\ge} \dfrac{\lossa\left(\wb^*(\mub^*(\Tcal),\Ta)\right) - (\lambda + \chi_{\max} ^2) \bnm{\wmin  - \wb^*(\mub^*(\Ta),\Tcal)  }^2}{ \lossa\left(\wb^*(\mub^*(\Ta),\Tcal)\right) } \nn \\
&  \ge 1- \dfrac{ (\lambda + \chi_{\max} ^2) \bnm{\wmin  - \wb^*(\mub^*(\Ta),\Tcal)  }^2}{  \lossa\left(\wb^*(\mub^*(\Ta),\Tcal)\right)}\nn \\
& \ge 1- \dfrac{ (\lambda + \chi_{\max} ^2) }{ \lossa(\wmin)}
\bnm{\wmin  - \wb^*(\mub^*(\Ta),\Tcal)  }^2\quad (\text{Since $\lossa\left(\wb^*(\mub^*(\Scal),\Sa)\right) \ge \lossa\left(\wmin\right)$})\nn \\
& \ov{(ii)}{\ge} 1- \dfrac{32 \lambda }{\ell^*} \dfrac{(1+CQ)^2\ymx ^2 H^2 }{\lambda^2}  \nn \\
& = 1-  \dfrac{32(1+CQ)^2\ymx ^2 H^2}{{\lambda \ell^*}},\label{eq:ffinal}
 \end{align}
Inequality (i) is due to the following:
\begin{align}
 \lossa & (\wb^*(\mub^*(\Ta),\Tcal)) \\
 & = \lossa(\wmin) + \nabla  \lossa (\wmin) ^\top (\wb^*(\mub^*(\Ta),\Tcal)-\wmin)  +(\wb^*(\mub^*(\Ta),\Tcal)-\wmin)^\top \nabla^2 \lossa (\wb')   (\wb-\wmin)^\top   \\
 &   \le \lossa(\wmin) +  \frac{\max_{\set{\text{eig} (\nabla^2 \lossa)}}}{2} \bnm{\wb^*(\mub^*(\Ta),\Tcal)-\wmin}^2 
 \quad \text{($ \nabla  \lossa (\wmin) =0$)}\\
  & \le \lossa(\wmin) +  (\lambda + \chi_{\max} ^2)\bnm{\wb^*(\mub^*(\Ta),\Tcal)-\wmin}^2 ;
\end{align}
and inequality (ii) follows from 
\begin{align}
&(1)\quad\lossa(\wmin) = \lambda \bnm{\wmin}^2 +  (y_a-h_{\wmin}(\xb_a))^2 \ge \chi_{\max} ^2\bnm{\wmin}^2 +  (y_a-h_{\wmin}(\xb_a))^2 \ge  \min_{\wb}\,\chi_{\max} ^2\bnm{\wb}^2 +  (y_a-h_{\wb}(\xb_a))^2 = \ell^*, \nn\\
&(2) \quad \bnm{\wb^*(\mub^*(\Ta),\Tcal)-\wmin} \le 2 w_{\max} = \dfrac{ 4 (1 +CQ) y_{\max} H}{\lambda}\qquad \left(\text{Due to Lemma~\ref{lem:wmax}}\right),\nn\\
&(3) \quad \lambda \ge \chi_{\max} ^2.
\end{align}
\end{proof}
\vspace{-5mm}
\subsection{Proof of Proposition~\ref{cor:alpha-sub-lin}}
\label{app:cor-alpha-sub-lin}
\vspace{-1mm}
\begin{numproposition}{\ref{cor:alpha-sub-lin}}
Given $0< y_{\min} \le |y|~\le~y_{\max}$,~$h_{\wb}(\xb)=\wb^\top \xb$,  
$\bnm{\xb}\le x_{\max}$,  we set the regularizing coefficient as $\lambda \ge\max\big\{\xm^2,   { {16}(1+CQ)^2\ymx ^2 \xm^2}/{    { y_{\min} ^2 } }.\big\}$. Then $f(\Scal)$ is a $\alpha$-submodular set function, where 
\begin{align}
   \alpha  \ge \widehat{\alpha}_f =   1 -  \dfrac{ {16}(1+CQ)^2\ymx ^2 \xm^2}{\lambda   { y_{\min} ^2 } }. \label{eq:alphaLin-app}
\end{align}
\end{numproposition}
\begin{proof}
The proof exactly follows the previous proof, except in the highlighted part. 
We assume that: $\Scal \subset \Tcal$. Hence, $|\Tcal|>0$ and define $\lossa(\wb) = \lambda \bnm{\wb}^2 +  (y_a-h_{\wb}(\xb_a))^2$,  $\wmin= \text{argmin}_{\wb} \lossa(\wb)$; $\ell^* = \min_{a\in\Dcal} {\min_{\wb}} \chi_{\max} ^2 \bnm{\wb}^2 +  (y_a-h_{\wb}(\xb_a))^2 $.
Then, we have that:
\begin{align}
  \dfrac{ f(\Scal \cup \setel{a})  -f(\Scal)}{f(\Tcal \cup \setel{a}) -f(\Tcal)}
 &  {\ge} \dfrac{\lossa\left(\wb^*(\mub^*(\Scal),\Sa)\right)}{ \lossa\left(\wb^*(\mub^*(\Ta),\Tcal)\right) }\nn \\
&   {\ge} \dfrac{\lossa\left(\wmin\right)}{ \lossa\left(\wb^*(\mub^*(\Ta),\Tcal)\right) } \nn \\
&  {\ge} \dfrac{\lossa\left(\wb^*(\mub^*(\Tcal),\Ta)\right) - (\lambda + \chi_{\max} ^2) \bnm{\wmin  - \wb^*(\mub^*(\Ta),\Tcal)  }^2}{ \lossa\left(\wb^*(\mub^*(\Ta),\Tcal)\right) } \nn \\
&  {\ge} 1- \dfrac{ (\lambda + \chi_{\max} ^2) \bnm{\wmin  - \wb^*(\mub^*(\Ta),\Tcal)  }^2}{  \lossa\left(\wb^*(\mub^*(\Ta),\Tcal)\right)}\nn \\
&  {\ge} 1- \dfrac{ (\lambda + \chi_{\max} ^2) }{ \lossa(\wmin)}\bnm{\wmin  - \wb^*(\mub^*(\Ta),\Tcal)  }^2\nn \\
&  {\ge} 1- \dfrac{\hll{8} \lambda }{\ell^*} \dfrac{(1+CQ)^2\ymx ^2 \xm^2 }{\lambda^2} \nn
\end{align}
\begin{align}
& = 1-  \dfrac{\hll{8}(1+CQ)^2\ymx ^2 \xm^2}{{\lambda \ell^*}}\nn\\
& \ge    1 -  \dfrac{\hll{16}(1+CQ)^2\ymx ^2 \xm^2}{\lambda  \hll{$y_{\min} ^2$} }
 \end{align}
where the highlighted part is due to second part of Lemma~\ref{lem:wmax} which gives:
\begin{align}
&\bnm{\wb^*(\mub^*(\Ta),\Tcal)-\wmin} \le 2 w_{\max} = \dfrac{ 2 (1 +CQ) y_{\max} x_{\max}}{\lambda}; 
\end{align}
and,  Claim~\ref{claim:solw} which shows that
 $\ell^* = \frac{\lambda y_{\min} ^2}{\lambda+\xm^2} \ge  y_{\min} ^2/2$.
 \end{proof}
\vspace{-2mm}
\subsection{Proof of Proposition~\ref{prop:curv}}
\label{app:curv}
\vspace{-1mm}
\begin{numproposition}{~\ref{prop:curv}}
Given the assumptions stated in Theorem~\ref{thm:alpha-sub-nonlin}.  the generalized curvature 
$k_{f}(\Scal)$ for any set $\Scal$ satisfies $\kappa_f(\Scal) \le \widehat{\kappa}_f = 1-\dfrac{\ell^*}{(CQ+1)\ymx^2}$.
\end{numproposition}
\begin{proof}
Let us define: $\lossa(\wb) = \lambda \bnm{\wb}^2 +  (y_a-h_{\wb}(\xb_a))^2$,  $\wmin= \text{argmin}_{\wb} \lossa(\wb)$.
Finally, we denote $\ell^* = \min_{a\in\Dcal} {\min_{\wb}} \chi_{\max} ^2 \bnm{\wb}^2 +  (y_a-h_{\wb}(\xb_a))^2 $.
By definition, we have
$1-\kappa_f (\Scal) = \min_{a\in \Dcal} \frac{f(a|\Scal\cp \setel{a})}{ f(a|\emptyset)}$.
We show that,
from Lemma~\ref{lem:fbounds}, we have that:
\begin{align}
  f(a|\Scal\cp \setel{a}) & \ge \lambda\bnm{\wb^*(\mub^*(\Scal\cp \setel{a}),\Scal)}^2 +  (y_a-h_{\wb^*(\mub^*(\Scal\cp\setel{a}),\Scal)}(\xb_a))^2  \ge  \lossa(\wmin) \ge \ell^* \label{eq:kappa-1}
\end{align}
Next, we note that:
\begin{align}
 f(a|\emptyset) & = f(\setel{a}) -f(\es)\\
 & = \lambda \bnm{\wb^* (\mub^*(\setel{a}),\setel{a})}^2 + (y_a-h_{ \wb^* (\mub^*(\setel{a}),\setel{a}) }(\xb_a))^2\nn\\
&\qquad+ \sum_{q\in[Q]}  \mu^* _q (\setel{a})\sum_{j\in V_q}
\left[  \dfrac{(y_j-h_{ \wb^* (\mub^*(\setel{a}),\setel{a}) }(\xb_j))^2 }{|V_q|}-\delta\right]\nn
  - \sum_{q\in[Q]} \mu^* _q (\es)\sum_{j\in V_q}\left[  \dfrac{(y_j-h_{ \wb^* (\mub^*(\es),\es) }(\xb_j))^2 }{|V_q|}-\delta\right]\nn\\
   & \ov{(i)}{\le} \lambda \bnm{\wb^* (\mub^*(\setel{a}),\setel{a})}^2 + (y_a-h_{ \wb^* (\mub^*(\setel{a}),\setel{a}) }(\xb_a))^2 \nn\\
&\qquad+ \sum_{q\in[Q]}  \mu^* _q (\setel{a})\sum_{j\in V_q}\left[  \dfrac{(y_j-h_{ \wb^* (\mub^*(\setel{a}),\setel{a}) }(\xb_j))^2 }{|V_q|}-\delta\right]
- \sum_{q\in[Q]} \mu^* _q (\setel{a})\sum_{j\in V_q}\left[  \dfrac{(y_j-h_{ \wb^* (\mub^*(\es),\es) }(\xb_j))^2 }{|V_q|}-\delta\right] \nn\\
  & {=} \lambda \bnm{\wb^* (\mub^*(\setel{a}),\setel{a})}^2 + (y_a-h_{ \wb^* (\mub^*(\setel{a}),\setel{a}) }(\xb_a))^2\nn\\
&\qquad+ \sum_{q\in[Q]}  \mu^* _q (\setel{a})\sum_{j\in V_q}\left[  \dfrac{(y_j-h_{ \wb^* (\mub^*(\setel{a}),\setel{a}) }(\xb_j))^2 }{|V_q|}\right] - \sum_{q\in[Q]} \mu^* _q (\setel{a})\sum_{j\in V_q}\left[  \dfrac{(y_j-h_{ \wb^* (\mub^*(\es),\es) }(\xb_j))^2 }{|V_q|}\right] \nn\\
    &  {\le} \lambda \bnm{\wb^* (\mub^*(\setel{a}),\setel{a})}^2 + (y_a-h_{ \wb^* (\mub^*(\setel{a}),\setel{a}) }(\xb_a))^2\nn\\
&\qquad+ \sum_{q\in[Q]}  \mu^* _q (\setel{a})\sum_{j\in V_q}\left[  \dfrac{(y_j-h_{ \wb^* (\mub^*(\setel{a}),\setel{a}) }(\xb_j))^2 }{|V_q|}\right]\label{eq:dummy} \\
& \ov{(ii)}{\le} (CQ+1)\,\ymx^2 \label{eq:kappa-2}
\end{align}
Here, (i) is because $\mub^*(\es) =\argmax_{\mub}\sum_{q\in[Q]}  \mu  _q  \sum_{j\in V_q}\left[  \dfrac{(y_j-h_{ \wb^* (\mub ,\es) }(\xb_j))^2 }{|V_q|}-\delta\right]$, (ii) is obtained by putting $\wb=\bm{0}$ in Eq.~\eqref{eq:dummy} which is now at the minimum, \ie,
\begin{align}
\wb^* (\mub^*(\setel{a}),\setel{a}) =\argmin_{\wb} \lambda \bnm{\wb}^2 + (y_a-h_{\wb }(\xb_a))^2
 + \sum_{q\in[Q]}  \mu^* _q (\setel{a})\sum_{j\in V_q}\left[  \dfrac{(y_j-h_{\wb}(\xb_j))^2 }{|V_q|}\right]
\end{align}

Hence, Eqs.~\eqref{eq:kappa-1} and~\eqref{eq:kappa-2} show that, $\kappa_f(\Scal) \le 1-\dfrac{\ell^*}{(CQ+1)\,\ymx^2}$.
\end{proof}
%
\newpage
\subsection{Auxiliary Lemmas}

\begin{lemma} 
\label{lem:fbounds}
If $f(\cdot)$ defined in Eq.~\eqref{eq:def}, we have that
\begin{align}
f(\Scal \cup \setel{a}) -f(\Scal) \ge \lambda\bnm{\wb^*(\mub^*(\Scal),\Sa)}^2 +  (y_a-h_{\wb^*(\mub^*(\Scal),\Sa)}(\xb_a))^2.   
\end{align}
and, 
\begin{align}
f(\Scal \cup \setel{a}) -f(\Scal) \le \lambda\bnm{\wb^*(\mub^*(\Sa),\Scal)}^2 +  (y_a-h_{\wb^*(\mub^*(\Sa),\Scal)}(\xb_a))^2.   
\end{align}

\end{lemma}

\begin{proof}
The proof of the lower bound of the marginal gain
\begin{align}
 f(\Scal \cup \setel{a}) -f(\Scal) \ge \lambda\bnm{\wb^*(\mub^*(\Scal),\Sa)}^2 +  (y_a-h_{\wb^*(\mub^*(\Scal),\Sa)}(\xb_a))^2. 
\end{align}
follows from the proof of Proposition~\ref{prop:mon}. 

Next we prove that
\begin{align}
f(\Scal \cup \setel{a}) -f(\Scal) \le \lambda\bnm{\wb^*(\mub^*(\Sa),\Scal)}^2 +  (y_a-h_{\wb^*(\mub^*(\Sa),\Scal)}(\xb_a))^2.   
\end{align}
To show this, we prove that:
\begin{align}
  f(\Scal \cup & \setel{a}) -f(\Scal)\nn\\
  &  =  F\left( \wb^*(\mub^*(\Scal\cup\setel{a}),\Scal\cup \setel{a}), \mub^*(\Scal\cup \setel{a}),\Scal\cup\setel{a}\right)
-F\left( \wb^*(\mub^*(\Scal),\Scal), \mub^*(\Scal),\Scal\right) \\
& =   F\left( \wb^*(\mub^*(\Scal\cup\setel{a}),\Scal\cup \setel{a}), \mub^*(\Scal\cup \setel{a}),\Scal\cup\setel{a}\right)-
 F\left( \wb^*(\mub^*(\Sa),\Scal ), \mub^*(\Sa),\Scal\right)  \nn \\
&\quad + \ubr{ F\left( \wb^*(\mub^*(\Sa),\Scal ), \mub^*(\Sa),\Scal\right)
-F\left( \wb^*(\mub^*(\Scal),\Scal), \mub^*(\Scal),\Scal\right)}_{\le 0}\\
& \overset{(i)}{\le}  F\left( \wb^*(\mub^*(\Scal\cup\setel{a}),\Scal\cup \setel{a}), \mub^*(\Scal\cup \setel{a}),\Scal\cup\setel{a}\right)-
 F\left( \wb^*(\mub^*(\Sa),\Scal ), \mub^*(\Sa),\Scal\right)\\
& =   \ubr{F\left( \wb^*(\mub^*(\Scal\cup\setel{a}),\Scal\cup \setel{a}), \mub^*(\Scal\cup \setel{a}),\Scal\cup\setel{a}\right)
 - F\left( \wb^*(\mub^*(\Sa),\Scal ), \mub^*(\Sa),\Sa \right)}_{\le 0} \nn \\
 &\qquad  +F\left( \wb^*(\mub^*(\Sa),\Scal ), \mub^*(\Sa),\Sa \right) -
 F\left( \wb^*(\mub^*(\Sa),\Scal ), \mub^*(\Sa),\Scal\right)\nn\\
& \ov{(ii)}{\le} 
 F\left( \wb^*(\mub^*(\Sa),\Scal ), \mub^*(\Sa),\Sa \right) -
 F\left( \wb^*(\mub^*(\Sa),\Scal ), \mub^*(\Sa),\Scal\right)\nn\\
   &  = \sum_{i\in\Sa}  [ \lambda\bnm{\wb^*(\mub^*(\Sa),\Scal)}^2 +  (y_i-h_{\wb^*(\mub^*(\Sa),\Scal)}(\xb_i))^2]  \nn\\
&  \qquad  +   \sum_{q\in[Q]}   \mu^* _ q (\Sa)   \left[ \frac{\sum_{j\in V_q}(y_j-h_{\wb^*(\mub^*(\Sa),\Scal)}(\xb_j))^2}{|V_q|}  -  \delta\right]  \\
& \qquad - \sum_{i\in\Scal}  [ \lambda\bnm{\wb^*(\mub^*(\Sa),\Scal)}^2 +  (y_i-h_{\wb^*(\mub^*(\Sa),\Scal)}(\xb_i))^2]  \nn\\
&  \qquad \qquad  -  \sum_{q\in[Q]}   \mu^* _ q (\Sa)   \left[ \frac{\sum_{j\in V_q}(y_j-h_{\wb^*(\mub^*(\Sa),\Scal)}(\xb_j))^2}{|V_q|}  -  \delta\right]  \\
& = \lambda\bnm{\wb^*(\mub^*(\Sa),\Scal)}^2 +  (y_a-h_{\wb^*(\mub^*(\Sa),\Scal)}(\xb_a))^2.
\end{align}
Here (i) is due to the fact that,
\begin{align}
\mub^*(\Scal)  = \argmax_{\mub} F(\wb^*(\mub,\Scal), \mub, \Scal)  
\end{align}
and (ii) is due to the fact that:
\begin{align}
\wb^*(\mub^*(\Sa), \Sa)  = \argmin_{\wb} F(\wb, \mub^*(\Sa), \Sa)  
\end{align}

\end{proof}
\newpage
\begin{lemma}
\label{lem:wmax}
Given that $\Scal \neq \es$; $h_{\wb}(\xb)=0$ for $\wb=\bm{0}$; and, $h_{\wb}(\xb)$ is $H$-Lipschitz, \ie,
$|h_{\wb}(\xb)|\le H \bnm{w}$. Then, we have $\bnm{\wg}\le \dfrac{ 2 (1 +CQ) y_{\max} H}{\lambda} $.
Moreover, if $h_{\wb}(\xb) =\wb^\top \xb$, we have that $\bnm{\wg}\le \dfrac{(1 +CQ) y_{\max} x_{\max}}{\lambda}$.
Note that, for the linear model, we are able to exploit the structure of the model much better and therefore the bound is tighter.
\end{lemma}
\vspace{-2mm}
\begin{proof}
First we define $\dehz(\xb)=\deh|_{\wb=\bm{0}}$.
\begin{align}
&F(\wg ,\mub,\Scal)\nn\\
&= \lambda \bnm{\wg}^2  |\Scal|+
\sum_{i\in\Scal} (y_i-h_{\wg} (\xb_i))^2 
+ \sum_{q\in[Q]} \frac{\mu_q}{|V_q|} \sum_{j\in V_q} \left[{(y_j-h_{\wg} (\xb_j))^2-\delta}\right] \nn\\ 
&= \lambda \bnm{\wg}^2  |\Scal| + \sum_{i\in\Scal} y_i ^2 +\sum_{q\in [Q]}\frac{\mu_q}{V_q} \sum_{j\in V_q}y_j ^2-\sumv \delta\nn\\
&\quad - 2\sum_{i\in\Scal} y_i h_{\wg}(\xb_i) -2  \sumv y_j h_{\wg}(\xb_j) + \underbrace{\sum_{i\in\Scal}  h^2 _{\wg}(\xb_i)
+ \sumv h^2 _{\wg}(\xb_j)}_{\ge 0}\nn\\[-4ex]
&\overset{(i)}{\ge} \lambda \bnm{\wg}^2  |\Scal| +\overbrace{\sum_{i\in\Scal} y_i ^2 +\sum_{q\in [Q]}\frac{\mu_q}{V_q} \sum_{j\in V_q}y_j ^2-\sumv \delta}^{F(\bm{0},\mu,\Scal)}\nn\\
&\quad - 2\sum_{i\in\Scal} y_i h_{\wg}(\xb_i) -2  \sumv y_j h_{\wg}(\xb_j).  \label{eq:pp1}
\end{align}
Here (i) is due to $\displaystyle \sum_{i\in\Scal}  h^2 _{\wg}(\xb_i)
+ \sumv h^2 _{\wg}(\xb_j) \ge 0$. 
Now since $\displaystyle F(\bm{0},\mu,\Scal)=\sum_{i\in\Scal} y_i ^2 +\sum_{q\in [Q]}\frac{\mu_q}{V_q} \sum_{j\in V_q}y_j ^2-\sumv \delta$, Eq.~\eqref{eq:pp1} gives us:
\begin{align}
 F(\wg,\mub,\Scal) - F(\bm{0},\mu,\Scal)& \ge \lambda \bnm{\wg}^2  |\Scal|\nn\\[1ex]   
 &\ - 2\sum_{i\in\Scal} y_i h_{\wg}(\xb_i) -2  \sumv y_j h_{\wg}(\xb_j) \label{eq:pp2}
\end{align}
Now since $\wg =\argmin_{\wb} F(\wb,\mub,\Scal)$, we have that
$F(\wg,\mub,\Scal) \le F(\bm{0},\mu,\Scal)$. Then, Eq.~\eqref{eq:pp2} implies that
\begin{align}
 &\lambda \bnm{\wg}^2  |\Scal|  
 - 2\sum_{i\in\Scal} y_i h_{\wg}(\xb_i) -2  \sumv y_j h_{\wg}(\xb_j) \le 0\nn\\
& \overset{(i)}{\implies}  \lambda \bnm{\wg}^2  |\Scal|  \le  2 |\Scal|\,  y_{\max} H \bnm{\wg} + 2  \sumv y_{\max} H   \bnm{\wg}\nn\\
&\hspace{3.6cm}  \le  2 (|\Scal| +CQ) y_{\max} H \bnm{\wg} \nn\\
&\implies  \bnm{\wg} \le
\frac {2 (|\Scal| +CQ) y_{\max} H}{\lambda |\Scal|} \le \frac{ 2 (1 +CQ) y_{\max} H}{\lambda} 
 \end{align}
 Here $(i)$ is due to $H$-Lipschitzness of $\hb_{\wg}(\xb)$.
 For linear model, we have $H=x_{\max}$. However, we use the structure of the model to obtain a better bound. More specifically, for linear model, we have:
 \begin{align}
 \wg &= \left(\lambda |\Scal| \II + \sum_{i\in\Scal} \xb_i \xb _i ^\top +\sumv \xb_j \xb _j ^\top \right)^{-1} 
 \left(\sum_{i\in\Scal} y_i \xb_i +\sumv y_j \xb_j\right) \nn\\
 {\implies} \bnm{\wg}& \le \frac{(|\Scal|+CQ)x_{\max} y_{\max}} {\text{Eig}_{\min}\left(\lambda |\Scal| \II + \sum_{i\in\Scal} \xb_i \xb _i ^\top +\sumv \xb_j \xb _j ^\top \right)}\nn\\
& \le  \frac{(|\Scal|+CQ)x_{\max} y_{\max}}{\lambda|\Scal|}\nn\\
 &\le \frac{(1+CQ)x_{\max} y_{\max}}{\lambda}.
 \end{align}
\end{proof}
 \begin{claim}
 \label{claim:solw}
$\min_{\wb} [\lambda \bnm{\wb}^2 +  (y_a-\wb^\top\xb_a)^2] = \frac{\lambda y_a ^2}{\lambda + \bnm{\xb_a}^2} $ 
 \end{claim}
 \begin{proof}
  We note that:
  \begin{align}
   \wmin = y_a(\lambda+\xb_a \xb_a ^\top)^{-1}\xb_a
  \end{align}
Hence, we have that:
\begin{align}
 \lambda \bnm{\wmin}^2   +  (y_a-\wmin ^\top\xb_a)^2  
& = y_a ^2 -2 y_a\wmin ^\top \xb_a + \wmin^\top (\lambda \II + \xb_a \xb_a ^\top) \wmin\\
 & = y_a ^2 -y_a\wmin ^\top \xb_a\\
 & = y_a ^2 -y_a ^2 \xb_a ^\top (\lambda+\xb_a \xb_a ^\top)^{-1} \xb_a\\
 & = y_a ^2 - y_a ^2 \xb_a ^\top \left[ \dfrac{1}{\lambda} - \dfrac{\xb_a \xb_a ^\top/\lambda^2}{1+\xb_a ^\top \xb_a/\lambda}  \right] \xb_a \ \   \text{(Due to Sherman Morrison formula)}\\
 & = \frac{\lambda y_a ^2}{\lambda + \bnm{\xb_a}^2} \nn\\
 &\ge \frac{\lambda y_{\min} ^2}{\lambda + x_{\max}^2}
\end{align}
 \end{proof}

\newpage
\section{Proofs of the technical results in Section~\ref{sec:algo}}
\label{app:algo-proofs}
\subsection{Proof of Lemma~\ref{mod-upper-weaksubmod}}
\label{app:modupper}
\begin{numlemma}{\ref{mod-upper-weaksubmod}}
Given a fixed set $\Schat$ and an $\alpha$-submodular function $f(\Scal)$, let the modular function $\mfs$ be defined as follows:
\begin{align}
\hspace{-4mm}\mfs= &    f(\Schat) -   \sum_{i \in \Schat}  \alpha f(i | \Schat \backslash \set{i})\nn\\[-1ex]
&+ \sum_{i \in \Schat \cap \Scal} \alpha f(i | \Schat \backslash \set{i}) +   \sum_{i \in \Scal \backslash \Schat} \frac{f(i | \emptyset)}{\alpha} .\hspace{-1mm}\label{eq:mdef-app}
\end{align}
Then, $f(\Scal) \le \mfs$  for all $\Scal\subseteq\Dcal$.
\end{numlemma}
\begin{proof}
 Recall that $f$ is $\alpha$-submodular with coefficient $\alphaf$ if $f(a | \Scal) \geq \alphaf f(a | \Tcal), a \notin \Tcal, \Scal \subseteq \Tcal$. Given this, the following inequalities follow  directly from:
 \begin{align}\label{temp-eq12}
     \alphaf [f(\Scal) - f(\Schat \cap \Scal)] \leq \sum_{i \in \Scal \backslash \Schat} f(i | \emptyset)
 \end{align}
 and similarly, 
  \begin{align}\label{temp-eq123}
     [f(\Schat) - f(\Schat \cap \Scal)] \geq \alphaf \sum_{i \in \Schat \backslash \Scal} f(i | \Schat \backslash i)
 \end{align}
 The inequalities above hold by considering a chain of sets from $\Schat \cap \Scal$ to either $\Schat$ or $\Scal$ and applying the weak-submodularity definition by considering sets $\Scal$ and $\Tcal$ appropriately. We then multiply $-1$ to inequality~\eqref{temp-eq123}, multiply $1/\alphaf$ to equation~\eqref{temp-eq12} and add both of them together. We then achieve:
 \begin{align}
     f(\Scal) \leq f(\Schat) - \alphaf \sum_{i \in \Schat \backslash \Scal} f(i | \Schat \backslash i) + \frac{1}{\alphaf} \sum_{i \in \Scal \backslash \Schat} f(i | \emptyset)
 \end{align}
 Rearranging this, we get the expression for the Lemma. 
\end{proof}

\subsection{Proof of Theorem~\ref{thm:a0}}
\label{app:a0}

\begin{numtheorem}{\ref{thm:a0}}
If the training algorithm in Algorithm~\ref{alg:selcon} (lines \ftrain, \strain, \ttrain)
provides perfect  estimates of the model parameters, it obtains a set $\widehat{\Scal}$ which satisfies:
\begin{align}
    f(\widehat{\Scal}) \leq \frac{k}{\af (1 + (k-1)(1 - \kf)\af)} f(\Scal^*)
\end{align}
where $\alphaf$ and $\kappaf$ are as stated in Theorem~\ref{thm:alpha-sub-nonlin} and   Proposition~\ref{prop:curv} respectively.
\end{numtheorem}

\begin{proof}
 From the definition of $\alpha$-submodularity, note that $\alphaf f(\Scal) \leq \sum_{i \in \Scal} f(i)$. Next, we can obtain the following inequality for any $k \in \Scal$ using weak submodularity:
\begin{align}
    f(\Scal) - f(k) \geq \alphaf \sum_{j \in \Scal \backslash k}( f(j | \Scal \backslash j)
\end{align}
We can add this up for all $k \in \Scal$ and obtain:
\begin{align}
    |\Scal| f(\Scal) - \sum_{k \in \Scal} f(k) &\geq \alphaf \sum_{k \in \Scal} \sum_{j \in \Scal \backslash k}( f(j | \Scal \backslash j) \nn \\
    &\geq \alphaf (|\Scal| - 1) \sum_{k \in \Scal} f(k | \Scal \backslash k)
\end{align}
Finally, from the definition of curvature, note that $f(k | \Scal \backslash k) \leq (1 - \kappaf) f(k)$.  Combining all this together, we obtain:
\begin{align}
    |\Scal| f(\Scal) \geq (1 + \alphaf (1 - \kappaf)(|S| - 1)) \sum_{j \in \Scal} f(j)
\end{align}
which implies:
\begin{align}
    \sum_{j \in \Scal} f(j) \leq \frac{|\Scal|}{1 + \alphaf (1 - \kappaf)(|S| - 1)} f(\Scal)
\end{align}
Combining this with the fact that $\alphaf f(\Scal) \leq \sum_{i \in \Scal} f(i)$, we obtain that:
\begin{align}
f(\Scal) \leq \frac{1}{\alphaf} \sum_{i \in \Scal} f(i) \leq \frac{|\Scal|}{\alphaf(1 + \alphaf (1 - \kappaf)(|S| - 1))} f(\Scal)
\end{align}
The approximation guarantee then follows from some simple observations. In particular, given an approximation
\begin{align}
    m^f(\Scal) = \frac{1}{\alphaf} \sum_{i \in \Scal} f(i)
\end{align}
which satisfies $f(\Scal) \leq m^f(\Scal) \leq \beta_f f(\Scal)$, we claim that optimizing $m^f$ essentially gives a $\beta_f$ approximation factor. To prove this, let $\Scal^*$ be the optimal subset, and $\hat{\Scal}$ be the subset obtained after optimizing $m^f$. The following chain of inequalities holds:
\begin{align}
    f(\hat{\Scal}) \leq m^f(\hat{\Scal}) \leq m^f(\Scal^*) \leq \beta_f f(\Scal^*)
\end{align}
This shows that $\hat{\Scal}$ is a $\beta_f$ approximation of $\Scal^*$. Finally, note that this is just the first iteration of \our, and with subsequent iterations, \our is guaranteed to reduce the objective value (see Appendix~\ref{app:conv-prop}).
\end{proof}

\subsection{Proof of Theorem~\ref{thm:a1}}
\label{app:a1}
\begin{numtheorem}{\ref{thm:a1}}
If the training algorithm (lines \ftrain, \strain, \ttrain) in Algorithm~\ref{alg:selcon} 
provides imperfect estimates, so that $\bnm{F(\west,\muest,\Scal)- F(\wb^*(\mub^*(\Scal),\Scal),\mub^*(\Scal),\Scal)} \le \epsilon$\ for any $\Scal$,
then Algorithm~\ref{alg:selcon} obtains a set $\widehat{\Scal}$ that satisfies:
\begin{align}
    f(\widehat{\Scal}) \leq \left(\dfrac{k}{\af(1 + (k-1)(1 - \kf)\af)} +  \frac{2k\epsilon}{\ell}\right) f(S^*),\nn
\end{align}
where $\ell= \min_{a\in\Dcal}\min_{\wb} \lambda||\wb||^2 + (y_i-h_{\wb}(\xb_i))^2 $, $\alphaf$ and  $\kappaf$ are obtained in Theorem~\ref{thm:alpha-sub-nonlin} and Proposition~\ref{prop:curv}, respectively.
\end{numtheorem}
\begin{proof}
Define:
\begin{align}
    \beta_f = \dfrac{k}{(1 + (k-1)(1 - \kf)\af)}
\end{align}
and also define, $\hat{f}(\Scal) = F(\west,\muest,\Scal)$ and $f(\Scal) = F(\wb^*(\mub^*(\Scal),\Scal),\mub^*(\Scal),\Scal)$. Note that instead of having access to $f$, the algorithm has access to $\hat{f}$ which satisfies:
\begin{align}
    |f(\Scal) - \hat{f}(\Scal)| \leq \epsilon, \forall \Scal
\end{align}
Let us assume that $\hat{f}$ is always smaller compared to $f$, i.e. in other words,
\begin{align}
    f(\Scal) \leq \hat{f}(\Scal) \leq f(\Scal) + \epsilon
\end{align}
Combining this with the fact that:
\begin{align}
    f(\Scal) \leq \frac{1}{\alphaf} \sum_{j \in \Scal} f(j) \leq \frac{\beta_f}{\alphaf} f(\Scal)
\end{align}
we obtain the following chain of inequalities:
\begin{align}
    f(\Scal) \leq \frac{1}{\alphaf} \sum_{j \in \Scal} f(j) \leq \frac{1}{\alphaf} \sum_{j \in \Scal} [\hat{f}(j)] \leq \frac{1}{\alphaf} \sum_{j \in \Scal} [f(j) + \epsilon] \leq \frac{\beta_f}{\alphaf} f(\Scal) + \frac{k \epsilon}{\alphaf}
\end{align}
where $|\Scal| = k$. Finally, we get the approximation factor by dividing by a lower bound of $l = \min_{\Scal: |\Scal| = k} f(\Scal)$ which can be obtained via a very similar proof technique to the weak submodularity and curvature results. Hence we get the final approximation factor as $\frac{\beta_f}{\alphaf} + \frac{k \epsilon}{l\alphaf}$.

We end by pointing out that we can get a similar result even if we do not assume that $\hat{f}$ is always smaller compared to $f$ and in fact, assume the more general condition:
\begin{align}
    f(\Scal) - \epsilon \leq \hat{f}(\Scal) \leq f(\Scal) + \epsilon
\end{align}
The only difference is we have an additional factor of 2 in the additive bound. In particular, we get the following chain of inequalities:
\begin{align}
    f(\Scal) \leq \frac{1}{\alphaf} \sum_{j \in \Scal} f(j) \leq \frac{1}{\alphaf} \sum_{j \in \Scal} [\hat{f}(j) + \epsilon] \leq \frac{1}{\alphaf} \sum_{j \in \Scal} [f(j) + 2\epsilon] \leq \frac{\beta_f}{\alphaf} f(\Scal) + \frac{2k \epsilon}{\alphaf}
\end{align}
The chain of inequalities holds because $f(j) \leq \hat{f}(j) + \epsilon$ and $\hat{f}(j) \leq f(j) + \epsilon$.
\end{proof}

%

\subsection{Convergence property}
\label{app:conv-prop}
We begin this section by showing that \our is guaranteed to reduce the objective value at every iteration as long as we obtain perfect solutions from the training algorithm (lines 3, 6, 8 in Algorithm~\ref{alg:selcon}).
\begin{lemma}
\our (Algorithm~\ref{alg:selcon}) is guaranteed to reduce the objective value of $f$ at every iteration as long as we obtain perfect solutions from the training sub-routine.
\end{lemma}
\begin{proof}
\our essentially uses modular upper bounds $m^f$ of $f$ at every iteration. Denote $\Scal_l$ as the set obtained in the $l$th iteration and let $\Scal_{l+1}$ be the one from the $l+1$th iteration. Then the following chain of inequalities hold:
\begin{align}
    f(\Scal_{l+1}) \leq m^f(\Scal_{l+1}) \leq m^f(\Scal_l) = f(\Scal_l)
\end{align}
The first inequality holds because $m^f$ is a modular upper bound, the second inequality holds because $\Scal_{l+1}$ is the solution of minimizing $m^f$ (and hence $m^f(\Scal_{l=1})$ is lower in value compared to $m^f(\Scal_l)$). The last equality holds because $m^f$ is a modular upper bound which is tight at $\Scal_l$ and hence $m^f(\Scal_l) = f(\Scal_l)$. This shows that $f(\Scal_{l+1}) \leq f(\Scal_{l})$.
\end{proof}
We end this section by pointing out that this chain of inequalities does not hold if we get inexact or approximation solutions to the training sub-routine.
%
In practice, we observe that the objective value of $f$ still reduces even though we obtain only inexact solutions since the inexact solutions are often close to the true solutions of the training step. 

\newpage
\section{Additional details about experimental setup}
\label{app:expsetup}
\subsection{Dataset details}

\begin{itemize}
    \item \textbf{Cadata:} California housing dataset is obtained from the LIBSVM package~\footnote{\url{https://www.csie.ntu.edu.tw/~cjlin/libsvmtools/datasets/regression.html}}. This spatial dataset contains 20,640 observations on housing prices with 9 economic covariates. As described in \cite{pace1997sparse}, here $\xb$ are information about households in a block, say median age, median income, total rooms/population, bedrooms/population, population/households,households etc. and $y$ is median price in median housing prices by all California census blocks. It has $\mathrm{dimension}(\xb)=$ 8.
    \item \textbf{Law:} This refers to the dataset on Law School Admissions Council’s National Longitudinal Bar Passage Study~\cite{law}. Here $\xb$ is information about a law student, including information on gender, race, family income, age, {\em etc.} and $y$ indicates GPA normalised to $[0,1]$ . We use race as a protected attribute for the fairness experiments.
    It has $\mathrm{dimension}(\xb)=$ 10.
    \item \textbf{NYSE-High:}  This dataset is obtained from the New York stock exchange (NYSE)~\footnote{\url{https://www.kaggle.com/dgawlik/nyse}} dataset as follows.
    Given the set $\set{s_i}$ with $s_i$ corresponding to the the highest stock price of the $i^{th}$ day, we define $s_{k+1} = \sum_{i\in[100]}w_i s_{k+1-i}$. Here $y_k=s_{k+1}$ and $\xb_k = [s_{k}, s_{k-1},...,s_{k-99}]$. This dataset has $\mathrm{dimension}(\xb)=100$.
    \item \textbf{NYSE-close:}  This dataset is obtained from the New York stock exchange (NYSE)~\footnote{\url{https://www.kaggle.com/dgawlik/nyse}} dataset as follows.
    Given the set $\set{s_i}$ with $s_i$ corresponding to the the closing stock price of the $i^{th}$ day, we define $s_{k+1} = \sum_{i\in[100]}w_i s_{k+1-i}$. Here $y_k=s_{k+1}$ and $\xb_k = [s_{k}, s_{k-1},...,s_{k-99}]$. This dataset has $\mathrm{dimension}(\xb)=100$.
    \end{itemize}



\subsection{Implementation details}

\xhdr{Our models}
We use two models--- a simple linear regression model and a two layer neural network that consists of a linear layer of $5$ hidden nodes and a ReLU activation unit. In all our experiments, we use a learning rate of 0.01. We choose the value of $\delta$
as the 30\% of the mean validation error obtained using \full. 

\xhdr{Implementation of \craig} \craig \cite{mirzasoleiman2019coresets} 
requires computing a $\mathcal{D} \times \mathcal{D}$ matrix with similarity measure for each pair of points in the training set. 
For the larger datasets, \ie, NYSE-close and NYSE-high, such a computation requires a large amount of memory. Hence, we use a stochastic version where we randomly select $R$ points and build $R \times R$ matrix and select $\frac{kR}{\mathcal{D}}$ each time and repeat the process $\frac{\Dcal}{R}$ times. We use $R = 50000$. Note that, for other datasets, since $|\Dcal| < 50000$ the stochastic version is same as the original version. 

\craig requires us to select the subset only once, since  features will not change even as the training proceeds. However, since \craig is an adaptive method, for the non-linear setting, we need to run \craig every epoch.   Despite using the stochastic version, we found \craig to be very slow in the non-linear setting and  therefore we don't report it.

\xhdr{Implementation of \glister}
\glister~\cite{killamsetty2020glister}, an another adaptive subset selection method where we select a new subset every $35^{\text{th}}$ epoch to help make a fair comparison against \our. 
We  update the model parameters after every selection step. 

\xhdr{Machine configuration}
We performed our experiments on a computer system with Ubuntu 16.04.6 LTS, an i-7 with 6 cores CPU and a total RAM of 125 GBs. The system had a single GeForce GTX 1080 GPU which was employed in our experiments.
%
%

 \section{Additional experiments}
\label{app:real} 

\begin{figure}[h]
\centering
\hspace*{-0.1cm}\hspace*{-0.6cm}\subfloat[Cadata]{\setcounter{subfigure}{1} \includegraphics[ width=0.20\textwidth]{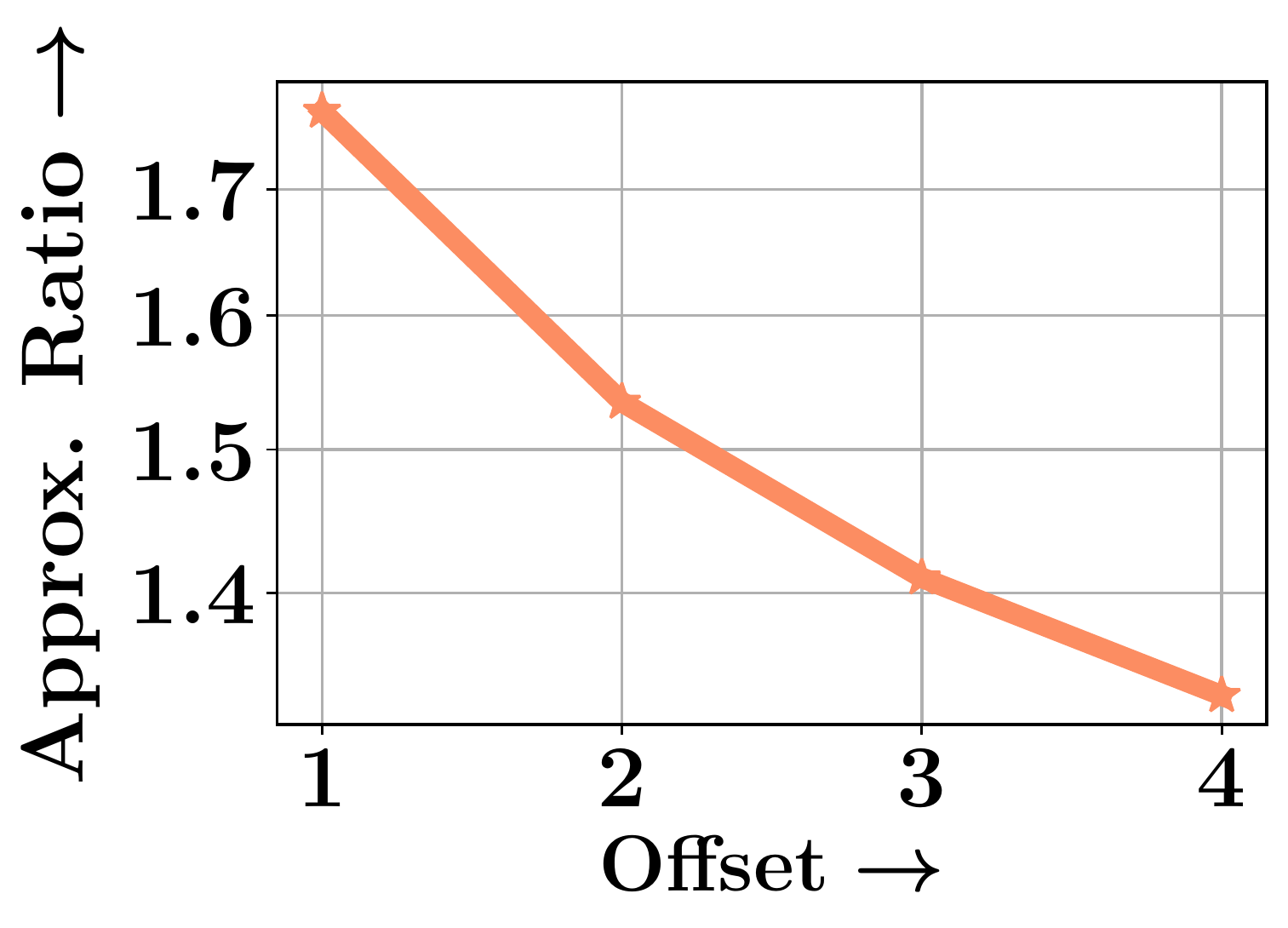}}\hspace{1mm}
\hspace*{0.1cm}\subfloat[Law]{\includegraphics[width=0.20\textwidth]{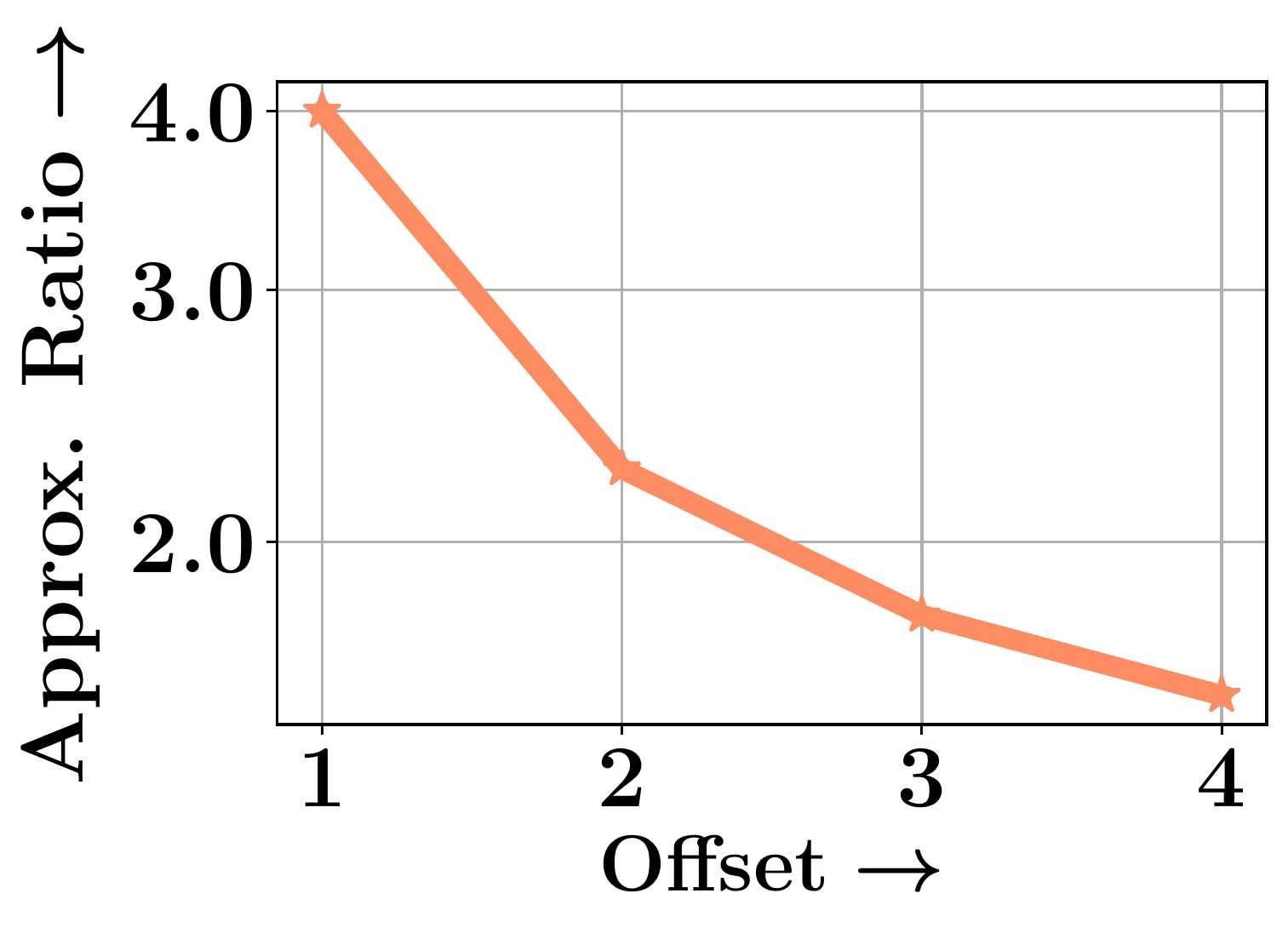}}\hspace{1mm}
\hspace*{0.1cm}\subfloat[NYSE-High]{ \includegraphics[width=0.20\textwidth]{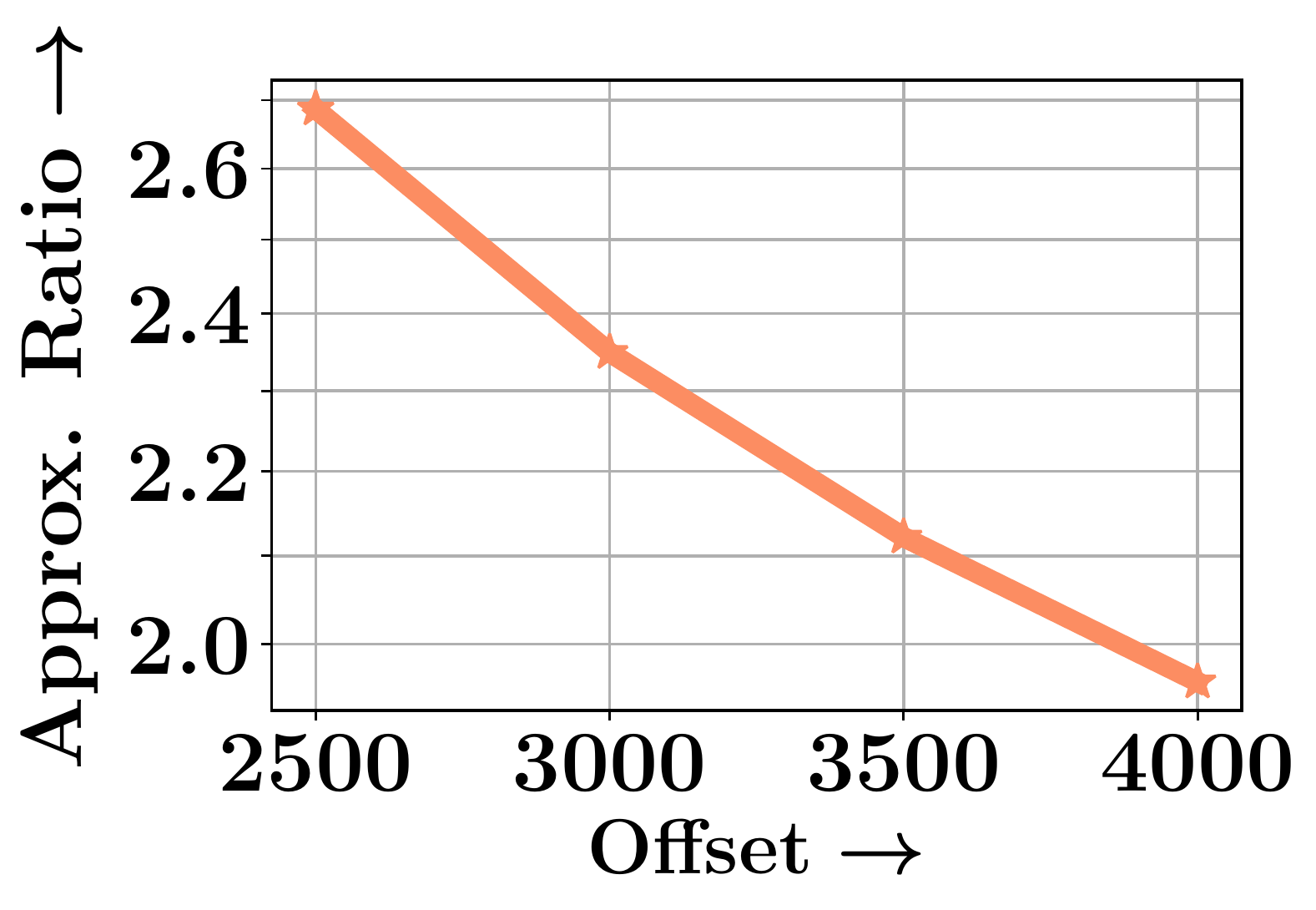}}\hspace{1mm}
\hspace*{0.1cm}\subfloat[NYSE-Close]{ \includegraphics[width=0.20\textwidth]{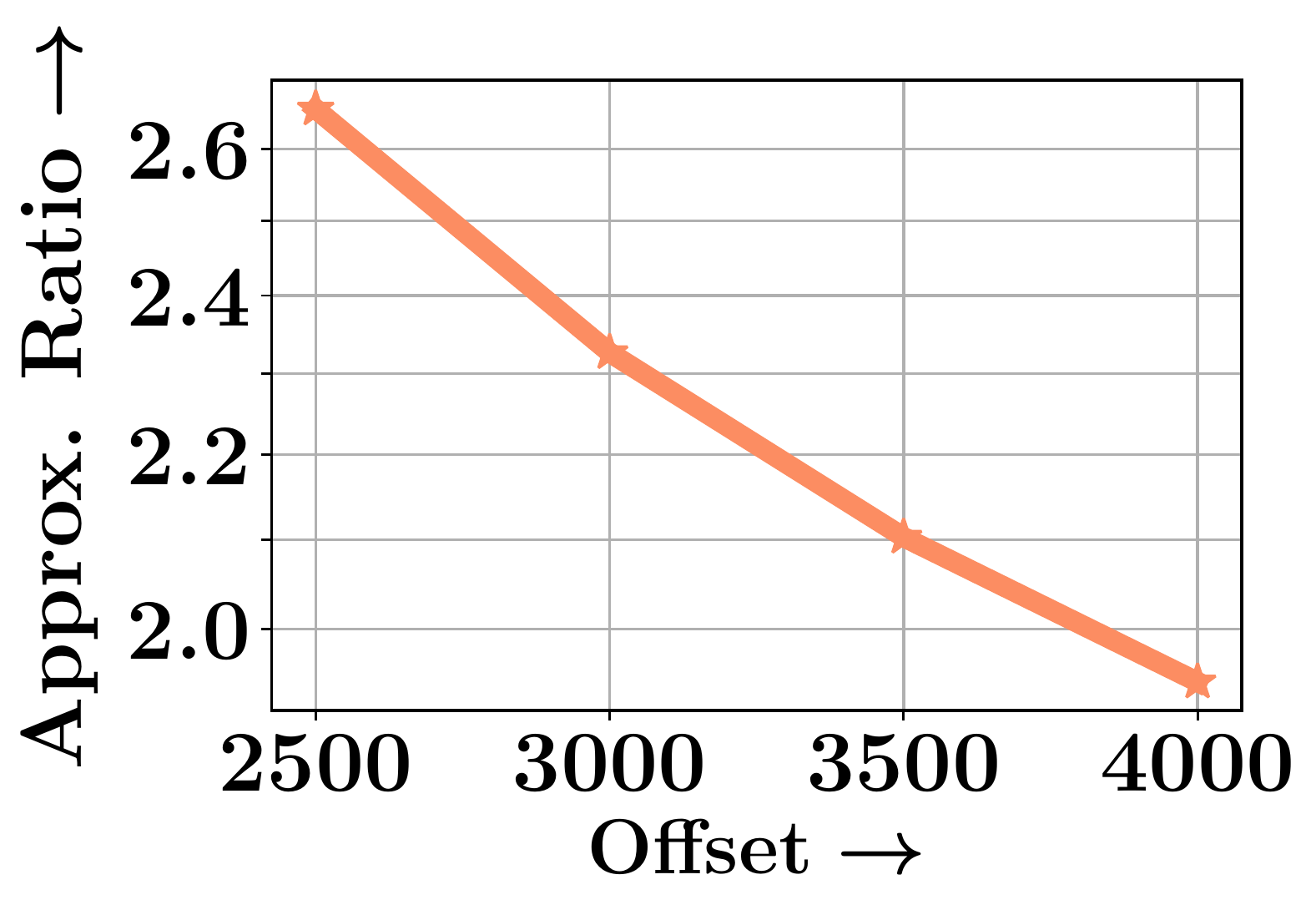}}
\caption{Variation in the approximation ratio with respect to the offset added to the response variables $y$.}
\label{fig:approx}
\end{figure}


\subsection{Discussion on adding offsets to the response variable $y$} The approximation ratio of \our is 
$f(\Schat)/f(\Scal^*) \le \dfrac{k}{\af(1 + (k-1)(1 - \kf)\af)}$ when the training method is accurate. A trite calculation shows that this quantity is $O(\ymx^4/y_{\min}^4)$. If $\ymx/y_{\min}$ is very high, the approximation ratio is affected. Such a problem can be easily overcome by adding an offset to  $y$ and then augmenting the feature $\xb$ with an additional term $1$--- which incorporates the effect of the added offset. 
We summarize the effect of this offset on the approximation ratio (for different datasets) in Figure~\ref{fig:approx} which shows that adding an offset
improves the approximation factor. Note that in the case of Cadata,  $y$ indicates the house price; whereas in the case of Law, $y$ indicates student GPA. Therefore, the approximation factor of these datasets is reasonable even without adding an offset. Whereas, for  NYSE-High and NYSE-Clone, the approximation factor is somewhat poorer at lower values of the offset (not shown in the plot).

\subsection{Significance Tests}

\begin{table}[!htbp]
    \centering
    
    \scalebox{0.7125}{
   
    \begin{tabular}{c|cccccc}
        
        \textsc{Random-selection} &  &   &   &   &     \\ \cline{2-2}
        \textsc{Random-with-constraints} & \multicolumn{1}{|c|}{0.000089} &  &   &   &   \\ \cline{2-3}
        \textsc{\craig} & \multicolumn{1}{|c|}{0.00012 } & \multicolumn{1}{c|}{0.50159} &  &  &    \\ \cline{2-4}
        \textsc{\glister} & \multicolumn{1}{|c|}{0.040043 } & \multicolumn{1}{c|}{0.00014}& \multicolumn{1}{c|}{0.00078} &  &   &   \\ \cline{2-5}
        \textsc{\our-without-constraints} & \multicolumn{1}{|c|}{0.001713} & \multicolumn{1}{c|}{0.601212} & \multicolumn{1}{c|}{0.88129} & \multicolumn{1}{c|}{0.00803} &  \\ \cline{2-6}
        \textsc{\our} & \multicolumn{1}{|c|}{0.0001} & \multicolumn{1}{c|}{0.00014} & \multicolumn{1}{c|}{ 0.00059} & \multicolumn{1}{c|}{ 0.0001} & \multicolumn{1}{c|}{0.0001} \\ \cline{2-7}
        \multicolumn{1}{c}{}& \rotatebox[origin=c]{60}{\textsc{Random-selection}} & \rotatebox[origin=c]{60}{\textsc{Random-with-constraints}} & \rotatebox[origin=c]{60}{\textsc{\craig}} & \rotatebox[origin=c]{60}{\textsc{\glister}} & \rotatebox[origin=c]{60}{\textsc{\our-without-constraints}} & \rotatebox[origin=c]{60}{\textsc{\our}}\\ 
    \end{tabular}}
    \caption{Pairwise significance p-values using Wilcoxon signed rank test.}
    \label{tab:significance}
\end{table}

In Table~\ref{tab:significance}, we show the p-values of two-tailed Wilcoxon signed-rank test \cite{wilcoxon1992individual} performed on every possible pair of data selection strategies to determine whether there is a significant statistical difference between the strategies in each pair, across all datasets. Our null hypothesis is that there is no difference between each pair of data selection strategies. From the results, it is evident that \textsc{SelCon} significantly outperforms other baselines at $p<0.01$.


\end{document}